\documentclass[twoside]{article}

\usepackage[accepted]{aistats2024}
\usepackage[round]{natbib}

\usepackage{graphicx}
\usepackage{threeparttable}
\usepackage{pifont}
\usepackage{enumitem}

\usepackage{algorithm, algorithmic}
\usepackage{amsfonts, bm, amssymb}

\usepackage{color}
\usepackage[dvipsnames]{xcolor}

\newtheorem{theorem}{Theorem}[section]
\newtheorem{assume}{Assumption}[section]
\newtheorem{lemma}{Lemma}[section]
\newtheorem{definition}{Definition}[section]
\newtheorem{remark}{Remark}[section]
\newtheorem{prop}{Proposition}[section]
\newtheorem{corollary}{Corollary}[section]
\newcommand{\qed}{\hfill$\blacksquare$}
\newenvironment{proof}{{\noindent\bf Proof}\ \ }{\hfill\qed\par\vspace{1em}}

\newcommand{\argmin}{\mathop{\rm argmin}}
\newcommand{\argmax}{\mathop{\rm argmax}}

\newcommand{\Argmax}{\mathop{\rm Argmax}}

\usepackage{subfiles}
\usepackage{xcolor}

\usepackage{verbatim}
\usepackage{footnotebackref}
\usepackage{booktabs}
\usepackage{multirow} 
\usepackage{xspace}

\def \alg{\texttt{FESS-GDA}\xspace}

\begin{document}



\twocolumn[

\aistatstitle{Stochastic Smoothed Gradient Descent Ascent for Federated Minimax Optimization}

\aistatsauthor{ Wei Shen \And Minhui Huang \And   Jiawei Zhang \And Cong Shen }

\aistatsaddress{ University of Virginia \And  Meta AI \And MIT \And University of Virginia } ]

\begin{abstract}
In recent years, federated minimax optimization has attracted growing interest due to its extensive applications in various machine learning tasks. While Smoothed Alternative Gradient Descent Ascent (Smoothed-AGDA) has proved successful in centralized nonconvex minimax optimization, how and whether smoothing techniques could be helpful in a federated setting remains unexplored. In this paper, we propose a new algorithm termed Federated Stochastic Smoothed Gradient Descent Ascent (\alg), which utilizes the smoothing technique for federated minimax optimization. We prove that \alg can be uniformly applied to solve several classes of federated minimax problems and prove new or better analytical convergence results for these settings. We showcase the practical efficiency of \alg in practical federated learning tasks of training generative adversarial networks (GANs) and fair classification. 
\end{abstract}

\section{INTRODUCTION}

\begin{table*}[h]
\centering
\begin{threeparttable}
\renewcommand*{\arraystretch}{1.2}
\caption{\small Comparison of per-client sample complexity and communication complexity for different classes of nonconvex minimax problems. For comparison, we only give the convergence results for finding an $\epsilon$-stationary point of $\Phi$ (Definition \ref{def:phi}) for NC-PL and of $\Phi_{1/2l}$ (Definition \ref{def:phi 1/2l}) for NC-1PC under full client participation ($m=M$). We also provide convergence results of finding an $\epsilon$-stationary point of $f$ (Definition \ref{def:f}), and consider partial client participation ($m<M$) in our paper. $\kappa:=l/\mu$ is the conditional number.}
\label{table}
\begin{tabular}{|l|c|c|c|}
\hline
\multirow{3}{*}{Algorithms}  &\multirow{3}{*}{\begin{tabular}[c]{@{}c@{}} Partial Client \\ Participation \end{tabular} }&\multicolumn{2}{|c|}{Full Client Participation (FCP)}\\ \cline{3-4}
& & \begin{tabular}[c]{@{}c@{}}Per-client Sample\\ complexity \end{tabular} & \begin{tabular}[c]{@{}c@{}}Communication \\ complexity \end{tabular} \\ \hline
\multicolumn{4}{|c|}{Nonconvex-Strongly-Concave (NC-SC)/ Nonconvex-PL (NC-PL)}
\\ \hline
Local SGDA \citep{sharma2022federated} &\ding{55} & \hspace{0.3cm}$O(\kappa^4 m^{-1}\epsilon^{-4})$\hspace{0.3cm} & $O(\kappa^3 \epsilon^{-3})$ \\ \hline
SAGDA \citep{yang2022sagda} & \checkmark & $O(\kappa^4 m^{-1}\epsilon^{-4})$ & $O(\kappa^2 \epsilon^{-2})$ \\ \hline
Fed-Norm-SGDA \citep{sharma2023federated} & \checkmark & $O(\kappa^4 m^{-1}\epsilon^{-4})$ & $O(\kappa^2 \epsilon^{-2})$ \\ \hline
FedSGDA-M\tnote{a} \hspace{0.1cm} \citep{wu2023solving} &\ding{55} & $O(\kappa^3 m^{-1}\epsilon^{-3})$ & $O(\kappa^2 \epsilon^{-2})$ \\ \hline
\textcolor{blue}{\alg} Corollary \ref{coro:sc full participation} & \checkmark & \textcolor{blue}{$O(\kappa^2 m^{-1}\epsilon^{-4})$} & \textcolor{blue}{$O(\kappa \epsilon^{-2})$}   \\ \hline 
\multicolumn{4}{|c|}{Nonconvex-One-Point-Concave (NC-1PC)} 
\\ \hline
Local SGDA+\tnote{b} \hspace{0.1cm} \citep{sharma2022federated} &\ding{55} & $O(\epsilon^{-8})$ & $O(\epsilon^{-7})$ \\ \hline
Fed-Norm-SGDA+\tnote{b}\hspace{0.1cm}  \tnote{c}  \hspace{0.1cm} \citep{sharma2023federated} \hspace{0.6cm} &\checkmark & $O(m^{-1}\epsilon^{-8})$ & $O(\epsilon^{-4})$ \\ \hline
\textcolor{blue}{\alg} Theorem \ref{thm:1pc} &\checkmark & \textcolor{blue}{$O(m^{-1}\epsilon^{-8})$} & \textcolor{blue}{$O(\epsilon^{-4})$} \\ \hline

\multicolumn{4}{|c|}{Nonconvex-Concave (NC-C)} 
\\ \hline
Local SGDA+\tnote{b} \hspace{0.1cm} \citep{sharma2022federated}    & \ding{55} &$O(m^{-1}\epsilon^{-8})$ & $O(\epsilon^{-7})$ \\ \hline
Fed-Norm-SGDA+\tnote{b} \hspace{0.1cm} \citep{sharma2023federated} & \checkmark &$O(m^{-1}\epsilon^{-8})$ & $O(\epsilon^{-4})$ \\ \hline
FedSGDA+ \citep{wu2023solving} &\ding{55} & $O(m^{-1}\epsilon^{-8})$ & $O(\epsilon^{-6})$ \\ \hline
\textcolor{blue}{\alg}\tnote{d} \hspace{0.1cm} Theorem \ref{thm:1pc} &\checkmark & \textcolor{blue}{$O(m^{-1}\epsilon^{-8})$} & \textcolor{blue}{$O(\epsilon^{-4})$} \\ \hline
\multicolumn{4}{|c|}{Objective function has a form of (\ref{special case}) (A special case of NC-C problems)} 
\\ \hline
\textcolor{blue}{\alg} Theorem \ref{thm:special case} & \checkmark & \textcolor{blue}{$O(m^{-1}\epsilon^{-4})$} & \textcolor{blue}{$O(\epsilon^{-2})$} \\ \hline

\end{tabular}
\scriptsize
\begin{tablenotes}
\item[a] Their better performance comes from using additional variance reduction, while we do not.
\item[b] Their proofs need additional assumptions that each local loss function $f_i$ also satisfies the NC-C (NC-1PC) condition, while ours only needs the global loss function $f$ to be NC-C (NC-1PC). They also assume $\|y_t\|^2\leq D$, but do not mention how to guarantee this. We use a projection operator in our algorithm to guarantee this.
\item[c] Their proof requires additional assumption that each local loss function $f_i$ satisfies the one-point-concave condition with a common global minimizer $y^*(x)$.
\item[d] We have better convergence results for the NC-C setting of finding an $\epsilon$-stationary point of $f$; see Theorem \ref{thm:c f} for details.
\end{tablenotes}
\end{threeparttable}
\end{table*}

\begin{table*}[h]
\centering
\begin{threeparttable}
\centering
\renewcommand*{\arraystretch}{1.2}
\caption{\small Comparison of per-client sample complexity and communication complexity needed to find $(x_T,y_T)$ that satisfy $\mathbb{E}\|x_T-x^*\|^2+\mathbb{E}\|y_T-y^*\|^2\leq \epsilon^2$. We use $\Tilde{O}$ to hide logarithmic terms.}
\label{table plpl}
\begin{tabular}{|c|c|c|c|c|c|}
\hline
Algorithms                         & Type                   & \begin{tabular}[c]{@{}c@{}} Partial Client \\ Participation \end{tabular} & \begin{tabular}[c]{@{}c@{}}Data \\ Heterogeneity \end{tabular} & \begin{tabular}[c]{@{}c@{}} Per-client Sample \\ complexity\end{tabular} & \begin{tabular}[c]{@{}c@{}} Communication \\ complexity \end{tabular}\\ \hline
\multirow{2}{*}{ \begin{tabular}[c]{@{}c@{}}Local SGDA\tnote{a}\\ \cite{deng2021local}\end{tabular} }        & \multirow{2}{*}{SC-SC} & \ding{55}    & \ding{55}             &$\Tilde{O}(M^{-1}\epsilon^{-2})$   &  $\Tilde{O}(M)$ \\ \cline{3-6} 
                                   &                        & \ding{55}       & \checkmark              &  $O(M^{-1}\epsilon^{-2})$ & $O(\epsilon^{-1})$  \\ \hline
\multirow{3}{*}{\begin{tabular}[c]{@{}c@{}} \alg\\ Theorem \ref{thm: plpl} \end{tabular}} & \multirow{3}{*}{PL-PL} & \ding{55}       &  \checkmark             & $\Tilde{O}(M^{-1}\epsilon^{-2})$  & $\Tilde{O}(1)$  \\ \cline{3-6} 
                                   &                        & \checkmark        & \ding{55}            & $\Tilde{O}(m^{-1}\epsilon^{-2})$  & $\Tilde{O}(1)$  \\ \cline{3-6} 
                                   &                        & \checkmark       & \checkmark             &  $\Tilde{O}(m^{-1}\epsilon^{-2})$ & $\Tilde{O}(m^{-1}\epsilon^{-2})$  \\ \hline
\end{tabular}
\scriptsize
\begin{tablenotes}
\item[a] Their proofs need an assumption that each local loss function $f_i$ satisfies the SC-SC condition, while ours only needs the global loss function $f$ to satisfy the PL-PL condition (Assumption \ref{assum: plpl}).
\end{tablenotes}
\end{threeparttable}
\end{table*}

Minimax optimization is widely encountered in modern machine learning tasks such as generative adversarial networks (GANs) \citep{goodfellow2014generative}, AUC maximization \citep{liu2019stochastic}, reinforcement learning \citep{zhang2021multi}, adversarial training \citep{goodfellow2014explaining}, and fair machine learning \citep{nouiehed2019solving}. In recent years, many progresses on minimax optimization problems have been reported, with the majority focusing on solutions at a single client level. However, modern machine learning tasks usually demand a huge amount of data. A significant portion of this data may be sensitive, rendering it unsuitable for sharing with servers due to privacy concerns \citep{leaute2013protecting}. Furthermore, data sourced from edge devices can be hindered by the limited communication capabilities with the server. To preserve data privacy and to address communication issues, federated learning (FL) was proposed \citep{mcmahan2017communication}. In FL, clients do not send their data directly to the server. Instead, each client trains its model locally using its own data. Periodically, clients communicate with the server, sending their models for aggregation. The server then returns the updated model to the clients. 

Solutions and analyses for federated minimax problems have been developed in recent years. Some focus on convex-concave problems \citep{deng2020distributionally, hou2021efficient, sun2022communication}, and others are devoted to more general nonconvex minimax problems \citep{deng2021local, sharma2022federated, sharma2023federated}. Because the objective functions are usually nonconvex in the min variables for many practical applications, we mainly focus on federated nonconvex minimax problems in this paper.  

Gradient descent ascent (GDA) and its stochastic version stochastic gradient descent ascent (SGDA) are the simplest single-loop algorithms for centralized minimax problems. Most existing federated minimax algorithms are extensions of GDA (SGDA) to the federated setting, i.e. Local SGDA \citep{deng2021local}, Fed-Norm-SGDA \citep{sharma2022federated}. \citet{zhang2020single} propose Smoothed-AGDA, a single-loop algorithm utilizing the smoothing technique, and prove that it has a faster convergence rate for centralized nonconvex-concave problems compared with GDA. \cite{yang2022faster} then prove that Smoothed-AGDA and its stochastic version Stochastic Smoothed-AGDA also have faster convergence rates for centralized nonconvex-PL (Polyak-Lojasiewicz) problems compared with GDA (SGDA). A natural question arises: \textbf{Can we utilize the smoothing techniques to design a faster algorithm for federated nonconvex minimax optimization?} 

Furthermore, in the current literature, usually two different algorithms (such as Local SGDA and Local SGDA+ \citep{deng2021local, sharma2022federated}) are needed for different nonconvex minimax settings, which limits their practical applicability. Another question thus arises: \textbf{Can we design a single, uniformly applicable algorithm for federated nonconvex minimax optimization?}

\subsection{Problem Setting}
In this paper, we study the federated minimax optimization problems in the following form:
\begin{equation}\label{problem}
    \min_{x\in X}\max_{y\in Y} \bigg\{f(x,y)=\frac{1}{M}\sum_{i=1}^M f_i(x,y)\bigg\},
\end{equation}
where $X=\mathbb{R}^{d_1}, Y\subseteq \mathbb{R}^{d_2}$, $M$ is the number of clients, $f_i(x,y)=\mathbb{E}_{\xi_i\sim \mathcal{D}_i}[f(x,y;\xi_i)]$ is the local loss function at client $i$, and $f(x,y;\xi_i)$ denotes the loss for the data point $\xi_i$, sampled from the local data distribution $\mathcal{D}_i$ at client $i$. 

For the nonconvex-concave setting, we also consider a special case:
\begin{align}
\label{special case}
    \min_{x\in X}\max_{y\in Y} f(x,y)=\min_{x\in X}\max_{y\in Y} F(x)^Ty,
\end{align}
where $Y=\{(y_1, ..., y_n)^T | \sum_{i=1}^n, y_i=1, y_i\geq 0\}$ and $F(x)=(f_1(x), ..., f_n(x))^T$
is a mapping from $X=\mathbb{R}^{d_1}$ to $\mathbb{R}^n$. Note that \eqref{special case} is equivalent to the problem of minimizing the point-wise maximum of a finite collection of functions:
\begin{align}
\label{point-wise}
\min_x\max_{1\le i\le n}f_i(x).
\end{align}
Problems in the form of \eqref{special case} and \eqref{point-wise} commonly appear in practical applications such as adversarial training \cite{nouiehed2019solving, madry2017towards} and fairness training \cite{nouiehed2019solving}.

\subsection{Contributions}
We propose a new algorithm termed \underline{FE}derated \underline{S}tochastic \underline{S}moothed \underline{G}radient \underline{D}escent \underline{A}scent (\alg). We prove that \alg can be uniformly used to solve several classes of federated nonconvex minimax problems, and prove new or better convergence results for these settings. We summarize our main theoretical results in Tables \ref{table}, \ref{table plpl} with the following abbreviations:

\emph{SC-SC}: Strongly-Convex in $x$, Strongly-Concave in $y$,

\emph{PL-PL}: PL condition in $x$, PL condition in $y$ (Assumption \ref{assum: plpl}),

\emph{NC-SC}: Nonconvex in $x$, Strongly-Concave in $y$,

\emph{NC-PL}: Nonconvex in $x$, PL condition in $y$ (Assumption \ref{assum:pl}),

\emph{NC-C}: Nonconvex in $x$, Concave in $y$ (Assumption \ref{assum:concave}),

\emph{NC-1PC}: Nonconvex in $x$, One-Point-Concave in $y$ (Assumption \ref{assum:1pc_y}).

More specifically, our contributions are the following. 

\begin{itemize}[leftmargin=*]\itemsep=0pt
    \item For NC-PL and NC-SC problems, we prove that \alg achieves a per-client sample complexity of $O(\kappa^2 m^{-1}\epsilon^{-4})$ and a communication complexity of $O(\kappa\epsilon^{-2})$ in terms of the stationarity of both $f$ and $\Phi$. The previously best-known results without variance reduction in the federated setting are $O(\kappa^4 m^{-1}\epsilon^{-4})$ per-client sample complexity and  $O(\kappa^2\epsilon^{-2})$ communication complexity. We improve these results by a factor of $O(\kappa^2)$ in the sample complexity and a factor of $O(\kappa)$ in the communication complexity.

    \item To the best of our knowledge, we are the first to prove convergence results of solving \eqref{special case} under a federated setting. We prove that \alg has a sample complexity of $O(m^{-1}\epsilon^{-4})$ and a communication complexity of $O(\epsilon^{-2})$ in terms of the stationarity of both $f$ and $\Phi$, which is much better than the complexity we can achieve for general NC-C problems.
    
    \item For general NC-C and NC-1PC problems, we prove that \alg achieves comparable performances as the current state-of-the-art algorithm, but with weaker assumptions. Moreover, we provide additional convergence results for these two settings in terms of the stationarity of $f$. For the NC-C problems, we prove a best-known per-client sample complexity of $O(m^{-1}\epsilon^{-6})$ and a communication complexity of $O(\epsilon^{-3})$ in terms of stationarity of $f$.
    
    \item To the best of our knowledge, we are the first to provide convergence results of general federated minimax problems with the PL-PL condition. We prove a better communication complexity of \alg in the PL-PL setting, compared with Local SGDA under the SC-SC setting \citep{deng2021local}, despite that PL-PL is much weaker than SC-SC.

\end{itemize}

\subsection{Related Works}

\textbf{Nonconvex-Strongly-Concave.} For stochastic NC-SC problems, \citet{lin2020gradient} proved that SGDA achieves $O(\kappa^3\epsilon^{-4})$ sample complexity with a batch size of $O(\epsilon^{-2})$. \cite{qiu20single_timescale_ncsc, luo20SREDA_ncsc_neurips} improved the sample complexity to {$O (\kappa^3\epsilon^{-3})$} with a variance-reduction technique. \cite{yang2022faster} proved that Stochastic Smoothed-AGDA can achieve $O(\kappa^2\epsilon^{-4})$ sample complexity. 

\noindent\textbf{Nonconvex-Concave.} \cite{lin2020gradient} analyzed GDA and SGDA for NC-C problems and proved that GDA can achieve $O(\epsilon^{-6})$ sample complexity for the deterministic setting and SGDA can achieve $O(\epsilon^{-8})$ sample complexity for the stochastic setting. \cite{zhang2020single} proposed Smoothed-AGDA and proved that it can achieve $O(\epsilon^{-4})$ sample complexity for the deterministic setting. For the stochastic setting, \cite{rafique18WCC_oms, zhang2022sapd_NC_C_arxiv} improved the complexity to $O (\epsilon^{-6})$ with a nested structure.

\noindent\textbf{Federated minimax.}  There is a growing interest in solving federated minimax problems. Some focused on the convex-concave setting  \citep{deng2020distributionally, hou2021efficient,liao21local_AdaGrad_CC_arxiv, sun2022communication}. There is also progress in the nonconvex setting. 
\cite{mahdavi20dist_robustfl_neurips} analyzed a nonconvex-linear setting. \cite{reisizadeh20robustfl_neurips} formulated robust federated learning problems as special cases of federated PL-PL and NC-PL minimax problems and analyzed the convergence results of their proposed methods for these settings. \cite{deng2021local} proposed Local SGDA and Local SGDA+ and analyzed their convergence results under several nonconvex settings. \cite{sharma2022federated} improved the convergence results in \cite{deng2021local}. \cite{yang2022sagda} proposed SAGDA and improved the communication complexity for the NC-PL setting. \cite{sharma2023federated} proposed Fed-Norm-SGDA and Fed-Norm-SGDA+ and further improved the convergence results under several nonconvex settings. \cite{tarzanagh2022fednest} proposed FEDNEST with a nested structure and showed $O(\kappa^3\epsilon^{-4})$ sample complexity for the NC-SC setting. \cite{huang2022adaptive} designed AdaFGDA allowing for adaptive learning rates and improved the sample complexity to $\Tilde{O}(\epsilon^{-3})$ for NC-PL setting with variance-reduction techniques.  Recently, \citet{wu2023solving} proposed FedSGDA-M and improved the sample complexity to $O(\kappa^3\epsilon^{-3})$ for the NC-PL setting with variance-reduction techniques.

\section{PRELIMINARIES}
\label{section: pre}

\textbf{Notations.} We denote the $l_2$ norm as $\|\cdot\|_2$. For a  differentiable function $g(x,y)$, we denote its gradient as $\nabla g(x,y)=(\nabla_x g(x,y)^T, \nabla_y g(x,y)^T)^T$. We define $\Phi(x)=\max_{y\in Y}f(x,y)$, $ P_Y(y)=\argmin_{y'\in Y}\frac{1}{2}\|y-y'\|^2$.

We state some common assumptions that will be used throughout the paper. They are commonly used in (federated) minimax optimization; e.g., \citep{yang2022faster, zhang2020single, deng2021local, sharma2022federated}.

\begin{assume} [Lipschitz smooth]
\label{assum:smooth}
Each local function $f_i$ is differentiable and there exists a positive constant $l$ such that for all $i \in [M]$, and for all $x_1, x_2\in X, y_1$, $y_2 \in Y$, we have
\begin{align*}
     &\| \nabla f_i(x_{1} , y_{1}) - \nabla f_i(x_{2} , y_{2} )\|
     \leq l [\| x_{1} - x_{2} \| + \| y_{1} - y_{2} \|].
\end{align*}
\end{assume}
\begin{assume}[Bounded variance]
	\label{assum:bdd_var}
	The gradient of each local function $f_i(x,y,\xi_i)$, with a random data sample $\xi_i\sim\mathcal{D}_i$, is unbiased and has bounded variance, i.e., there exists a constant $\sigma > 0$ such that for all $i \in [M]$, and for all $x\in X, y\in Y$, $\mathbb{E}[ \nabla f_i(x, y; \xi_i) ] = \nabla f_i(x, y)$, and $\mathbb{E} \| \nabla f_i(x, y; \xi_i) - \nabla f_i(x, y) \|^2  \leq \sigma^2.$
	
\end{assume}

\begin{assume}[Bounded heterogeneity]
	\label{assum:bdd_hetero}
	To bound the heterogeneity of the local functions $\{ f_i(x, y) \}$ across the clients, we assume there exits a constant $\sigma_G>0$ such that
	\begin{align*}
    &\sup_{x\in X, y \in Y, i\in[M]} \|\nabla f_i(x,y)-\nabla f(x,y)\|^2\leq \sigma_G^2.
	\end{align*}
\end{assume}

\begin{assume}
	\label{assum:phi}
	$\Phi(x)=\max_{y \in Y} f(x,y)$ is lower bounded by a finite $\Phi^* > -\infty$. 
\end{assume}

The following notions of stationarity measures are also commonly used in the study of minimax optimization.

\begin{definition} [Stationarity measures of $f$]\quad
\label{def:f} We say $(\hat{x}, \hat{y})$ is an $(\epsilon_1, \epsilon_2)$-stationary point of a differentiable function $f(\cdot, \cdot)$ if $\|\nabla_x f(\hat{x}, \hat{y}) \| \leq \epsilon_1$ and  $l\|P_Y(\hat{y}+1/l\nabla_y f(\hat{x}, \hat{y}))-\hat{y} \| \leq \epsilon_2$. If  $(\hat{x}, \hat{y})$ is an $(\epsilon, \epsilon)$-stationary point of $f$, we say it is an $\epsilon$-stationary point of $f$.
\end{definition}

\begin{definition} [Stationarity measures of $\Phi$]\quad
\label{def:phi}
We say $\hat{x}$ is an $\epsilon$-stationary point of a differentiable function $\Phi(\cdot)$ if $\|\nabla\Phi(\hat{x})\|\leq\epsilon$.
\end{definition}
When $f$ satisfies the PL condition in $y$, $\Phi(x)=\max_{y\in Y} f(x,y)$ is $2\kappa l$-Lipschitz smooth \citep{nouiehed2019solving}. Thus, the stationarity measure of $\Phi$ is widely used in NC-PL and NC-SC settings. 
However, for other settings like NC-C, NC-1PC, $\Phi(x)$ is not guaranteed to be smooth, and the stationarity measure of the Moreau Envelope of the $\Phi(x)$ (Definition \ref{def:phi 1/2l}) is commonly used.

\begin{definition}[Moreau envelope]
A function $\Phi_{\lambda}(z)$ is the Moreau envelope of $\Phi(x)$ with $\lambda > 0$, if for all $z \in X$, $\Phi_\lambda(z) = \min_{x} \Phi (x) + (1/2\lambda) \|z-x\|^2.$
\end{definition}

\begin{definition} [Stationarity measures of $\Phi_{1/2l}$]
\label{def:phi 1/2l}
We say $\hat{x}$ is an $\epsilon$-stationary point of $\Phi_{1/2l}(\cdot)$ if $\|\nabla\Phi_{1/2l}(\hat{x})\|\leq\epsilon$.
\end{definition}

\section{FESS-GDA}
\label{section: alg}

\subsection{Algorithm}
Inspired by the success of Smoothed-AGDA in the centralized setting \citep{zhang2020single, yang2022faster}, we propose \alg, which is compactly presented in Algorithm \ref{alg1}, for the federated minimax optimization problem. We consider a system with $M$ clients and one central server. In each communication round, the server first randomly samples $m$ clients and then sends them the current global model $(x_t,y_t)$. For all participating clients, they synchronize their local models with the global model and perform $K$ local updates with their local data and local learning rate $\eta_{x,l}, \eta_{y,l}$. After the completion of local updates, each client sends back their local models to the server. Then, instead of a standard aggregation for local models like Local SGDA \citep{deng2021local}, the key difference of \alg here is that we introduce an auxiliary parameter $z_t$ to smooth the update of $x_t$.

Note that with  a small local learning rate that $x^k_{t,i}\approx x_t, y^k_{t,i}\approx y_t$ and Assumption \ref{assum:bdd_var}, the local updates can be approximated as 
\begin{align*}
x^{k+1}_{t,i} &\approx x^k_{t,i}-\eta_{x,l}\nabla_x f_i(x_t,y_t),\\
y^{k+1}_{t,i} &\approx y^k_{t,i}+\eta_{y,l}\nabla_y f_i(x_t,y_t),
\end{align*}
and with Assumption \ref{assum:bdd_hetero}, the update of $x_t,y_t$ can be approximated as 
\begin{align*}
    x_{t+1}&\approx x_t-\eta_{x,l}\eta_{x,g}K[\nabla_x f(x_t,y_t)+p(x_t-z_t)],\\
    y_{t+1}&\approx y_t+\eta_{y,l}\eta_{y,g}K\nabla_y f(x_t,y_t),\\
    z_{t+1}&=z_t+\beta(x_{t+1}-z_t),
\end{align*}
which has a similar form as the Smoothed-AGDA in the centralized setting.

Define $\hat{f}(x,y,z)=f(x,y)+\frac{p}{2}\|x-z\|^2$. Thus, in each communication round, we approximately perform gradient descent ascent of the following problem
\begin{align*}
    \min_x\max_y \hat{f}(x,y,z_t)=f(x,y)+\frac{p}{2}\|x-z_t\|^2.
\end{align*}

We set $\beta\in (0,1)$ to guarantee that $z_t$ is not too far from $x_t$. We choose $p=2l$ for the NC-PL, NC-1PC, NC-C settings so that $\hat{f}(x,y,z)$ is $l$-strongly convex in $x$. For the PL-PL setting, since $f$ satisfies the PL condition in $x$, we set $p=0$. Note that when $p=0, Y=\mathbb{R}^{d_2}$, \alg is equivalent to FSGDA \citep{yang2022sagda}, and when $p=0, Y=\mathbb{R}^{d_2}, \eta_{x,g}=\eta_{y,g}=1$, \alg is equivalent to Local SGDA \citep{deng2021local}.
\begin{algorithm}
  \caption{\alg}
  \label{alg1}
  \begin{algorithmic}[1]
  \STATE Input: $x_0,y_0,z_0, \eta_{x,l}, \eta_{y,l}, \eta_{x,g},\eta_{y,g},  \beta, p, T, K$
  \FOR{$t=0,1,\cdots,T-1$}
  \STATE Server randomly samples a subset $S_t$ of clients with $|S_t|=m$, and send them $(x_t,y_t)$.
  \FOR{each client $i \in  S_t$}
    \STATE $x^1_{t,i}=x_t, y^1_{t,i}=y_t$.
    \FOR{$k=1,2,\cdots,K$}
    \STATE $x^{k+1}_{t,i}=x^k_{t,i}-\eta_{x,l}\nabla_x f_i(x^k_{t,i},y^k_{t,i},\xi^k_{t,i})$
    \STATE $y^{k+1}_{t,i}=P_Y(y^k_{t,i}+\eta_{y,l}\nabla_y f_i(x^k_{t,i},y^k_{t,i},\xi^k_{t,i}))$
    \ENDFOR
    \STATE Each client send their local models $(x^{K+1}_{t,i}, y^{K+1}_{t,i})$ to the server.
  \ENDFOR
    \STATE Server aggregate local models and compute $(x_{t+1}, y_{t+1})$.
    \STATE $ x_{t+1}=x_{t}+\eta_{x,g}(\frac{1}{m}\sum_{i\in S_{t}}x^{K+1}_{t,i}-x_{t})-\eta_{x,l}\eta_{x,g}K p(x_t-z_t)$
    \STATE $ y_{t+1}=P_Y(y_{t}+\eta_{y,g}(\frac{1}{m}\sum_{i\in S_{t}}y^{K+1}_{t,i}-y_{t}))$
    \STATE $ z_{t+1}=z_t+\beta(x_{t+1}-z_t)$
  \ENDFOR
  \end{algorithmic}
\end{algorithm}

\subsection{Convergence}
We analyze the convergence behaviors of \alg under the following settings. All proofs are deferred to the appendix. 
\subsubsection{Nonconvex-PL}
Nonconvex-PL is a well-known weaker setting compared with Nonconvex-Strongly-Concave (NC-SC). Thus, the results in this section also hold for NC-SC.
\begin{assume} [PL condition in $y$]
\label{assum:pl}
Assume $X=\mathbb{R}^{d_1}, Y=\mathbb{R}^{d_2}$. 
For any fixed $x \in X$, $\max_{y\in Y} f(x,y)$ has a nonempty solution set and a finite optimal value. 
There exists $\mu > 0$ such that: $
   \Vert \nabla_yf(x,y) \Vert^2 \geq 2\mu [\max_{y}f(x, y)-f(x,y)], \forall x \in X,y\in Y.
$
\end{assume}
We denote $\kappa:=l/\mu$ in this section.
\begin{theorem}
\label{thm:sc}
Under Assumptions \ref{assum:smooth}, \ref{assum:bdd_var}, \ref{assum:bdd_hetero}, \ref{assum:phi} and \ref{assum:pl}, if we apply Algorithm \ref{alg1} with appropriately chosen parameters (see Appendix \ref{app: nc-pl}) and full client participation: $m=M$ or with homogeneous data: $\sigma_G=0$, we can find an $(\epsilon, \epsilon/\sqrt{\kappa})$-stationary point of $f$ with a per-client sample complexity of $O(\kappa^2m^{-1}\epsilon^{-4})$ and a communication complexity of $O(\kappa\epsilon^{-2})$. For partial client participation: $m<M$ and heterogeneous data: $\sigma_G>0$, we can find an $(\epsilon, \epsilon/\sqrt{\kappa})$-stationary point of $f$ with a per-client sample complexity of $O(\kappa^2m^{-1}\epsilon^{-4})$ and a communication complexity of $O(\kappa^2m^{-1}\epsilon^{-4})$.
\end{theorem}

The formal statement and proof of Theorem~\ref{thm:sc} can be found in Appendix \ref{app: nc-pl}.
When $m=M$ or $\sigma_G=0$, we set the number of local updates $K=\Theta(\kappa m^{-1}\epsilon^{-2})$ and can have a communication complexity of $O(\kappa\epsilon^{-2})$. 
However, when $m<M$ and $\sigma_G>0$, our result does not show any convergence benefits from multiple local updates and can set $K=O(1)$ with a communication complexity of $O(\kappa^2m^{-1}\epsilon^{-4})$.
Similar behaviors have also been observed in other federated minimization and minimax works \citep{yang2021achieving, jhunjhunwala2022fedvarp, yang2022sagda,sharma2023federated}. As for the complexity, our per-client sample complexity exhibits a linear speedup w.r.t the number of participated clients.
 
When $M=1$, our results recover the convergence results of Smoothed-AGDA in the centralized setting \citep{yang2022faster}. Similar to \cite{yang2022faster}, we can also translate an $(\epsilon, \epsilon/\sqrt{\kappa})$-stationary point of $f$ to an $\epsilon$-stationary point of $\Phi$ under the federated setting, as stated below.

\begin{prop} [Translation] \quad
\label{Translation} 
 Under Assumptions \ref{assum:smooth}, \ref{assum:bdd_var}, \ref{assum:bdd_hetero}, \ref{assum:phi} and \ref{assum:pl}, if $(\tilde{x}, \tilde{y})$ is an $(\epsilon, \epsilon/\sqrt{\kappa})$-stationary point of $f$, then we can find an $O(\epsilon)$-stationary point of $\Phi$ by solving $\min_x\max_y \{f(x, y) + l\|x-\tilde{x}\|^2\}$ from the initial point $(\tilde{x}, \tilde{y})$ using \alg. When $m=M$ or $\sigma_G=0$,  we need additional $O(\kappa^5 m^{-1} \epsilon^{-2}\log(\kappa))$ per-client sample complexity and $O(\kappa\log(\kappa))$ communication complexity. When $m<M$ and $\sigma_G>0$, we need additional $O(\kappa^5 m^{-1} \epsilon^{-2}\log(\kappa))$ per-client sample complexity and $O(\kappa^5 m^{-1} \epsilon^{-2}\log(\kappa))$ communication complexity.
\end{prop}

With Proposition \ref{Translation}, we have the following corollary.

\begin{corollary} \label{coro:sc full participation}
Under Assumptions \ref{assum:smooth}, \ref{assum:bdd_var}, \ref{assum:bdd_hetero}, \ref{assum:phi} and \ref{assum:pl}, when $m=M$ or $\sigma_G=0$, we can use \alg to find  an $\epsilon$-stationary point of $\Phi$  with a per-client sample complexity of $O(\kappa^2m^{-1}\epsilon^{-4}+\kappa^5 m^{-1} \epsilon^{-2}\log(\kappa))$ and a communication complexity of $O(\kappa\epsilon^{-2}+\kappa\log(\kappa))$.  When $m<M$ and $\sigma_G>0$, we can use \alg to find  an $\epsilon$-stationary point of $\Phi$  with a per-client sample complexity of $O(\kappa^2m^{-1}\epsilon^{-4}+\kappa^5 m^{-1} \epsilon^{-2}\log(\kappa))$ and a communication complexity of $O(\kappa^2m^{-1}\epsilon^{-4}+\kappa^5 m^{-1} \epsilon^{-2}\log(\kappa))$.
\end{corollary}
When $\epsilon$ is small such that $\epsilon\leq \Tilde{O}(\kappa^{-3/2})$, the  sample and communication complexity needed to find an $\epsilon$-stationary point of $\Phi$ have the same order as the  complexity in finding $(\epsilon, \epsilon/\sqrt{\kappa})$-stationary point of $f$. 
{Therefore, in terms of finding an $\epsilon$-stationary point of $\Phi$, our result presents the best-known communication complexity under similar settings. Compared with previous algorithms without variance reduction, we improve the sample complexity by a factor of $O(\kappa^2)$. We also establish additional convergence results in terms of stationarity of $f$.}

\subsubsection{Nonconvex-One-Point-Concave}
Nonconvex-One-Point-Concave (Assumption \ref{assum:1pc_y}) is a weaker setting than Nonconvex-Concave, and is studied in many federated minimax works \citep{deng2021local, sharma2022federated, sharma2023federated}. {We use the following assumptions for this setting.}

\begin{assume}[Compactness in $y$]
	\label{assum:x, bounded y}
	$X = \mathbb{R}^{d_1}$. $Y$ is a convex, compact set of $\mathbb{R}^{d_2}$, and $D(Y)$ denotes the diameter of $Y$.
\end{assume}

\begin{assume}[Lipschitz continuity in $y$]
\label{assum:bounded G_y}
    For any $x\in X, y,y'\in Y$, we have a finite number $G_y$, such that $\|f(x,y)-f(x,y')\|\leq G_y\|y-y'\|.$

\end{assume}
{A similar assumption (Lipschitz continuity in $x$) is also used in \cite{deng2021local, sharma2022federated, sharma2023federated}.}

\begin{assume}[One-Point-Concave in $y$]
   	\label{assum:1pc_y}
    For all $x \in X$, for all $y \in Y$, we have $\langle \nabla_y f(x, y), y - y^*(x) \rangle \leq f(x, y) - f(x, y^*(x)),$  where $y^*(x) \in \argmax_{y\in Y} f(x, y)$. 
\end{assume}

\begin{theorem} 
\label{thm:1pc}
Under Assumptions \ref{assum:smooth}, \ref{assum:bdd_var}, \ref{assum:bdd_hetero}, \ref{assum:phi}, \ref{assum:x, bounded y}, \ref{assum:bounded G_y}, \ref{assum:1pc_y} and $\epsilon^2 \leq l D(Y)$, if we apply Algorithm \ref{alg1} with appropriately chosen parameters (see Appendix \ref{app: nc-1pc}), with full client participation: $m=M$ or with homogeneous data: $\sigma_G=0$, we can find an $(\epsilon, \epsilon^2)$-stationary point of $f$ and an $\epsilon$-stationary point of $\Phi_{1/2l}$ with a per-client sample complexity of $O(m^{-1}\epsilon^{-8})$ and a communication complexity of $O(\epsilon^{-4})$.
\end{theorem}

We achieve comparable sample and communication complexity as the state-of-the-art algorithm Fed-Norm-SGDA+ \citep{sharma2023federated}. However, their proof requires an additional assumption that each local loss function $f_i$ satisfies the NC-1PC condition with a common global minimizer $y^*(x)$, while ours only requires the global loss functions to be NC-1PC. Moreover, several federated minimax works \citep{deng2021local, sharma2022federated, sharma2023federated} assume $\|y_t\|^2\leq D$, but did not specify how to guarantee it. We not only use this assumption but also use the projection operator in our algorithm to achieve this guarantee.

If we set $M=1$, $K=1$, and assume $\sigma=0$, then Problem \eqref{problem} reduces to the centralized deterministic minimax optimization problem and \alg reduces to Smoothed-GDA (Algorithm \ref{alg2}) \citep{zhang2020single}. Additionally, we have the following corollary for Smoothed-GDA under a centralized deterministic NC-1PC setting.

\begin{algorithm}
  \caption{Smoothed-GDA}
  \label{alg2}
  \begin{algorithmic}[1]
  \STATE Input: $x_0,y_0,z_0, \eta_{x}, \eta_{y}, \beta, p, T$
  \FOR{$t=0,1,...,T-1$}
    \STATE $x_{t+1}=x_{t}-\eta_{x}[\nabla_x f(x_{t},y_{t})+p(x_t-z_t)]$
    \STATE $y_{t+1}=P_Y(y_{t}+\eta_{y}\nabla_y f(x_{t},y_{t}))$
    \STATE $ z_{t+1}=z_t+\beta(x_{t+1}-z_t)$
  \ENDFOR
  \end{algorithmic}
\end{algorithm}

\begin{corollary}
\label{thm:1pc, dc}
Under Assumptions \ref{assum:smooth}, \ref{assum:phi}, \ref{assum:x, bounded y}, \ref{assum:bounded G_y}, \ref{assum:1pc_y}, and when $M=1$, $\epsilon^2 \leq l D(Y)$, if we apply Algorithm \ref{alg2} with appropriately chosen parameters (see Appendix \ref{app: nc-1pc}), we can find an $(\epsilon, \epsilon^2)$-stationary point of $f$ and an $\epsilon$-stationary point of $\Phi_{1/2l}$ with a sample complexity of $O(\epsilon^{-4})$. 
\end{corollary} 
Compared to the $O(\epsilon^{-4})$ sample complexity of Smoothed-GDA achieved in \cite{zhang2020single} under NC-C setting, we achieve the same sample complexity under a weaker condition (NC-1PC).

\subsubsection{Nonconvex-Concave}
Since NC-1PC is weaker than NC-C, the results in Theorem \ref{thm:1pc} also hold for NC-C. Moreover, we have improved complexity results in terms of the stationarity of $f$, as presented in this section.

\begin{assume} [Concavity in $y$]
\label{assum:concave}
For all $x \in  X$ and all $ y, y' \in Y$, we have  $f(x,y)\leq f(x,y')+\langle\nabla_y f(x,y'), y-y'\rangle.$

\end{assume}

\begin{theorem} 
\label{thm:c f}
Under Assumptions \ref{assum:smooth}, \ref{assum:bdd_var}, \ref{assum:bdd_hetero}, \ref{assum:phi}, \ref{assum:x, bounded y}, \ref{assum:bounded G_y}, \ref{assum:concave} and $\epsilon \leq 2l D(Y)$, if we apply Algorithm \ref{alg1} to optimize $\Tilde{f}(x,y)=f(x,y)-\frac{\epsilon}{4D(Y)}\|y-y_0\|^2$, $y_0\in Y$ with full client participation: $m=M$ or with homogeneous data: $\sigma_G=0$, we can find an $\epsilon$-stationary point of $f$ with a per-client sample complexity of $O(m^{-1}\epsilon^{-6})$ and a communication complexity of $O(\epsilon^{-3})$.
\end{theorem}
To the best of our knowledge, this is the best-known sample and communication complexity achieved in terms of stationarity of $f$ under similar settings.

Moreover, we have the following corollary for the centralized deterministic setting.
\begin{corollary}
\label{thm:c f, dc}
Under Assumptions \ref{assum:smooth}, \ref{assum:phi}, \ref{assum:x, bounded y}, \ref{assum:bounded G_y}, \ref{assum:concave}, when $M=1$ and $\epsilon \leq 2l D(Y)$, we can apply Algorithm \ref{alg2} to optimize $\Tilde{f}(x,y)=f(x,y)-\frac{\epsilon}{4D(Y)}\|y-y_0\|^2$, $y_0\in Y$, we can find an $\epsilon$-stationary point of $f$ with a sample complexity of $O(\epsilon^{-3})$.
\end{corollary} 
We improve the sample complexity of Smoothed-GDA under centralized deterministic NC-C setting from $O(\epsilon^{-4})$ to $O(\epsilon^{-3})$ in terms of stationarity of $f$.

\subsubsection{Minimizing the Point-Wise Maximum of Finite Functions}
We now consider optimizing $f$ in a form of (\ref{special case}), which is widely used in practical applications. 
\cite{zhang2020single} proved that Smoothed-AGDA can achieve a sample complexity of $O(\epsilon^{-2})$ in terms of stationarity of $f$ for solving \eqref{special case} under centralized and deterministic settings, which is much better than the complexity needed for solving general nonconvex-concave problems. However, to the best of our knowledge, solving \eqref{special case} under stochastic and federated settings remains unexplored. 

For any stationary solution of \eqref{special case} denoted as $(x^*, y^*)$, the following KKT conditions hold: 
\begin{align*}
    & \nabla F(x^*)y^*=0, \\
    & \sum_{i=1}^my_i^*=1, y_i^*\ge 0,\forall i\in [n], \\
    & \lambda-\nu_i=f_i(x^*), \forall i\in[n], \nu_i\ge 0, \\
    & \nu_iy_i^*=0, \forall i\in[n],
\end{align*}
% $\nabla F(x^*)y^*=0, \sum_{i=1}^my_i^*=1, y_i^*\ge 0,\forall i\in [n], \lambda-\nu_i=f_i(x^*), \forall i\in[n], \nu_i\ge 0, \nu_iy_i^*=0, \forall i\in[n]$, 
where $\nabla F(x) $ denotes the Jacobian matrix of $F$ at $x$,  $\lambda$ and $\nu$ are the multipliers for the equality constraint $\sum_{i=1}^n y_i=1$, and the inequality constraint $y_i\ge 0$ respectively. 
We denote $\mathcal{I}_+(y^*)$ as the set of indices for which $y_i^*>0$. We make following assumption on this set.

\begin{assume}[Strict complementarity]
    \label{assum: strict complementarity}
For  any stationary solution $(x^*, y^*)$ of \eqref{special case}, we have $\nu_i>0, \forall i\notin\mathcal{I}_+(y^*)$.
\end{assume}

\begin{remark}
Assumption~\ref{assum: strict complementarity} is commonly used in the optimization literature \citep{forsgren2002interior, carbonetto2009interior, liang2014local, namkoong2016stochastic, lu2019snap, zhang2020single}.
This assumption generally holds if there is a linear term in the objective function and the data is from a continuous distribution \citep{zhang2020proximal, lu2019snap, zhang2020single}.
\end{remark}

\begin{theorem} 
\label{thm:special case}
Under Assumptions \ref{assum:smooth}, \ref{assum:bdd_var}, \ref{assum:bdd_hetero}, \ref{assum:phi}, \ref{assum:bounded G_y}, \ref{assum: strict complementarity}, if we apply Algorithm \ref{alg1} with appropriately chosen parameters (see Appendix \ref{app: special case}) to solve Problem (\ref{special case}), and assume $\|x_t\|\leq D_x$ for all $t$, then with full client participation: $m=M$ or with homogeneous data: $\sigma_G=0$, we can find an $\epsilon$-stationary point of $f$ and $\Phi_{1/2l}$ with a per-client sample complexity of $O(m^{-1}\epsilon^{-4})$ and a communication complexity of $O(\epsilon^{-2})$.
\end{theorem}

To the best of our knowledge, we are the first to prove convergence results for solving Problem~\eqref{special case} under a federated setting.  Setting $M=1$, our results also indicate that we can find an $\epsilon$-stationary point of $f$ and $\Phi_{1/2l}$ of \eqref{special case} with a sample complexity of $O(\epsilon^{-4})$ under the centralized stochastic setting. {Assumptions similar to $\|x_t\|\leq D_x$ are also made in \cite{deng2021local, sharma2022federated, sharma2023federated}. }

\subsubsection{PL-PL}
The PL-PL condition is much weaker than SC-SC and contains a richer class of functions. For example, according to \cite{yang2020global}, $h(x,y)=x^2+3\sin^2x\sin^2y-4y^2-10\sin^2y$ satisfies Assumption \ref{assum: plpl}, \ref{assum: saddle point} (see Proposition 1 in Appendix of \cite{yang2020global}). However, $h(x,y)$ is nonconvex-nonconcave.

\cite{reisizadeh20robustfl_neurips} formulated robust federated learning as a special case of general federated minimax PL-PL problems and proposed FLRA. In their robust federated learning settings, each local client has its own local max variables and FLRA only communicates the min variables between the clients and the server. In this section, we consider a more general federated minimax setting (\ref{problem}) with the PL-PL condition. To the best of our knowledge, we are the first to prove convergence results for this general setting. 

\begin{assume}[Two-sided PL condition]
\label{assum: plpl}
Assume $X=\mathbb{R}^{d_1}, Y=\mathbb{R}^{d_2}$. 
For any fixed $y$, $\min_x f(x,y)$ has a nonempty solution set and a finite optimal value, and for any fixed $x$, $\max_y f(x,y)$ has a nonempty solution set and a finite optimal value.
There exist constants $\mu_1, \mu_2>0$ such that: $\forall x,y$, $\Vert \nabla_xf(x,y) \Vert^2 \geq 2\mu_1 [f(x, y)-\min_{x} f(x,y)]$ and $\Vert \nabla_yf(x,y) \Vert^2 \geq 2\mu_2 [\max_{y}f(x, y)-f(x,y)]$. 
\end{assume} 

\begin{assume}[Existence of saddle point]\quad
\label{assum: saddle point}
    $(x^*, y^*)$ is a saddle point of $f$ if for any $(x,y): f(x^*,y) \leq f(x^*, y^*) \leq  f(x,y^*)$. We assume $f$ has at least one saddle point.
\end{assume}

{Since $f$ is already $\mu_1$-PL in $x$, we set $p=0$ in this section. We further denote $\kappa' =\max\{l/\mu_1,l/\mu_2\}$, $\kappa''=\min\{l/\mu_1,l/\mu_2\}$ in this section.}

\begin{theorem} 
\label{thm: plpl}
Under Assumptions \ref{assum:smooth}, \ref{assum:bdd_var}, \ref{assum:bdd_hetero}, \ref{assum:phi}, \ref{assum: plpl}, \ref{assum: saddle point}, if we apply Algorithm \ref{alg1} with appropriately chosen parameters (see Appendix \ref{app: plpl}) for full client participation: $m=M$ or with homogeneous data: $\sigma_G=0$, we can find $(x_T,y_T)$ satisfying $\mathbb{E}\|x_T-x^*\|^2+\mathbb{E}\|y_T-y^*\|^2\leq \epsilon^2$ with a  per-client sample complexity of $O(m^{-1}\kappa'^3\kappa''^4\epsilon^{-2}\log(\epsilon^{-1}\kappa'))$ and a communication complexity of $O(\kappa'\kappa''^2\log(\epsilon^{-1}\kappa'))$.
For partial client participation: $m<M$ and heterogeneous data: $\sigma_G>0$, we can find $x_T, y_T$ satisfying $\mathbb{E}\|x_T-x^*\|^2+\mathbb{E}\|y_T-y^*\|^2\leq \epsilon^2$ with a  per-client sample complexity of $O(m^{-1}\kappa'^3\kappa''^4\epsilon^{-2}\log(\epsilon^{-1}\kappa'))$ and a communication complexity of $O(m^{-1}\kappa'^3\kappa''^4\epsilon^{-2}\log(\epsilon^{-1}\kappa'))$.
\end{theorem}
When $M=1$, our results recover the convergence results in \cite{yang2020global}.
For full client participation with heterogeneous data, we achieve a better communication complexity compared to \cite{deng2021local} and \cite{reisizadeh20robustfl_neurips}. Moreover, we provide additional convergence results for partial client participation. 

\label{section: exp}
\begin{figure*}[!htp]
    \centering
    \vspace{-0.05in}
    	    \includegraphics[width=15em]{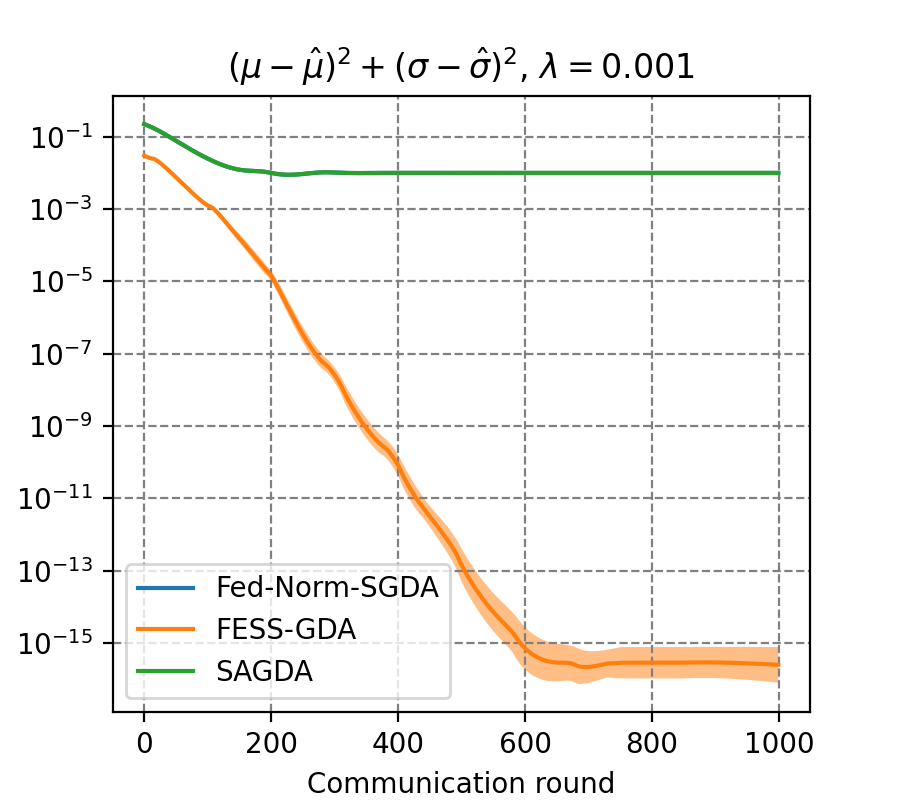}
    	    \includegraphics[width=15em]{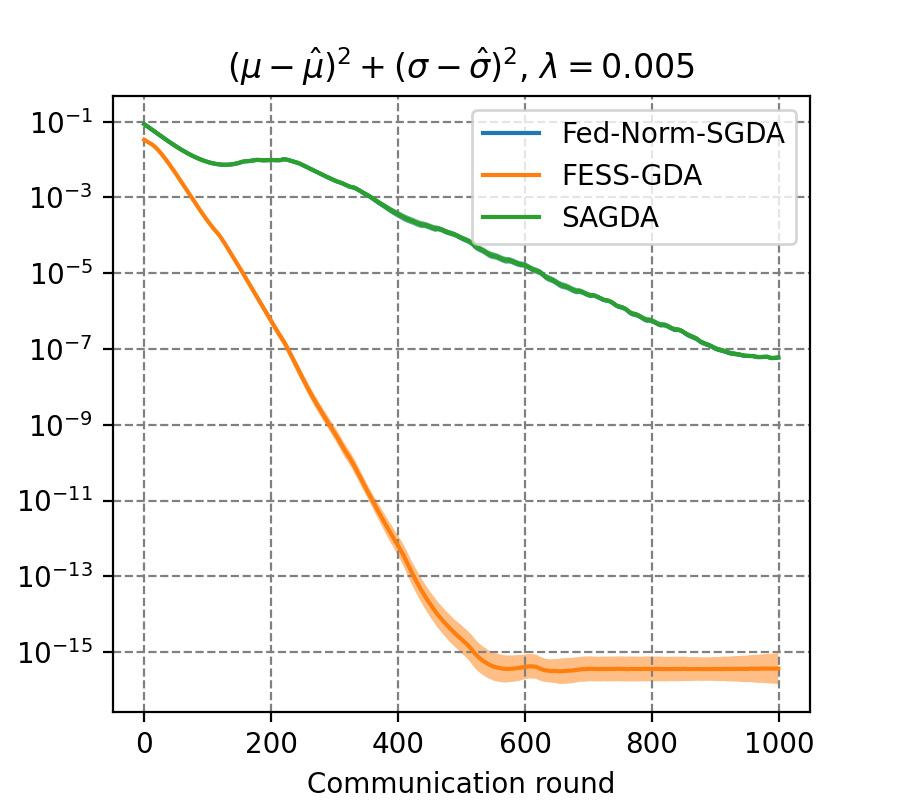}
    	    \includegraphics[width=15em]{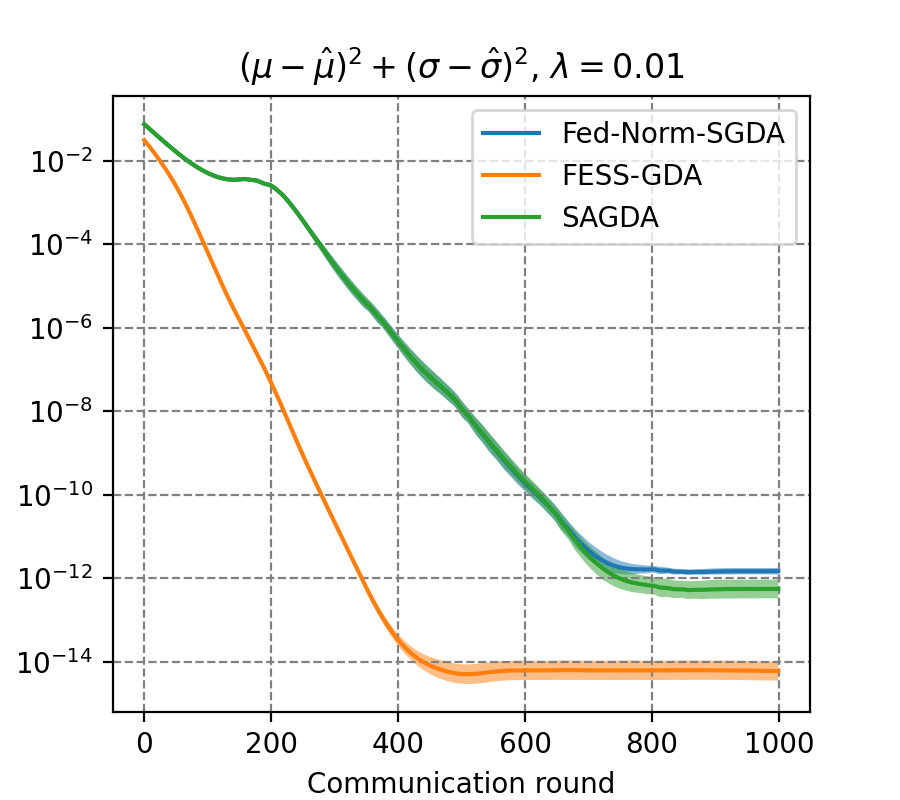}
    \vspace{-0.1in}
    \caption{\small Comparison among Fed-Norm-SGDA, SAGDA and \alg for training a regularized WGAN with different regularization coefficients $\lambda$.} 
    \vspace{-0.1in}
    \label{fig:sc}
\end{figure*}

\section{EXPERIMENTS}

We perform GAN training and fair classification tasks in the federated setting to demonstrate the practical effectiveness and efficiency of \alg and verify our theoretical claims.  We conduct our experiments on a computer with two NVIDIA RTX 3090 GPUs.

\subsection{GAN}
We consider a setting similar to \cite{yang2022faster}, \cite{loizou2020stochastic}, using a Wasserstein GAN~\citep{arjovsky2017wasserstein} to approximate a one-dimensional Gaussian distribution in the federated setting. We first randomly generate a synthetic dataset of $n=10000$ datapoints $z$ sampled from a normal distribution with zero mean and unit variance and their corresponding real data $x^{\text{real}}= \hat{\mu}+\hat\sigma z$, where $\hat{\mu}=0, \hat{\sigma}=0.1$. We then evenly divide them into 10 disjoint sets for 10 clients.  The generator is defined as
$G_{\mu,\sigma}(z) = \mu + \sigma z$ and the discriminator is defined as $D_{\phi_1, \phi_2}(x) =\phi_1 x + \phi_2 x^2$. 
The problem can be formulated as 
\begin{align*}
    \min_{\mu,\sigma}\max_{\phi_1,\phi_2} \Big \{  f(\mu, \sigma, \phi_1, \phi_2) = \frac{1}{n}\sum_{j=1}^n D_{\phi}(x^{\text{real}}_j)-\\    D_{\phi}(G_{\mu,\sigma}(z_j))-\lambda \|\phi\|^2 \Big \},
\end{align*}
where $\lambda>0$ is the regularization coefficient to make the problem strongly concave. 

We set a batch size of 100 for every update, and each client communicates with the server after every 10 local updates. We use the term $(\mu-\hat{\mu})^2+(\sigma-\hat{\sigma})^2$ to measure the algorithm performances.

With $\lambda=0.001, 0.005$ and $0.01$,  we compare the performances among Fed-Norm-SGDA, SAGDA and \alg (see Figure \ref{fig:sc}). We use $\beta=0.05, p=1$ for \alg. For each algorithm, we test their local learning rate from $\{1e-1, 1e-2, 1e-3\}$ and global learning rate from $\{1,2\}$ in order to select the best for each algorithm under different $\lambda$. Each experiment is repeated 5 times and we report the average performance. As we can see from Figure \ref{fig:sc}, \alg achieves a significant speedup over Fed-Norm-SGDA and SAGDA with carefully tuned learning rates under different $\lambda$. Especially, when $\lambda$ is relatively small, the performance gap between Fed-Norm-SGDA, SAGDA and \alg is more pronounced. Note that a smaller $\lambda$ means a larger condition number $\kappa$ (if we assume that the problem has a similar Lipschitz smooth constant $l$ for different $\lambda$). This clearly validates our theoretical results that \alg improves the dependence of $\kappa$ for nonconvex-strongly-concave problems.

\subsection{Fair Classification}
\label{exp: fc}
\begin{figure}[!htp]
    \centering
    \vspace{-0.05in}
    \includegraphics[width=20em]{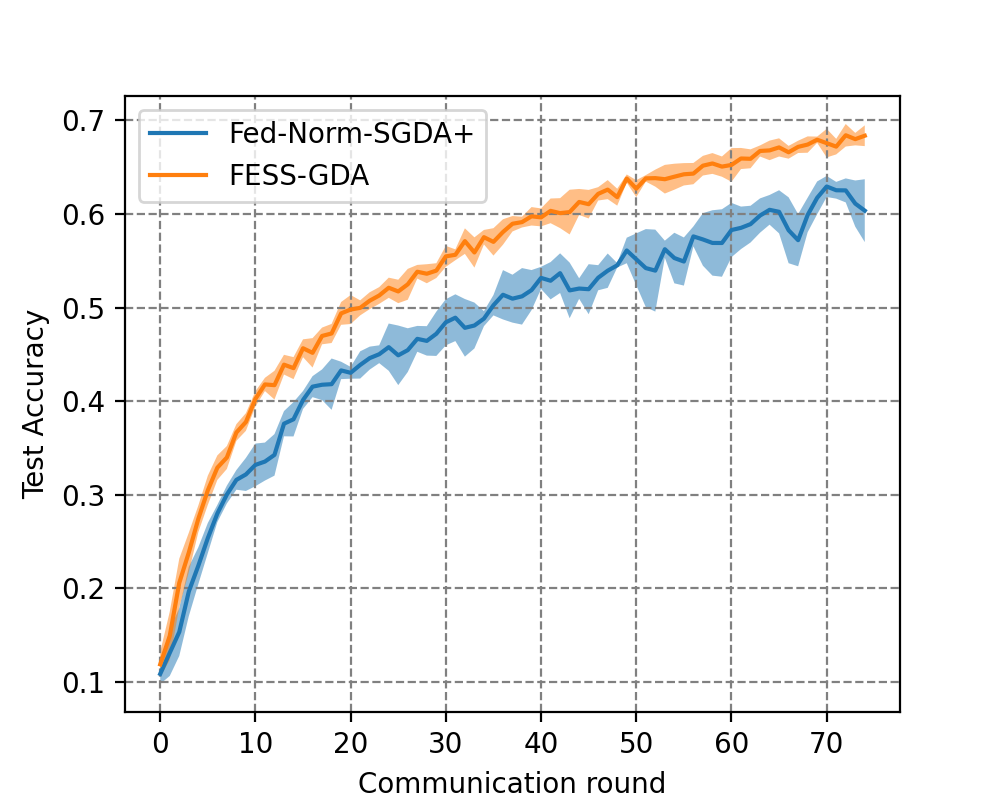}
    \vspace{-0.1in}
    \caption{\small  Comparison between Fed-Norm-SGDA+ and \alg for the fair classification task on CIFAR-10. }
    \vspace{-0.1in}
    \label{fig:c}
\end{figure}
We consider a similar setting as \cite{wu2023solving, sharma2022federated, nouiehed2019solving}. The fair classification problem can be formulated as 
\begin{align*}
    \min_{x}\max_{y\in Y}\sum_{c=1}^C F_c(x)y_c, 
\end{align*}
where $Y=\{(y_1, ..., y_C)^T | \sum_{c=1}^C y_c=1, y_c\geq 0\}$, $x$ is the parameters of the model, and $F_c$ is the loss function of class $c$. Clearly, this problem has the same form as \eqref{special case} and is nonconvex-concave. We run the experiment on the CIFAR-10 dataset~\citep{krizhevsky2009learning} with a convolutional neural network. We evenly divide the dataset into 10 disjoint sets for 10 clients. We compare the performances of Fed-Norm-SGDA+ and \alg for solving this problem and use the test accuracy as the performance measure. We set a batch size of 100 and inner loop $K=20$. For both algorithms, we adjust their local learning rate from $\{1e-1,1e-2\}$ and global learning rate from $\{1,1.5\}$. For \alg, we adjust its $\beta$ from $\{0.1, 0.5, 0.9\}$ and its $p$ from $\{0, 1e-2, 1e-1\}$. For Fed-Norm-SGDA+, we adjust its $S$ from $\{1, 5, 10, 20\}$. We tune all the parameters to achieve the best empirical performance for both algorithms. {Each experiment is repeated 5 times and we report the average performance.} As we can see from Figure \ref{fig:c}, \alg achieves a better performance than Fed-Norm-SGDA+. 

\section{CONCLUSION}
In this paper, we have proposed a new federated minimax optimization algorithm named \alg. We showed that \alg can be uniformly used for solving different classes of federated nonconvex minimax problems and theoretically established new or better convergence results for the considered settings. We further showcased the practical efficiency of \alg in practical federated learning tasks of training GANs and fair classification tasks.

\section*{Acknowledgements}
The work of WS and CS was supported in part by the US National Science Foundation (NSF) under awards ECCS-2033671, ECCS-2143559, CPS-2313110, CNS-2002902, and Virginia Commonwealth Cyber Initiative Innovation and Commercialization Award VV-1Q23-005. The work of JZ was partially supported by MIT Postdoctoral Fellowship for Engineering Excellence 2023-2025.

\bibliography{ref}
%%%%%%%%%%%%%%%%%%%%%%%%%%%%%%%%%%%%%%%%%%%%%%%%%%%%%%%%%%%%
 \section*{Checklist}

 \begin{enumerate}

  \item For all models and algorithms presented, check if you include:
  \begin{enumerate}
    \item A clear description of the mathematical setting, assumptions, algorithm, and/or model. [Yes]
    \item An analysis of the properties and complexity (time, space, sample size) of any algorithm. [Yes]
    \item (Optional) Anonymized source code, with specification of all dependencies, including external libraries. [Yes]
  \end{enumerate}

  \item For any theoretical claim, check if you include:
  \begin{enumerate}
    \item Statements of the full set of assumptions of all theoretical results. [Yes] See Sections~\ref{section: pre}, \ref{section: alg}.
    \item Complete proofs of all theoretical results. [Yes] All proofs are in Appendix.
    \item Clear explanations of any assumptions. [Yes] See Sections~\ref{section: pre}, \ref{section: alg}.
  \end{enumerate}

  \item For all figures and tables that present empirical results, check if you include:
  \begin{enumerate}
    \item The code, data, and instructions needed to reproduce the main experimental results (either in the supplemental material or as a URL). [Yes]
    \item All the training details (e.g., data splits, hyperparameters, how they were chosen). [Yes] See Section~\ref{section: exp}.
     \item A clear definition of the specific measure or statistics and error bars (e.g., with respect to the random seed after running experiments multiple times). [Yes] See Section~\ref{section: exp}.
     \item A description of the computing infrastructure used. (e.g., type of GPUs, internal cluster, or cloud provider). [Yes] See Section~\ref{section: exp}.
  \end{enumerate}

  \item If you are using existing assets (e.g., code, data, models) or curating/releasing new assets, check if you include:
  \begin{enumerate}
    \item Citations of the creator If your work uses existing assets. [Yes] We used the CIFAR10 dataset and cited it.
    \item The license information of the assets, if applicable. [Not Applicable]
    \item New assets either in the supplemental material or as a URL, if applicable. [Yes]
    \item Information about consent from data providers/curators. [Not Applicable]
    \item Discussion of sensible content if applicable, e.g., personally identifiable information or offensive content. [Not Applicable]
  \end{enumerate}

  \item If you used crowdsourcing or conducted research with human subjects, check if you include:
  \begin{enumerate}
    \item The full text of instructions given to participants and screenshots. [Not Applicable]
    \item Descriptions of potential participant risks, with links to Institutional Review Board (IRB) approvals if applicable. [Not Applicable]
    \item The estimated hourly wage paid to participants and the total amount spent on participant compensation. [Not Applicable]
  \end{enumerate}

  \end{enumerate}

\clearpage
\onecolumn
\appendix
\allowdisplaybreaks
% \aistatstitle{Stochastic Smoothed Gradient Descent Ascent for Federated Minimax Optimization}
% \aistatstitle{Supplementary Material: Stochastic Smoothed Gradient Descent Ascent for Federated Minimax Optimization}

\hsize\textwidth
\linewidth\hsize \toptitlebar {\centering
{\Large\bfseries Supplementary Material: Stochastic Smoothed Gradient Descent Ascent for Federated Minimax Optimization \par}}
\bottomtitlebar

The supplementary material is organized as follows. In Section \ref{app: notations}, we introduce notations that will be used throughout the supplementary material. In Section \ref{app: lemmas}, we present some preliminary lemmas. In Section \ref{app: Potential Function}, we introduce necessary lemmas of the potential function for NC-PL and NC-1PC. In the subsequent sections, we provide the convergence results of \alg for NC-PL functions (Section \ref{app: nc-pl}), NC-SC functions (Section \ref{app: nc-sc}), NC-1PC functions (Section \ref{app: nc-1pc}), NC-C functions (Section \ref{app: nc-c}), functions having a form of \eqref{special case} (Section \ref{app: special case}), and PL-PL functions (Section \ref{app: plpl}). In Section \ref{app: Translation}, we prove Proposition \ref{Translation}. Finally, in Section \ref{app: exp}, we provide additional results and details of our experiments.

\section{Notations}
\label{app: notations}
We introduce the following notations, which will play a significant role in our proof.
\begin{align*}
    &\hat{f}(x,y,z)=f(x,y)+\frac{p}{2}\|x-z\|^2,\\
    &\Psi(y,z)=\min_{x \in X} \hat{f}(x,y,z),\\
    &\Phi(x)=\max_{y \in Y} f(x,y),\\
    &\Phi(x,z)=\max_{y \in Y} \hat{f}(x,y,z),,\\
    &P(z)=\min_{x \in X}\max_{y \in Y}\hat{f}(x,y,z),\\
    &V _t =  V (x_t, y_t, z_t) = \hat{f}(x_t, y_t,z_t) - 2\Psi(y_t, z_t) + 2P(z_t),\\
    &x^*(y,z)=\arg\min_{x \in X}\hat{f}(x,y,z),\\
    &x^*(z)=\arg\min_{x \in X}\Phi(x,z),\\
    &y^*(x)\in\argmax_{y \in Y}f(x,y),\\
    &\hat{y}^*(z) \in\argmax_{y \in Y}\Psi(y,z),\\
    &x^+(y, z)=x-\eta_xK\nabla_x\hat{f}(x, y, z ),\\
    &y^+(z)=P_Y(y+\eta_yK\nabla_y f(x^*(y,z),y)).
\end{align*}
We denote $w_t=(x_t, y_t)$, $\eta_x=\eta_{x,g}\eta_{x,l}$, $\eta_y=\eta_{y,g}\eta_{y,l}$ for simplicity.

\noindent We summarize the main updates of \alg as:
\begin{align*}
    &x_{t+1}=x_t-\eta_xK[u_{x,t}-e_{x,t}+ p(x_t-z_t)],\\
    &y_{t+1}=P_Y(y_t+\eta_yK(u_{y,t}-e_{y,t})),\\
    &z_{t+1}=z_t+\beta(x_{t+1}-z_t),\\
    &u_{x,t}=\frac{1}{m}\sum_{i \in S_t}\nabla_x f_i(w_t),\\
    &u_{y,t}=\frac{1}{m}\sum_{i \in S_t}\nabla_y f_i(w_t),\\
    &e_{x,t}=\frac{1}{mK}\sum_{i\in S_t}\sum_{j\in [K]}\left(\nabla_x f_i(w_t)-\nabla_x f_i(w^j_{t,i},\xi^j_{t,i})\right),\\
    &\Bar{e}_{x,t}=\mathbb{E}[e_{x,t}]=\frac{1}{mK}\sum_{i\in S_t}\sum_{j\in [K]}\left(\nabla_x f_i(w_t)-\nabla_x f_i(w^j_{t,i})\right),\\
    &e_{y,t}=\frac{1}{mK}\sum_{i\in S_t}\sum_{j\in [K]}\left(\nabla_y f_i(w_t)-\nabla_y f_i(w^j_{t,i},\xi^j_{t,i})\right),\\
    &\Bar{e}_{y,t}=\mathbb{E}[e_{y,t}]=\frac{1}{mK}\sum_{i\in S_t}\sum_{j\in [K]}\left(\nabla_y f_i(w_t)-\nabla_y f_i(w^j_{t,i})\right).
\end{align*}
We further define the following notations
\begin{align*}
    &d_{x,t}=\mathbb{E}\|\nabla_x \hat{f}(w_t,z_t)-u_{x,t}+e_{x,t}-p(x_t-z_t)\|^2,\\
    &d_{y,t}=\mathbb{E}\|\nabla_y f(w_t)-u_{y,t}+e_{y,t}\|^2.
\end{align*}
Define $\Bar{y}_{t+1}=P_Y(y_t+\eta_y K \nabla_y f(w_t))$, when $Y=\mathbb{R}^{d_2}$, we have $\Bar{y}_{t+1}=y_t+\eta_yK\nabla_yf(w_t)$. Define $\Phi^*=\min_{x\in X}\max_{y\in Y}f(x,y),\quad \Delta=V_0-\Phi^*$. Because $V(x,y,z)=P(z)+(f(x,y,z)-\Psi(y,z))+(P(z)-\Psi(y,z))\geq P(z)\geq \Phi^*$, we have 
\begin{align}\label{v_0-v_t}
    V_0 - V_t \leq V_0 - \min V_t \leq V_0-\Phi^*=\Delta.
\end{align}

\section{Preliminary Lemmas}
\label{app: lemmas}
\begin{lemma}[Lemma C.1 \citep{yang2022faster}]
\label{lemma: helper lemma of optimal x}
    When $p>l$, we have
	\begin{align*}
		&\|x^*(y, z)- x^*(y, z^\prime)\| \leq \gamma_1 \|z-z^\prime\|, \\
		&\|x^*(z) - x^*(z^\prime) \leq \gamma_1\|z-z^\prime\|,\\
		&\|x^*(y,z) - x^*(y^\prime, z)\| \leq \gamma_2\|y - y^\prime\|,
	\end{align*}
	where $\gamma_1 = \frac{p}{-l+p}$, $\gamma_2 = \frac{l+p}{-l+p}$.
\end{lemma}

\begin{lemma}[\cite{karimi2016linear}]
\label{lemma: property of sc}
If function $g(x)$ is $l$-smooth and satisfies the PL condition with constant $\mu$,  then the following conditions hold
\begin{align*}
    g(x) - \min_{z} g(z) & \geq \frac{\mu}{2} \|x_p - x\|^2, \\
    \|\nabla_x g(x)\|^2 & \geq 2 \mu (g(x) - \min_z g(z) ),\\
    \|\nabla_x g(x)\|^2 &\geq  \mu\|x_p - x\|^2,
\end{align*}
where $x_p$ is the projection of $x$ onto the optimal set.
\end{lemma}

\begin{lemma}
When $p=2l$, we have
\label{stationary of Phi}
    $$\|\nabla_x \Phi(x^*(x_t))\|=\|\nabla_x \Phi_{1/2l}(x_t)\|=p\|x_t-x^*(x_t)\|.$$
\end{lemma}
\begin{proof}
    Note that $x^*(x_t)=\argmin_x\{\max_y f(x,y)+\frac{p}{2}\|x-x_t\|^2\}$. According to Lemma A.4 in \citet{yang2022faster}, we have $\|\nabla_x \Phi(x^*(x_t))\|=\|\nabla_x \Phi_{1/2l}(x_t)\|=p\|x_t-x^*(x_t)\|$.
\end{proof}

\begin{lemma}
\label{lemma: bound e}
When the local step sizes $\eta_{x,l}, \eta_{y,l}$ satisfy
    \begin{align*}
        \eta_{x,l}\leq& \frac{1}{2l\sqrt{2(2K-1)(K-1)}},\\
        \eta_{y,l}\leq& \frac{1}{2l\sqrt{2(2K-1)(K-1)}},
    \end{align*} 
the following inequalities hold:
    \begin{align*}
        &\mathbb{E}\|\Bar{e}_{x,t}\|^2\leq l^2[24K^2\eta_{x,l}^2\mathbb{E}\|\nabla_x f(w_{t})\|^2+24K^2\eta_{y,l}^2\mathbb{E}\|\nabla_y f(w_{t})\|^2+24K^2(\eta_{x,l}^2+\eta_{y,l}^2)\sigma_G^2+3K(\eta_{x,l}^2+2K\eta_{y,l}^2)\sigma^2],\\
        &\mathbb{E}\|\Bar{e}_{y,t}\|^2\leq l^2[24K^2\eta_{x,l}^2\mathbb{E}\|\nabla_x f(w_{t})\|^2+24K^2\eta_{y,l}^2\mathbb{E}\|\nabla_y f(w_{t})\|^2+24K^2(\eta_{x,l}^2+\eta_{y,l}^2)\sigma_G^2+3K(\eta_{x,l}^2+2K\eta_{y,l}^2)\sigma^2].
    \end{align*}
\end{lemma}

\begin{proof}
According to the definition of $\|\Bar{e}_{x,t}\|$, we have
    \begin{align*}
        &\mathbb{E}\|w_t-w_{t,i}^{j+1}\|^2\\
        =&\mathbb{E}\|x^j_{t,i}-\eta_{x,l}\nabla_x f_i(w^j_{t,i},\xi^j_{t,i})-x_t\|^2+\mathbb{E}\|P_Y(y^j_{t,i}+\eta_{y,l}\nabla_y f_i(w^j_{t,i},\xi^j_{t,i}))-y_t\|^2\\
        \overset{(a)}{\leq}&\mathbb{E}\|x^j_{t,i}-x_t-\eta_{x,l}\nabla_x f_i(w^j_{t,i})\|^2+\mathbb{E}\|P_Y(y^j_{t,i}+\eta_{y,l}\nabla_y f_i(w^j_{t,i},\xi^j_{t,i}))-y_t\|^2+\eta_{x,l}^2\sigma^2\\
        \overset{(b)}{\leq}&\left(1+\frac{1}{2K-1}\right)\mathbb{E}\|x_{t,i}^j-x_t\|^2+2K\eta_{x,l}^2\|\nabla_x f_i(w_{t,i}^j)\|^2+\left(1+\frac{1}{2K-1}\right)\mathbb{E}\|y_{t,i}^j-y_t\|^2+\\
        &2K\mathbb{E}\|P_Y(y_{t,i}^j+\eta_{y,l}\nabla_y f_i(w_{t,i}^j, \xi^j_{t,i}))-y_{t,i}^j\|^2+\eta_{x,l}^2\sigma^2\\
        \overset{(c)}{\leq}&\left(1+\frac{1}{2K-1}\right)\mathbb{E}\|w_{t,i}^j-w_t\|^2+2K\eta_{x,l}^2\mathbb{E}\|\nabla_x f_i(w_{t,i}^j)\|^2+2K\eta_{y,l}^2\mathbb{E}\|\nabla_y f_i(w_{t,i}^j, \xi^j_{t,i})\|^2+\eta_{x,l}^2\sigma^2\\
        \overset{(d)}{\leq} & \left(1+\frac{1}{2K-1}\right)\mathbb{E}\|w_{t,i}^j-w_t\|^2+2K\eta_{x,l}^2\mathbb{E}\|\nabla_x f_i(w_{t,i}^j)\|^2+2K\eta_{y,l}^2\mathbb{E}\|\nabla_y f_i(w_{t,i}^j)\|^2+(\eta_{x,l}^2+2K\eta_{y,l}^2)\sigma^2\\
        \leq&\left(1+\frac{1}{2K-1}\right)\mathbb{E}\|w_{t,i}^j-w_t\|^2+4K\eta_{x,l}^2\mathbb{E}\|\nabla_x f_i(w_{t})\|^2+4K\eta_{y,l}^2\mathbb{E}\|\nabla_y f_i(w_{t})\|^2+(\eta_{x,l}^2+2K\eta_{y,l}^2)\sigma^2+\\
        &4K\eta_{x,l}^2\mathbb{E}\|\nabla_x f_i(w_t)-\nabla_x f_i(w_{t,i}^j)\|^2+4K\eta_{y,l}^2\mathbb{E}\|\nabla_y f_i(w_t)-\nabla_y f_i(w_{t,i}^j)\|^2\\
        \overset{(e)}{\leq}&\left(1+\frac{1}{2K-1}\right)\mathbb{E}\|w_{t,i}^j-w_t\|^2+4K\eta_{x,l}^2\mathbb{E}\|\nabla_x f_i(w_{t})\|^2+4K\eta_{y,l}^2\mathbb{E}\|\nabla_y f_i(w_{t})\|^2+(\eta_{x,l}^2+2K\eta_{y,l}^2)\sigma^2+\\
        &4K\eta_{x,l}^2l^2\|w_t-w_{t,i}^j\|^2+4K\eta_{y,l}^2l^2\|w_t-w_{t,i}^j\|^2\\
        =&\left(1+\frac{1}{2K-1}+4Kl^2(\eta_{x,l}^2+\eta_{y,l}^2)\right)\mathbb{E}\|w_{t,i}^j-w_t\|^2+4K\eta_{x,l}^2\mathbb{E}\|\nabla_x f_i(w_{t})\|^2+4K\eta_{y,l}^2\mathbb{E}\|\nabla_y f_i(w_{t})\|^2+\\
        &(\eta_{x,l}^2+2K\eta_{y,l}^2)\sigma^2\\
        \overset{(f)}{\leq}&\left(1+\frac{1}{K-1}\right)\mathbb{E}\|w_{t,i}^j-w_t\|^2+4K\eta_{x,l}^2\mathbb{E}\|\nabla_x f_i(w_{t})\|^2+4K\eta_{y,l}^2\mathbb{E}\|\nabla_y f_i(w_{t})\|^2+(\eta_{x,l}^2+2K\eta_{y,l}^2)\sigma^2\\
        \overset{(g)}{\leq}&\sum_{\tau=0}^{j-1}\left(1+\frac{1}{K-1}\right)^\tau[4K\eta_{x,l}^2\mathbb{E}\|\nabla_x f_i(w_{t})\|^2+4K\eta_{y,l}^2\mathbb{E}\|\nabla_y f_i(w_{t})\|^2+(\eta_{x,l}^2+2K\eta_{y,l}^2)\sigma^2]\\
        \overset{(h)}{\leq}&12K^2\eta_{x,l}^2\mathbb{E}\|\nabla_x f_i(w_{t})\|^2+12K^2\eta_{y,l}^2\mathbb{E}\|\nabla_y f_i(w_{t})\|^2+3K(\eta_{x,l}^2+2K\eta_{y,l}^2)\sigma^2\\
        \overset{(i)}{\leq}&24K^2\eta_{x,l}^2\mathbb{E}\|\nabla_x f(w_{t})\|^2+24K^2\eta_{y,l}^2\mathbb{E}\|\nabla_y f(w_{t})\|^2+24K^2(\eta_{x,l}^2+\eta_{y,l}^2)\sigma_G^2+3K(\eta_{x,l}^2+2K\eta_{y,l}^2)\sigma^2.
    \end{align*}
    $(a), (d)$ are a consequence of the bounded variance of the stochastic gradient. $(b)$ arises from the property that $\|a+b\|^2\leq (1+c)\|a\|^2+(1+1/c)\|b\|^2, c>0$. $(c)$ is a result of the nonexpansiveness of the projection operator. $(e)$ is derived from the $l$-smoothness of the function $f$, while $(f)$ is established based on the condition:
    \begin{align*}
        l^2\eta_{x,l}^2\leq& \frac{1}{8(2K-1)(K-1)},\\
        l^2\eta_{y,l}^2\leq& \frac{1}{8(2K-1)(K-1)}.
    \end{align*}
    $(h)$ is due to
    \begin{align*}
        \sum_{\tau=0}^{j-1}\left(1+\frac{1}{K-1}\right)^\tau\leq& (K-1)\left(\frac{K}{K-1}\right)^K\\
        \frac{1}{K}\sum_{\tau=0}^{j-1}\left(1+\frac{1}{K-1}\right)^\tau \leq & \left(\frac{K}{K-1}\right)^{K-1}\leq e\leq 3\\
        \sum_{\tau=0}^{j-1}\left(1+\frac{1}{K-1}\right)^\tau \leq &3K.
    \end{align*}
    $(i)$ is from Assumption \ref{assum:bdd_hetero}. We thus have
    \begin{align*}
        &\mathbb{E}\|\Bar{e}_{x,t}\|^2\leq l^2[24K^2\eta_{x,l}^2\mathbb{E}\|\nabla_x f(w_{t})\|^2+24K^2\eta_{y,l}^2\mathbb{E}\|\nabla_y f(w_{t})\|^2+24K^2(\eta_{x,l}^2+\eta_{y,l}^2)\sigma_G^2+3K(\eta_{x,l}^2+2K\eta_{y,l}^2)\sigma^2].
    \end{align*}
    The inequality for $y$ can be proven in a similar fashion. 
    % Similarly, the following inequality for $y$ can be proven in a similar fashion:
    % \begin{align*}
    %     &\mathbb{E}\|\Bar{e}_{y,t}\|^2\leq l^2[24K^2\eta_{x,l}^2\mathbb{E}\|\nabla_x f(w_{t})\|^2+24K^2\eta_{y,l}^2\mathbb{E}\|\nabla_y f(w_{t})\|^2+24K^2(\eta_{x,l}^2+\eta_{y,l}^2)\sigma_G^2+3K(\eta_{x,l}^2+2K\eta_{y,l}^2)\sigma^2]
    % \end{align*}
\end{proof}

\begin{lemma}
The following inequalities establish upper bounds for $d_{x,t}$ and $d_{y,t}$:
\label{lemma: bound d}
        \begin{align*}
        d_{x,t}=&\mathbb{E}\|\nabla_x \hat{f}(w_t,z_t)-u_{x,t}+e_{x,t}-p(x_t-z_t)\|^2
        \leq2\mathbb{E}\|\Bar{e}_{x,t}\|^2+4\left(M-m\right)\frac{\sigma_G^2}{mM}+\frac{2}{mK}\sigma^2,\\
        d_{y,t}=&\mathbb{E}\|\nabla_y f(w_t)-u_{y,t}+e_{y,t}\|^2\leq2\mathbb{E}\|\Bar{e}_{y,t}\|^2+4\left(M-m\right)\frac{\sigma_G^2}{mM}+\frac{2}{mK}\sigma^2.
    \end{align*}
\end{lemma}
\begin{proof}
According to the definition of $d_{x,t}$, we have
    \begin{align*}
        d_{x,t}=&\mathbb{E}\|\nabla_x \hat{f}(w_t,z_t)-u_{x,t}+e_{x,t}-p(x_t-z_t)\|^2\\
        =&\mathbb{E}\|\nabla_x f(w_t)-u_{x,t}+e_{x,t}\|^2\\
        \leq&2\mathbb{E}\|\nabla_x f(w_t)-u_{x,t}\|^2+2\mathbb{E}\|e_{x,t}\|^2\\
        \leq&2\mathbb{E}\left\|\frac{1}{M}\sum_{j\in [M]}\nabla_x f_j(w_t)-\frac{1}{m}\sum_{i \in S_t}\nabla_x f_i(w_t)\right\|^2+2\mathbb{E}\|\Bar{e}_{x,t}\|^2+\frac{2}{mK}\sigma^2\\
        \leq&2\mathbb{E}\left\|\left(\frac{1}{M}-\frac{1}{m}\right)\sum_{j\in S_t}\nabla_x f_j(w_t)-m\left(\frac{1}{M}-\frac{1}{m}\right)\nabla_x f(w_t)+\frac{1}{M}\sum_{i \in [M]/S_t}\nabla_x f_i(w_t)-\frac{M-m}{M}\nabla_x f(w_t)\right\|^2+\\
        &2\mathbb{E}\|\Bar{e}_{x,t}\|^2+\frac{2}{mK}\sigma^2\\
        \leq&4\mathbb{E}\left\|\left(\frac{1}{M}-\frac{1}{m}\right)\sum_{j\in S_t}\nabla_x f_j(w_t)-m\left(\frac{1}{M}-\frac{1}{m}\right)\nabla_x f(w_t)\right\|^2+\\
        &4\mathbb{E}\left\|\frac{1}{M}\sum_{i \in [M]/S_t}\nabla_x f_i(w_t)-\frac{M-m}{M}\nabla_x f(w_t)\right\|^2+2\mathbb{E}\|\Bar{e}_{x,t}\|^2+\frac{2}{mK}\sigma^2\\
        \overset{(a)}{\leq}&4\left[m\left(\frac{1}{M}-\frac{1}{m}\right)^2+(M-m)\frac{1}{M^2}\right]\sigma_G^2+2\mathbb{E}\|\Bar{e}_{x,t}\|^2+\frac{2}{mK}\sigma^2\\
        \leq&2\mathbb{E}\|\Bar{e}_{x,t}\|^2+4\left(M-m\right)\frac{\sigma_G^2}{mM}+\frac{2}{mK}\sigma^2,
    \end{align*}where $(a)$ is due to $\|\sum_i^k x_i\|^2\leq k\sum_i^k\|x_i\|^2$ and Assumption \ref{assum:bdd_hetero}.
    Similarly, we have
    \begin{align*}
        d_{y,t}=&\mathbb{E}\|\nabla_y f(w_t)-u_{y,t}+e_{y,t}\|^2\leq2\mathbb{E}\|\Bar{e}_{y,t}\|^2+4\left(M-m\right)\frac{\sigma_G^2}{mM}+\frac{2}{mK}\sigma^2.
    \end{align*}
\end{proof}

\begin{lemma}
Under the update rule of \alg, we have
\label{lemma: bound |x_t-x_{t+1}|}
    \begin{align*}
    \mathbb{E}\|x_{t+1}-x_t\|^2\leq&2\eta_x^2K^2\mathbb{E}\|\nabla_x \hat{f}(w_t,z_t)\|^2+2\eta_x^2K^2d_{x,t},\\
    \mathbb{E}\|y_{t+1}-y_t\|^2
    \leq&2\mathbb{E}\|\Bar{y}_{t+1}-y_t\|^2+2\eta_y^2K^2d_{y,t}.
\end{align*}
When $Y=\mathbb{R}^{d_2}$, we have $\mathbb{E}\|y_{t+1}-y_t\|^2\leq2\eta_y^2K^2\mathbb{E}\|\nabla_y f(w_t)\|^2+2\eta_y^2K^2d_{y,t}$.
\end{lemma}
\begin{proof}
According to the update rule of $x_t, y_t$, we have
    \begin{align*}
    \mathbb{E}\|x_{t+1}-x_t\|^2=&\eta_x^2K^2\mathbb{E}\|u_{x,t}-e_{x,t}-p(x_t-z_t)\|^2\\
    \leq& 2\eta_x^2K^2\mathbb{E}\|\nabla_x \hat{f}(w_t,z_t)\|^2+2\eta_x^2K^2\mathbb{E}\|\nabla_x \hat{f}(w_t,z_t)-u_{x,t}+e_{x,t}-p(x_t-z_t)\|^2\\
    =&2\eta_x^2K^2\mathbb{E}\|\nabla_x \hat{f}(w_t,z_t)\|^2+2\eta_x^2K^2d_{x,t},\\
    \mathbb{E}\|y_{t+1}-y_t\|^2=&\mathbb{E}\|P_Y(y_t+\eta_y K(u_{y,t}-e_{y,t}))-y_t\|^2\\
    \overset{(a)}{\leq}& 2\mathbb{E}\|P_Y(y_t+\eta_y K\nabla_y f(w_t))-y_t\|^2+2\eta_y^2K^2\|\nabla_y f(w_t)-u_{y,t}+e_{y,t}\|^2\\
    =&2\mathbb{E}\|\Bar{y}_{t+1}-y_t\|^2+2\eta_y^2K^2d_{y,t},
\end{align*} where $(a)$ is due to the nonexpansiveness of the projection operator.
\end{proof}

\section{Intermediate Lemmas for Potential Function}
\label{app: Potential Function}
Recall the potential function is defined as
\begin{align*}
V _t =  V (x_t, y_t, z_t) = \hat{f}(x_t, y_t,z_t) - 2\Psi(y_t, z_t) + 2P(z_t).
\end{align*} 
The outline of the convergence proof for \alg aims to demonstrate the monotonic decrease of $V_t$. In this section, we present the essential lemmas required to establish bounds on the potential function.
\begin{lemma}
\label{lemma: f(., z_t)}
When $\eta_x\leq \frac{1}{4(p+l)K}$, we have the following inequality:
    \begin{align*}
        &\mathbb{E}\hat{f}(w_t, z_t)-\mathbb{E}\hat{f}(w_{t+1},z_t)\\
        \geq&\frac{\eta_xK}{4}\|\nabla_x \hat{f}(w_t,z_t)\|^2-\frac{\eta_x K}{2}\|\Bar{e}_{x,t}\|^2-(p+l)\eta_x^2K^2d_{x,t}+\mathbb{E} \langle \nabla_y f(w_t), y_{t} - y_{t+1}\rangle -\frac{p+l}{2}\mathbb{E}\|y_{t+1}-y_t\|^2.
    \end{align*}
\end{lemma}
\begin{proof}
Because of the $(p+l)$-smoothness of $\hat{f}(\cdot, z)$, we have
    \begin{align*}
    &\mathbb{E}\hat{f}(w_t, z_t)-\mathbb{E}\hat{f}(w_{t+1},z_t)\\
    \geq&\mathbb{E}\langle \nabla_x \hat{f}(w_t, z_t), x_{t} - x_{t+1}\rangle +\mathbb{E}\langle \nabla_y \hat{f}(w_t, z_t), y_{t} - y_{t+1}\rangle - \frac{p+l}{2}\mathbb{E}\|x_{t+1}-x_t\|^2-\frac{p+l}{2}\mathbb{E}\|y_{t+1}-y_t\|^2\\
    = &  \eta_x K \mathbb{E} \langle \nabla_x \hat{f}(w_t, z_t), u_{x,t}+p(x_t-z_t)-\Bar{e}_{x,t}\rangle +\mathbb{E} \langle \nabla_y f(w_t), y_{t} - y_{t+1}\rangle - \frac{p+l}{2}\mathbb{E}\|x_{t+1}-x_t\|^2-\frac{p+l}{2}\mathbb{E}\|y_{t+1}-y_t\|^2\\
    = &  \eta_x K \mathbb{E} \langle \nabla_x \hat{f}(w_t, z_t), \nabla_x \hat{f}(w_t, z_t)-\Bar{e}_{x,t}\rangle +\mathbb{E} \langle \nabla_y f(w_t), y_{t} - y_{t+1}\rangle - \frac{p+l}{2}\mathbb{E}\|x_{t+1}-x_t\|^2-\frac{p+l}{2}\mathbb{E}\|y_{t+1}-y_t\|^2\\
    =&\eta_x K \mathbb{E}\|\nabla_x \hat{f}(w_t,z_t)\|^2+\frac{\eta_x K}{2}\mathbb{E}\|\nabla_x \hat{f}(w_t,z_t)-\Bar{e}_{x,t}\|^2-\frac{\eta_x K}{2}\mathbb{E}\|\nabla_x \hat{f}(w_t,z_t)\|^2-\frac{\eta_x K}{2}\mathbb{E}\|\Bar{e}_{x,t}\|^2+\\
    &\mathbb{E} \langle \nabla_y f(w_t), y_{t} - y_{t+1}\rangle - \frac{p+l}{2}\mathbb{E}\|x_{t+1}-x_t\|^2-\frac{p+l}{2}\mathbb{E}\|y_{t+1}-y_t\|^2\\
    \geq&\frac{\eta_xK}{2}\mathbb{E}\|\nabla_x \hat{f}(w_t,z_t)\|^2-\frac{\eta_x K}{2}\mathbb{E}\|\Bar{e}_{x,t}\|^2+\mathbb{E} \langle \nabla_y f(w_t), y_{t} - y_{t+1}\rangle - \frac{p+l}{2}\mathbb{E}\|x_{t+1}-x_t\|^2-\frac{p+l}{2}\mathbb{E}\|y_{t+1}-y_t\|^2\\
    \overset{(a)}{\geq}&(\frac{\eta_xK}{2}-(p+l)\eta_x^2K^2)\mathbb{E}\|\nabla_x \hat{f}(w_t,z_t)\|^2-(p+l)\eta_x^2K^2d_{x,t}-\frac{\eta_x K}{2}\mathbb{E}\|\Bar{e}_{x,t}\|^2+\\
    &\mathbb{E} \langle \nabla_y f(w_t), y_{t} - y_{t+1}\rangle -\frac{p+l}{2}\mathbb{E}\|y_{t+1}-y_t\|^2\\
    \overset{(b)}{\geq}&\frac{\eta_xK}{4}\mathbb{E}\|\nabla_x \hat{f}(w_t,z_t)\|^2-\frac{\eta_x K}{2}\mathbb{E}\|\Bar{e}_{x,t}\|^2-(p+l)\eta_x^2K^2d_{x,t}+\mathbb{E} \langle \nabla_y f(w_t), y_{t} - y_{t+1}\rangle -\frac{p+l}{2}\mathbb{E}\|y_{t+1}-y_t\|^2
\end{align*}
where the $(a)$ is due to the Lemma \ref{lemma: bound |x_t-x_{t+1}|}, and $(b)$ is due to the condition $\eta_x\leq \frac{1}{4(p+l)K}$.
\end{proof}

\begin{lemma}
The $z$-update in \alg yields
\label{lemma: f(w_t, .)}
    \begin{align*}
        \hat{f}(w_{t+1}, z_t) - \hat{f}(w_{t+1}, z_{t+1})\geq& \frac{p}{2\beta}\|z_t-z_{t+1}\|^2.
    \end{align*}
\end{lemma}
\begin{proof}
By definition of $\hat{f}$ and the update rule of $z$, as $0<\beta<1$, we have
\begin{align*} \nonumber
        \hat{f}(w_{t+1}, z_t) - \hat{f}(w_{t+1}, z_{t+1}) = & \frac{p}{2}[\|x_{t+1}-z_t\|^2 - \|x_{t+1}-z_{t+1}\|^2] \\ \nonumber
        =& \frac{p}{2}\left[\frac{1}{\beta^2}\|(z_{t+1} - z_t)\|^2 - \|(1-\beta)(x_{t+1}-z_t)\|^2 \right] \\ \nonumber
		=& \frac{p}{2}\left[\frac{1}{\beta^2}\|z_{t+1} - z_t\|^2 - \frac{(1-\beta)^2}{\beta^2}\|z_{t+1}-z_t\|^2 \right] \\
        \geq& \frac{p}{2\beta}\|z_t-z_{t+1}\|^2.
\end{align*}
\end{proof}

\begin{lemma}
\label{lemma: Psi}
With $L_{\Psi} = l+l\gamma_2$, $\gamma_2=\frac{l+p}{p-l}$, we have
	\begin{align*}
		\Psi(y_{t+1},z_t) - \Psi(y_t,z_t) \geq &\langle \nabla_y \Psi(y_t,z_t), y_{t+1}-y_t\rangle - \frac{L_{\Psi}}{2}\|y_{t+1}-y_t\|^2, \\
		\Psi(y_{t+1},z_{t+1}) - \Psi(y_{t+1},z_t) 
		\geq & \frac{p}{2}(z_{t+1}-z_t)^\top [z_{t+1}+z_t - 2x^*(y_{t+1},z_{t+1})].
	\end{align*}
\end{lemma}
\begin{proof}
    Since the dual function $\Psi(\cdot,z)$ is $L_{\Psi}$-smooth by Lemma B.3 in \cite{zhang2020single}, we have
    \begin{align*}
		\Psi(y_{t+1},z_t) - \Psi(y_t,z_t) \geq &\langle \nabla_y \Psi(y_t,z_t), y_{t+1}-y_t\rangle - \frac{L_{\Psi}}{2}\|y_{t+1}-y_t\|^2.
    \end{align*}
    By the definition of $x^*(y_{t+1},z_{t})$, we have
    \begin{align} \nonumber
		\Psi(y_{t+1},z_{t+1}) - \Psi(y_{t+1},z_t) =& \hat{f}(x^*(y_{t+1},z_{t+1}), y_{t+1}, z_{t+1}) - \hat{f}(x^*(y_{t+1},z_{t}), y_{t+1}, z_{t})  \\ \nonumber
		\geq & \hat{f}(x^*(y_{t+1},z_{t+1}), y_{t+1}, z_{t+1}) - \hat{f}(x^*(y_{t+1},z_{t+1}), y_{t+1}, z_{t})  \\ \nonumber
		= & \frac{p}{2}\left[\|z_{t+1}-x^*(y_{t+1},z_{t+1})\|^2 - \|z_{t}-x^*(y_{t+1},z_{t+1})\|^2\right] \\ \nonumber
		=& \frac{p}{2}(z_{t+1}-z_t)^\top [z_{t+1}+z_t - 2x^*(y_{t+1},z_{t+1})].
	\end{align}
\end{proof}
\begin{lemma}
\label{lemma: P}
	\begin{align*}
		P(z_{t+1}) - P(z_t)
		\leq  &\frac{p}{2}(z_{t+1}-z_t)^\top[z_{t+1}+z_t -2x^*(\hat{y}^*(z_{t+1}),z_t)].
	\end{align*}
\end{lemma}
\begin{proof}
By the definition of $\hat{y}^*(z_{t})$ and $ x^*(\hat{y}^*(z_{t+1}), z_{t+1})$, we have
	\begin{align*} \nonumber
		P(z_{t+1}) - P(z_t) = & \Psi(\hat{y}^*(z_{t+1}),z_{t+1}) - \Psi(\hat{y}^*(z_{t}),z_{t}) \\ \nonumber
		\leq & \Psi(\hat{y}^*(z_{t+1}),z_{t+1}) - \Psi(\hat{y}^*(z_{t+1}),z_{t}) \\ \nonumber
		= & \hat{f}(x^*(\hat{y}^*(z_{t+1}), z_{t+1}),\hat{y}^*(z_{t+1}),z_{t+1}) - \hat{f}(x^*(\hat{y}^*(z_{t+1}),z_t),\hat{y}^*(z_{t+1}),z_{t}) \\ \nonumber
		\leq & \hat{f}(x^*(\hat{y}^*(z_{t+1}), z_{t}),\hat{y}^*(z_{t+1}),z_{t+1}) - \hat{f}(x^*(\hat{y}^*(z_{t+1}),z_t),\hat{y}^*(z_{t+1}),z_{t}) \\ 
		= &\frac{p}{2}(z_{t+1}-z_t)^\top[z_{t+1}+z_t -2x^*(\hat{y}^*(z_{t+1}),z_t)].
	\end{align*}
\end{proof}
\begin{lemma}
\label{lemma: Psi P}
The following inequality holds:
 \begin{align*}
     &2\mathbb{E}(\Psi(y_{t+1},z_{t+1}) - \Psi(y_{t+1},z_t))-2\mathbb{E}(P(z_{t+1}) - P(z_t))\\
     \geq&-\left(2p\gamma_1 + \frac{p}{6\beta}-48p\beta\gamma_1^2\right)\|z_{t+1}-z_t\|^2 - 24p\beta\mathbb{E}\|x^*(z_{t})-x^*(y_{t}, z_{t})\|^2 - 48p\beta\gamma_2^2\|\Bar{y}_{t+1}-y_t\|^2-\\ \nonumber
    &48p\beta\gamma_2^2\eta_y^2K^2d_{y,t}.
 \end{align*}
\end{lemma}
\begin{proof}
Combining Lemmas \ref{lemma: Psi} and \ref{lemma: P}, we have
 \begin{align}\nonumber
     &2\mathbb{E}(\Psi(y_{t+1},z_{t+1}) - \Psi(y_{t+1},z_t))-2\mathbb{E}(P(z_{t+1}) - P(z_t))\\ \nonumber
     \geq&2p\mathbb{E}(z_{t+1}-z_t)^\top [x^*(\hat{y}^*(z_{t+1}),z_t)- x^*(y_{t+1},z_{t+1})]\\ \nonumber
     =&2p\mathbb{E}(z_{t+1}-z_t)^\top [x^*(\hat{y}^*(z_{t+1}),z_t)- x^*(\hat{y}^*(z_{t+1}),z_{t+1})] + 2p\mathbb{E}(z_{t+1}-z_t)^\top [x^*(\hat{y}^*(z_{t+1}),z_{t+1})- x^*(y_{t+1},z_{t+1})] \\ \nonumber
	\overset{(a)}{\geq} & -2p\gamma_1 \|z_{t+1}-z_t\|^2 + 2p\mathbb{E}(z_{t+1}-z_t)^\top [x^*(\hat{y}^*(z_{t+1}),z_{t+1})- x^*(y_{t+1},z_{t+1})] \\ \nonumber
	\geq & -\left(2p\gamma_1 + \frac{p}{6\beta} \right)\|z_{t+1}-z_t\|^2 - 6p\beta\mathbb{E}\|x^*(z_{t+1})-x^*(y_{t+1}, z_{t+1})\|^2\\ \nonumber
    \geq& -\left(2p\gamma_1 + \frac{p}{6\beta} \right)\|z_{t+1}-z_t\|^2 - 24p\beta\mathbb{E}\|x^*(z_{t})-x^*(y_{t}, z_{t})\|^2 - 24p\beta\mathbb{E}\|x^*(z_{t})-x^*(z_{t+1})\|^2-\\ \nonumber
    &24p\beta\mathbb{E}\|x^*(y_{t+1},z_{t+1})-x^*(y_{t+1}, z_{t})\|^2-24p\beta\mathbb{E}\|x^*(y_{t+1},z_{t})-x^*(y_{t}, z_{t})\|^2\\ \nonumber
    \overset{(b)}{\geq}& -\left(2p\gamma_1 + \frac{p}{6\beta}+48p\beta\gamma_1^2\right)\|z_{t+1}-z_t\|^2 - 24p\beta\mathbb{E}\|x^*(z_{t})-x^*(y_{t}, z_{t})\|^2 - 24p\beta\gamma_2^2\|y_{t+1}-y_t\|^2\\ \nonumber
    \overset{(c)}{\geq}& -\left(2p\gamma_1 + \frac{p}{6\beta}+48p\beta\gamma_1^2\right)\|z_{t+1}-z_t\|^2 - 24p\beta\mathbb{E}\|x^*(z_{t})-x^*(y_{t}, z_{t})\|^2 - 48p\beta\gamma_2^2\|\Bar{y}_{t+1}-y_t\|^2-\\ \nonumber
    &48p\beta\gamma_2^2\eta_y^2K^2d_{y,t},
 \end{align}
 where $(a) $ and $(b)$ are due to Lemma \ref{lemma: helper lemma of optimal x}, and $(c)$ is due to Lemma \ref{lemma: bound |x_t-x_{t+1}|}.
\end{proof}

\begin{lemma}
\label{lemma: f Psi}
Suppose we have $\eta_y\leq \frac{1}{8(2L_\Psi+l+p)K}$, $\eta_y=\eta_x/256$ and $p=2l$. In the unconstrained case when $Y=\mathbb{R}^{d_2}$, we have
 \begin{align*}\nonumber
     &\mathbb{E}\hat{f}(w_t, z_t) - \mathbb{E}\hat{f}(w_{t+1}, z_t)+2(\mathbb{E}\Psi(y_{t+1},z_t) - \mathbb{E}\Psi(y_t,z_t))\\ \nonumber
     \geq&\frac{\eta_xK}{8}\mathbb{E}\|\nabla_x \hat{f}(w_t,z_t)\|^2+\frac{\eta_yK}{8}\mathbb{E}\|\nabla_y f(w_t)\|^2-\frac{\eta_x K}{2}\mathbb{E}\|\Bar{e}_{x,t}\|^2-(p+l)\eta_x^2K^2d_{x,t}-\\
     &\frac{3\eta_yK}{4}\mathbb{E}\|\Bar{e}_{y,t}\|^2-(2L_\Psi+l+p)\eta_y^2K^2d_{y,t}.
\end{align*}
In the constrained case when $Y\subset \mathbb{R}^{d_2}$ is convex and compact, we have
 \begin{align*}
     &\mathbb{E}\hat{f}(w_t, z_t) - \mathbb{E}\hat{f}(w_{t+1}, z_t)+2(\mathbb{E}\Psi(y_{t+1},z_t) - \mathbb{E}\Psi(y_t,z_t))\\ \nonumber
     \geq&\frac{\eta_xK}{8}\mathbb{E}\|\nabla_x \hat{f}(w_t,z_t)\|^2+\frac{1}{8\eta_yK}\mathbb{E}\|\Bar{y}_{t+1}-y_t\|^2-\frac{\eta_x K}{2}\mathbb{E}\|\Bar{e}_{x,t}\|^2-(p+l)\eta_x^2K^2d_{x,t}-\\
     &4\eta_yK\mathbb{E}\|\Bar{e}_{y,t}\|^2-2\eta_yKd_{y,t},
 \end{align*}
\end{lemma}
\begin{proof}
Combining Lemmas \ref{lemma: f(w_t, .)} and \ref{lemma: Psi}, we have
 \begin{align}\nonumber
     &\mathbb{E}\hat{f}(w_t, z_t) - \mathbb{E}\hat{f}(w_{t+1}, z_t)+2(\mathbb{E}\Psi(y_{t+1},z_t) - \mathbb{E}\Psi(y_t,z_t))\\ \nonumber
     \geq&\frac{\eta_xK}{4}\|\nabla_x \hat{f}(w_t,z_t)\|^2-\frac{\eta_x K}{2}\mathbb{E}\|\Bar{e}_{x,t}\|^2-(p+l)\eta_x^2K^2d_{x,t}+\mathbb{E} \langle \nabla_y f(w_t), y_{t} - y_{t+1}\rangle -\frac{p+l}{2}\mathbb{E}\|y_{t+1}-y_t\|^2+\\ \label{f Psi A1}
     &2\mathbb{E}\langle\nabla_y\Psi(y_t,z_t), y_{t+1}-y_t\rangle-L_\Psi\mathbb{E}\|y_{t+1}-y_t\|^2.
\end{align}
Denote $A_1=\mathbb{E} \langle \nabla_y f(w_t), y_{t} - y_{t+1}\rangle -\frac{p+l}{2}\mathbb{E}\|y_{t+1}-y_t\|^2+2\mathbb{E}\langle\nabla_y\Psi(y_t,z_t), y_{t+1}-y_t\rangle-L_\Psi\mathbb{E}\|y_{t+1}-y_t\|^2$. 

When $Y=\mathbb{R}^{d_2}$, we have
    \begin{align}\nonumber
     A_1=& \mathbb{E}\langle \nabla_y f(w_t), y_{t+1}-y_{t}\rangle +2\mathbb{E}\langle \nabla_y \Psi(y_t,z_t)-\nabla_y f(w_t), y_{t+1}-y_t\rangle - \frac{2L_\Psi+l+p}{2}\mathbb{E}\|y_t-y_{t+1}\|^2 \\ \nonumber
     \overset{(a)}{\geq}& \eta_yK\mathbb{E}\langle \nabla_y f(w_t), u_{y,t}-e_{y,t}\rangle-2\eta_yK\mathbb{E}\langle \nabla_y \Psi(y_t,z_t)-\nabla_y f(w_t), u_{y,t}-e_{y,t}\rangle-\\ \nonumber
     &\eta_y^2K^2(2L_\Psi+l+p)\|\nabla_y f(w_t)\|^2-\eta_y^2K^2(2L_\Psi+l+p)d_{y,t}
     \\\nonumber
     \geq& \eta_yK\mathbb{E}\langle \nabla_y f(w_t), \nabla_y f(w_t)-\Bar{e}_{y,t}\rangle-2\eta_yK\mathbb{E}\|\nabla_y \Psi(y_t,z_t)-\nabla_y f(w_t)\|\|\nabla_y f(w_t)-\Bar{e}_{y,t}\|-\\ \nonumber
     &\eta_y^2K^2(2L_\Psi+l+p)\mathbb{E}\|\nabla_y f(w_t)\|^2-\eta_y^2K^2(2L_\Psi+l+p)d_{y,t}
     \\\nonumber
     \geq&\eta_yK\mathbb{E}\| \nabla_y f(w_t)\|^2-\frac{\eta_yK}{2}\mathbb{E}\|\nabla_y f(w_t)\|^2-\frac{\eta_yK}{2}\mathbb{E}\|\Bar{e}_{y,t}\|^2-\frac{\eta_yK}{8}\mathbb{E}\|\nabla_y f(w_t)-\Bar{e}_{y,t}\|^2-\\ \nonumber
     &8\eta_yK\mathbb{E}\|\nabla_y \Psi(y_t,z_t)-\nabla_y f(w_t)\|^2-\eta_y^2K^2(2L_\Psi+l+p)\mathbb{E}\|\nabla_y f(w_t)\|^2-\eta_y^2K^2(2L_\Psi+l+p)d_{y,t}
     \\\nonumber
     \geq&\frac{\eta_yK}{2}\mathbb{E}\| \nabla_y f(w_t)\|^2-\frac{\eta_yK}{2}\mathbb{E}\|\Bar{e}_{y,t}\|^2-\frac{\eta_yK}{4}\mathbb{E}\|\nabla_y f(w_t)\|^2-\frac{\eta_yK}{4}\mathbb{E}\|\Bar{e}_{y,t}\|^2-\\ \nonumber
     &8\eta_yKl^2\mathbb{E}\|x^*(y_t,z_t)-x_t\|^2-\eta_y^2K^2(2L_\Psi+l+p)\mathbb{E}\|\nabla_y f(w_t)\|^2-\eta_y^2K^2(2L_\Psi+l+p)d_{y,t}
     \\\nonumber
     \overset{(b)}{\geq}&\left(\frac{\eta_yK}{4}-\eta_y^2K^2(2L_\Psi+l+p)\right)\mathbb{E}\|\nabla_y f(w_t)\|^2-\frac{3\eta_yK}{4}\mathbb{E}\|\Bar{e}_{y,t}\|^2-(2L_\Psi+l+p)\eta_y^2K^2d_{y,t}-\\ \nonumber
     &\frac{8\eta_yKl^2}{(p-l)^2}\mathbb{E}\|\nabla_x \hat{f}(w_t,z_t)\|^2\\ \label{A1 Y=R}
     \overset{(c)}{\geq}&\frac{\eta_yK}{8}\mathbb{E}\|\nabla_y f(w_t)\|^2-8\eta_yK\mathbb{E}\|\nabla_x \hat{f}(w_t,z_t)\|^2-\frac{3\eta_yK}{4}\mathbb{E}\|\Bar{e}_{y,t}\|^2-(2L_\Psi+l+p)\eta_y^2K^2d_{y,t}
 \end{align}
 where $(a)$ is due to Lemma \ref{lemma: bound |x_t-x_{t+1}|} and when $Y=\mathbb{R}^{d_2}, y_{t+1}-y_t=\eta_yK(u_{y,t}-e_{y,t})$,  $(b)$ is because of the $(p-l)$-strongly convexity of $\hat{f}(\cdot, y,z)$, and $(c)$ is due to the condition $\eta_y\leq \frac{1}{8(2L_\Psi+l+p)K}$ and $p=2l$. Combining \eqref{f Psi A1} and \eqref{A1 Y=R}, we have
 \begin{align*}
     &\mathbb{E}\hat{f}(w_t, z_t) - \mathbb{E}\hat{f}(w_{t+1}, z_t)+2(\mathbb{E}\Psi(y_{t+1},z_t) - \mathbb{E}\Psi(y_t,z_t))\\ \nonumber
     \geq&\left(\frac{\eta_xK}{4}-8\eta_yK\right)\mathbb{E}\|\nabla_x \hat{f}(w_t,z_t)\|^2-\frac{\eta_x K}{2}\mathbb{E}\|\Bar{e}_{x,t}\|^2-(p+l)\eta_x^2K^2d_{x,t}+\\
     &\frac{\eta_yK}{8}\mathbb{E}\|\nabla_y f(w_t)\|^2-\frac{3\eta_yK}{4}\mathbb{E}\|\Bar{e}_{y,t}\|^2-(2L_\Psi+l+p)\eta_y^2K^2d_{y,t}\\
     \geq&\frac{\eta_xK}{8}\mathbb{E}\|\nabla_x \hat{f}(w_t,z_t)\|^2+\frac{\eta_yK}{8}\mathbb{E}\|\nabla_y f(w_t)\|^2-\frac{\eta_x K}{2}\mathbb{E}\|\Bar{e}_{x,t}\|^2-(p+l)\eta_x^2K^2d_{x,t}-\\
     &\frac{3\eta_yK}{4}\mathbb{E}\|\Bar{e}_{y,t}\|^2-(2L_\Psi+l+p)\eta_y^2K^2d_{y,t},
 \end{align*}where the last inequality is because of the condition $\eta_y=\eta_x/256$. 
 
 When $Y\subset \mathbb{R}^{d_2}$ is convex and compact, we have
    \begin{align}\nonumber
     A_1
     =& \mathbb{E}\langle \nabla_y f(w_t), y_{t+1}-y_{t}\rangle +2\mathbb{E}\langle \nabla_y \Psi(y_t,z_t)-\nabla_y f(w_t), y_{t+1}-y_t\rangle - \frac{2L_\Psi+l+p}{2}\mathbb{E}\|y_t-y_{t+1}\|^2 \\ \nonumber
     \geq& \mathbb{E}\langle u_{y,t}-e_{y,t}, y_{t+1}-y_{t}\rangle+\mathbb{E}\langle \nabla_y f(w_t)-u_{y,t}+e_{y,t}, y_{t+1}-\Bar{y}_{t+1}\rangle+\mathbb{E}\langle\nabla_y f(w_t)-u_{y,t}+e_{y,t}, \Bar{y}_{t+1}-y_t\rangle-\\ \nonumber
     &2\mathbb{E}\| \nabla_y \Psi(y_t,z_t)-\nabla_y f(w_t)\| \|y_{t+1}-y_t\| - \frac{2L_\Psi+l+p}{2}\mathbb{E}\|y_t-y_{t+1}\|^2 \\\nonumber
     =&\mathbb{E}\langle u_{y,t}-e_{y,t}, P_Y(y_t+\eta_yK(u_{y,t}-e_{y,t}))-y_t\rangle+\\ \nonumber
     &\mathbb{E}\langle \nabla_y f(w_t)-u_{y,t}+e_{y,t}, P_Y(y_t+\eta_yK(u_{y,t}-e_{y,t}))-P_Y(y_t+\eta_yK\nabla_y f(w_t))\rangle+\\ \nonumber
     &\mathbb{E}\langle \Bar{e}_{y,t}, \Bar{y}_{t+1}-y_t\rangle-2\mathbb{E}\| \nabla_y \Psi(y_t,z_t)-\nabla_y f(w_t)\| \|y_{t+1}-y_t\| - \frac{2L_\Psi+l+p}{2}\mathbb{E}\|y_t-y_{t+1}\|^2 \\\nonumber
     \geq&\frac{1}{\eta_y K}\mathbb{E}\|y_{t+1}-y_t\|^2-\eta_yK \mathbb{E}\|\nabla_y f(w_t)-u_{y,t}+e_{y,t}\|^2-4\eta_yK\mathbb{E}\|\Bar{e}_{y,t}\|^2-\frac{1}{8\eta_yK}\mathbb{E}\|\Bar{y}_{t+1}-y_t\|^2-\\ \nonumber
     &\frac{1}{8\eta_yK}\mathbb{E}\|y_{t+1}-y_t\|^2-8\eta_yK\mathbb{E}\|\nabla_y \Psi(y_t,z_t)-\nabla_y f(w_t)\|^2- \frac{2L_\Psi+l+p}{2}\mathbb{E}\|y_t-y_{t+1}\|^2 \\\nonumber
     \geq&\left(\frac{1}{\eta_yK}-\frac{1}{8\eta_yK}- \frac{2L_\Psi+l+p}{2}\right)\mathbb{E}\|y_{t+1}-y_t\|^2-\eta_yK d_{y,t}-4\eta_yK\mathbb{E}\|\Bar{e}_{y,t}\|^2-\\ \nonumber
     &\frac{1}{8\eta_yK}\mathbb{E}\|\Bar{y}_{t+1}-y_t\|^2-8\eta_yKl^2\mathbb{E}\|x_t-x^*(y_t,z_t)\|^2\\ \nonumber
     \overset{(a)}{\geq}&\frac{1}{2\eta_yK}\mathbb{E}\|y_{t+1}-y_t\|^2-\eta_y Kd_{y,t}-4\eta_yK\mathbb{E}\|\Bar{e}_{y,t}\|^2-\frac{1}{8\eta_yK}\mathbb{E}\|\Bar{y}_{t+1}-y_t\|^2-\frac{8\eta_yKl^2}{(p-l)^2}\mathbb{E}\|\nabla_x f(w_t,z_t)\|^2\\ \nonumber
     \overset{(b)}{\geq}&\frac{1}{4\eta_yK}\mathbb{E}\|\Bar{y}_{t+1}-y_t\|^2-\frac{1}{\eta_yK}\mathbb{E}\|\Bar{y}_{t+1}-y_{t+1}\|^2 -\eta_y Kd_{y,t}-4\eta_yK\mathbb{E}\|\Bar{e}_{y,t}\|^2-8\eta_yK\mathbb{E}\|\nabla_x f(w_t,z_t)\|^2-\\ \nonumber
     &\frac{1}{8\eta_yK}\mathbb{E}\|\Bar{y}_{t+1}-y_t\|^2\\ \label{A1 Y<R}
     \geq&\frac{1}{8\eta_yK}\mathbb{E}\|\Bar{y}_{t+1}-y_t\|^2-2\eta_y Kd_{y,t}-4\eta_yK\mathbb{E}\|\Bar{e}_{y,t}\|^2-8\eta_yK\mathbb{E}\|\nabla_x f(w_t,z_t)\|^2,
 \end{align}
 where $(a)$ is due to the condition $\eta_y\leq \frac{1}{8(2L_\Psi+l+p)K}$ and the $(p-l)$-strongly convexity of $\hat{f}(\cdot, y,z)$, $(b)$ is due to the condition $p=2l$ and $\|a\|^2\geq \frac{1}{2}\|a-b\|^2-2\|b\|^2$. Combining \eqref{f Psi A1} and \eqref{A1 Y<R}, we have
 \begin{align*}
     &\mathbb{E}\hat{f}(w_t, z_t) - \mathbb{E}\hat{f}(w_{t+1}, z_t)+2(\mathbb{E}\Psi(y_{t+1},z_t) - \mathbb{E}\Psi(y_t,z_t))\\ \nonumber
     \geq&\left(\frac{\eta_xK}{4}-8\eta_yK\right)\mathbb{E}\|\nabla_x \hat{f}(w_t,z_t)\|^2-\frac{\eta_x K}{2}\mathbb{E}\|\Bar{e}_{x,t}\|^2-(p+l)\eta_x^2K^2d_{x,t}+\\
     &\frac{1}{8\eta_yK}\mathbb{E}\|\Bar{y}_{t+1}-y_t\|^2-4\eta_yK\mathbb{E}\|\Bar{e}_{y,t}\|^2-2\eta_yKd_{y,t}\\
     \geq&\frac{\eta_xK}{8}\mathbb{E}\|\nabla_x \hat{f}(w_t,z_t)\|^2+\frac{1}{8\eta_yK}\mathbb{E}\|\Bar{y}_{t+1}-y_t\|^2-\frac{\eta_x K}{2}\mathbb{E}\|\Bar{e}_{x,t}\|^2-(p+l)\eta_x^2K^2d_{x,t}-\\
     &4\eta_yK\mathbb{E}\|\Bar{e}_{y,t}\|^2-2\eta_yKd_{y,t},
 \end{align*}where the last inequality is because of the condition $\eta_y=\eta_x/256$.

\end{proof}

\begin{lemma}
\label{lemma: v}
    Define potential function $ V _t =  V (x_t, y_t, z_t) = \hat{f}(x_t, y_t,z_t) - 2\Psi(y_t, z_t) + 2P(z_t)$, with $p=2l, \eta_x\leq 1/(1000Kl),\eta_y=\eta_x/256,  \beta\leq\eta_yKl/80000$, $\eta_{x,l}\leq \min\{\frac{1}{2l\sqrt{2(2K-1)(K-1)}}, \sqrt{\frac{\beta}{6144\eta_x p l^2K^3}}\}, \eta_{y,l}\leq\min\{\frac{1}{2l\sqrt{2(2K-1)(K-1)}}, \\\sqrt{\frac{\eta_y}{3072\eta_x l^2K^2}}\}$, 
when $Y=\mathbb{R}^{d_2}$, we have
	\begin{align*} \nonumber
		&\mathbb{E} V _t -\mathbb{E} V _{t+1}\\
        \geq&\frac{\eta_xK}{32}\mathbb{E}\|\nabla_x \hat{f}(w_t,z_t)\|^2+\frac{\eta_yK}{32}\mathbb{E}\|\nabla_y f(w_t)\|^2+\frac{p\beta}{16}\mathbb{E}\|x_t-z_t\|^2-24p\beta\mathbb{E}\|x^*(z_{t})-x^*(y_{t}, z_{t})\|^2-\\
        &25l\eta_x^2K^2 (M-m)\frac{\sigma_G^2}{mM}-15l\eta_x^2K\frac{\sigma^2}{m}-4\eta_xKl^2[24K^2(\eta_{x,l}^2+\eta_{y,l}^2)\sigma_G^2+3K(\eta_{x,l}^2+2K\eta_{y,l}^2)\sigma^2]; 
	\end{align*}
when $Y\subset \mathbb{R}^{d_2}$ is convex and compact and under Assumption \ref{assum:bounded G_y}, we have
	\begin{align*} \nonumber
		&\mathbb{E} V _t -\mathbb{E} V _{t+1}\\
        \geq&\frac{\eta_xK}{32}\mathbb{E}\|\nabla_x \hat{f}(w_t,z_t)\|^2+\frac{1}{16\eta_yK}\mathbb{E}\|\Bar{y}_{t+1}-y_t\|^2+\frac{p\beta}{16}\mathbb{E}\|x_t-z_t\|^2-24p\beta\mathbb{E}\|x^*(z_{t})-x^*(y_{t}, z_{t})\|^2-\\
        &25l\eta_x^2K^2 (M-m)\frac{\sigma_G^2}{mM}-15l\eta_x^2K\frac{\sigma^2}{m}-8\eta_yK(M-m)\frac{\sigma_G^2}{mM}-\frac{4\eta_y\sigma^2}{m}-\\
        &4\eta_xKl^2[24K^2(\eta_{x,l}^2+\eta_{y,l}^2)(\sigma_G^2+G_y^2)+3K(\eta_{x,l}^2+2K\eta_{y,l}^2)\sigma^2].
	\end{align*}

\end{lemma}
\begin{proof}
    Combining Lemma \ref{lemma: f(w_t, .)}, Lemma \ref{lemma: f Psi} and Lemma \ref{lemma: Psi P}, when $Y=\mathbb{R}^{d_2}$, we have
	\begin{align*} \nonumber
		&\mathbb{E} V _t -\mathbb{E} V _{t+1}\\
        \geq &\frac{\eta_xK}{8}\mathbb{E}\|\nabla_x \hat{f}(w_t,z_t)\|^2 +\left(\frac{\eta_yK}{8}-48p\beta\gamma_2^2\eta_y^2K^2\right)\mathbb{E}\|\nabla_y f(w_t)\|^2-24p\beta\mathbb{E}\|x^*(z_{t})-x^*(y_{t}, z_{t})\|^2\\ \nonumber
        &\left(\frac{p}{2\beta}-2p\gamma_1 - \frac{p}{6\beta}-48p\beta\gamma_1^2\right)\mathbb{E}\|z_{t+1}-z_t\|^2-\frac{\eta_xK}{2}\mathbb{E}\|\Bar{e}_{x,t}\|^2-\frac{3\eta_yK}{4}\mathbb{E}\|\Bar{e}_{y,t}\|^2-\\
        &(p+l)\eta_x^2K^2d_{x,t}-\left(2L_\Psi+p+l+48p\beta\gamma_2^2\right)\eta_y^2K^2d_{y,t}\\
        \geq&\frac{\eta_xK}{8}\mathbb{E}\|\nabla_x \hat{f}(w_t,z_t)\|^2+\frac{\eta_yK}{16}\mathbb{E}\|\nabla_y f(w_t)\|^2+\frac{p}{4\beta}\mathbb{E}\|z_{t+1}-z_t\|^2-\\
        &24p\beta\mathbb{E}\|x^*(z_{t})-x^*(y_{t}, z_{t})\|^2-\frac{\eta_xK}{2}\mathbb{E}\|\Bar{e}_{x,t}\|^2-\frac{3\eta_yK}{4}\mathbb{E}\|\Bar{e}_{y,t}\|^2-\\
        &(p+l)\eta_x^2K^2d_{x,t}-\left(2L_\Psi+p+l+48p\beta\gamma_2^2\right)\eta_y^2K^2d_{y,t}\\
        \overset{(a)}{\geq}&\frac{\eta_xK}{8}\mathbb{E}\|\nabla_x \hat{f}(w_t,z_t)\|^2+\frac{\eta_yK}{16}\mathbb{E}\|\nabla_y f(w_t)\|^2+\frac{p\beta}{8}\mathbb{E}\|x_t-z_t\|^2-\frac{p\beta}{2}\mathbb{E}\|x_{t+1}-x_t\|^2-\\
        &24p\beta\mathbb{E}\|x^*(z_{t})-x^*(y_{t}, z_{t})\|^2-\frac{\eta_xK}{2}\mathbb{E}\|\Bar{e}_{x,t}\|^2-\frac{3\eta_yK}{4}\mathbb{E}\|\Bar{e}_{y,t}\|^2-\\
        &(p+l)\eta_x^2K^2d_{x,t}-\left(2L_\Psi+p+l+48p\beta\gamma_2^2\right)\eta_y^2K^2d_{y,t}\\
        \overset{(b)}{\geq}&\frac{\eta_xK}{8}\mathbb{E}\|\nabla_x \hat{f}(w_t,z_t)\|^2+\frac{\eta_yK}{16}\mathbb{E}\|\nabla_y f(w_t)\|^2+\frac{p\beta}{8}\mathbb{E}\|x_t-z_t\|^2-p\beta\eta_x^2K^2\mathbb{E}\|\nabla_x \hat{f}(w_t,z_t)\|^2-p\beta\eta_x^2K^2 d_{x,t}\\
        &24p\beta\mathbb{E}\|x^*(z_{t})-x^*(y_{t}, z_{t})\|^2-\frac{\eta_xK}{2}\mathbb{E}\|\Bar{e}_{x,t}\|^2-\frac{3\eta_yK}{4}\mathbb{E}\|\Bar{e}_{y,t}\|^2-\\
        &(p+l)\eta_x^2K^2d_{x,t}-\left(2L_\Psi+p+l+48p\beta\gamma_2^2\right)\eta_y^2K^2d_{y,t}\\
        \geq&\frac{\eta_xK}{16}\mathbb{E}\|\nabla_x \hat{f}(w_t,z_t)\|^2+\frac{\eta_yK}{16}\mathbb{E}\|\nabla_y f(w_t)\|^2+\frac{p\beta}{8}\mathbb{E}\|x_t-z_t\|^2-24p\beta\mathbb{E}\|x^*(z_{t})-x^*(y_{t}, z_{t})\|^2-\\
        &\frac{\eta_xK}{2}\mathbb{E}\|\Bar{e}_{x,t}\|^2-\frac{3\eta_yK}{4}\mathbb{E}\|\Bar{e}_{y,t}\|^2-(p+l+p\beta)\eta_x^2K^2d_{x,t}-\left(2L_\Psi+p+l+48p\beta\gamma_2^2\right)\eta_y^2K^2d_{y,t}
	\end{align*}
 where in $(a)$, we use 
 \begin{align}
     \frac{p}{\beta}\|z_{t+1}-z_t\|^2=p\beta\|x_{t+1}-z_t\|^2\geq \frac{p\beta}{2}\|x_t-z_t\|^2-2p\beta\|x_{t+1}-x_t\|^2,
 \end{align}and in $(b)$, we use Lemma \ref{lemma: bound |x_t-x_{t+1}|}.

 Denote $A_2=-\frac{\eta_xK}{2}\mathbb{E}\|\Bar{e}_{x,t}\|^2-\frac{3\eta_yK}{4}\mathbb{E}\|\Bar{e}_{y,t}\|^2-(p+l+p\beta)\eta_x^2K^2d_{x,t}-\left(2L_\Psi+p+l+48p\beta\gamma_2^2\right)\eta_y^2K^2d_{y,t}$, we have
 \begin{align*}
     A_2\overset{(a)}{\geq} & -(\eta_xK+10l\eta_x^2K^2)\mathbb{E}\|\Bar{e}_{x,t}\|^2-(\eta_yK+24l\eta_y^2K^2)\mathbb{E}\|\Bar{e}_{y,t}\|^2-\\
        &20l\eta_x^2K^2 (M-m)\frac{\sigma_G^2}{mM}-10l\eta_x^2K\frac{\sigma^2}{m}-48l\eta_y^2K^2 (M-m)\frac{\sigma_G^2}{mM}-24l\eta_y^2K\frac{\sigma^2}{m}\\
        \geq& -25l\eta_x^2K^2 (M-m)\frac{\sigma_G^2}{mM}-15l\eta_x^2K\frac{\sigma^2}{m}-\\
        &4\eta_xKl^2[24K^2\eta_{x,l}^2\mathbb{E}\|\nabla_x f(w_{t})\|^2+24K^2\eta_{y,l}^2\mathbb{E}\|\nabla_y f(w_{t})\|^2+24K^2(\eta_{x,l}^2+\eta_{y,l}^2)\sigma_G^2+3K(\eta_{x,l}^2+2K\eta_{y,l}^2)\sigma^2]\\
        \overset{(b)}{\geq}&-25l\eta_x^2K^2 (M-m)\frac{\sigma_G^2}{mM}-15l\eta_x^2K\frac{\sigma^2}{m}-4\eta_xKl^2[24K^2(\eta_{x,l}^2+\eta_{y,l}^2)\sigma_G^2+3K(\eta_{x,l}^2+2K\eta_{y,l}^2)\sigma^2]-\\
        &4\eta_xKl^2[48K^2\eta_{x,l}^2\mathbb{E}\|\nabla_x \hat{f}(w_{t},z_t)\|^2+48K^2p^2\eta_{x,l}^2\mathbb{E}\|x_t-z_t\|^2+24K^2\eta_{y,l}^2\mathbb{E}\|\nabla_y f(w_{t})\|^2]\\
        \overset{(c)}{\geq} &-25l\eta_x^2K^2 (M-m)\frac{\sigma_G^2}{mM}-15l\eta_x^2K\frac{\sigma^2}{m}-4\eta_xKl^2[24K^2(\eta_{x,l}^2+\eta_{y,l}^2)\sigma_G^2+3K(\eta_{x,l}^2+2K\eta_{y,l}^2)\sigma^2]-\\
        &\frac{\eta_xK}{32}\mathbb{E}\|\nabla_x \hat{f}(w_t,z_t)\|^2-\frac{\eta_yK}{32}\mathbb{E}\|\nabla_y f(w_t)\|^2-\frac{p\beta}{16}\mathbb{E}\|x_t-z_t\|^2,
 \end{align*}
 where $(a)$ is because $(p+l+p\beta)\leq 5l, (2L_\Psi+p+l+48p\beta\gamma_2^2)\leq 12l $ and Lemma \ref{lemma: bound d}, in $(b)$, we use Lemma \ref{lemma: bound e}, in $(c)$, we use the condition
 \begin{align*}
     &\eta_{y,l}^2\leq \frac{\eta_y }{3072\eta_x l^2K^2}\\
     &\eta_{x,l}^2\leq \frac{\beta }{6144\eta_x p l^2K^3}\leq \frac{1}{6144 l^2K^2}.
 \end{align*}
 So we have
	\begin{align} \nonumber
		&\mathbb{E} V _t -\mathbb{E} V _{t+1}\\ \nonumber
        \geq&\frac{\eta_xK}{32}\mathbb{E}\|\nabla_x \hat{f}(w_t,z_t)\|^2+\frac{\eta_yK}{32}\mathbb{E}\|\nabla_y f(w_t)\|^2+\frac{p\beta}{16}\mathbb{E}\|x_t-z_t\|^2-24p\beta\mathbb{E}\|x^*(z_{t})-x^*(y_{t}, z_{t})\|^2-\\ \label{v:Y=R}
        &25l\eta_x^2K^2 (M-m)\frac{\sigma_G^2}{mM}-15l\eta_x^2K\frac{\sigma^2}{m}-4\eta_xKl^2[24K^2(\eta_{x,l}^2+\eta_{y,l}^2)\sigma_G^2+3K(\eta_{x,l}^2+2K\eta_{y,l}^2)\sigma^2].
	\end{align}
 When $Y\subset \mathbb{R}^{d_2}$ is convex and compact and under Assumption \ref{assum:bounded G_y}, similarly, we have
	\begin{align} \nonumber
		&\mathbb{E} V _t -\mathbb{E} V _{t+1}\\ \nonumber
  \geq &\frac{\eta_xK}{8}\mathbb{E}\|\nabla_x \hat{f}(w_t,z_t)\|^2 +\left(\frac{\eta_yK}{8}-48p\beta\gamma_2^2\eta_y^2K^2\right)\frac{1}{\eta_y^2K^2}\mathbb{E}\|\Bar{y}_{t+1}-y_t\|^2-24p\beta\mathbb{E}\|x^*(z_{t})-x^*(y_{t}, z_{t})\|^2\\ \nonumber
        &\left(\frac{p}{2\beta}-2p\gamma_1 - \frac{p}{6\beta}-48p\beta\gamma_1^2\right)\mathbb{E}\|z_{t+1}-z_t\|^2-\frac{\eta_xK}{2}\mathbb{E}\|\Bar{e}_{x,t}\|^2-4\eta_yK\mathbb{E}\|\Bar{e}_{y,t}\|^2-\\ \nonumber
        &(p+l)\eta_x^2K^2d_{x,t}-2\eta_yKd_{y,t}\\ \nonumber
        \geq&\frac{\eta_xK}{32}\mathbb{E}\|\nabla_x \hat{f}(w_t,z_t)\|^2+\frac{1}{16\eta_yK}\mathbb{E}\|\Bar{y}_{t+1}-y_t\|^2+\frac{p\beta}{16}\mathbb{E}\|x_t-z_t\|^2-24p\beta\mathbb{E}\|x^*(z_{t})-x^*(y_{t}, z_{t})\|^2-\\ \nonumber
        &25l\eta_x^2K^2 (M-m)\frac{\sigma_G^2}{mM}-15l\eta_x^2K\frac{\sigma^2}{m}-8\eta_yK(M-m)\frac{\sigma_G^2}{mM}-\frac{4\eta_y\sigma^2}{m}\\ \label{v:Y<R}
        &4\eta_xKl^2[24K^2(\eta_{x,l}^2+\eta_{y,l}^2)(\sigma_G^2+G_y^2)+3K(\eta_{x,l}^2+2K\eta_{y,l}^2)\sigma^2].
	\end{align}
 The majority of the terms in \eqref{v:Y<R} closely resemble those in \eqref{v:Y=R}. There are, however, two notable distinctions. First, there is an additional error term of $-8\eta_yK(M-m)\frac{\sigma_G^2}{mM}-\frac{4\eta_y\sigma^2}{m}$ attributed to the presence of $-2\eta_yKd_{y,t}$. Second, there is an additional error of $-96\eta_xK^3l^2G_y^2$, which arises from our utilization of Assumption \ref{assum:bounded G_y}.
 \end{proof}

\section{Nonconvex-PL}
\label{app: nc-pl}

\begin{lemma}
\label{lemma: sc Y=R}
Under Assumption \ref{assum:pl} and $p=2l$, we have
    \begin{align*}
        &\|x^*(y_{t},z_{t})-x^*(z_{t})\|^2\leq\frac{2}{l\mu}\|\nabla_y f(w_t)\|^2+\frac{2}{l\mu}\|\nabla_x \hat{f}(w_t,z_t)\|^2.
    \end{align*}
\end{lemma}
\begin{proof}
Because $\hat{f}(\cdot, y,z)$ is $(p-l)$-strongly convex, we have
    \begin{align*}
        &\|x^*(y_{t},z_{t})-x^*(z_{t})\|^2\\
        \overset{(a)}{\leq}&\frac{2}{p-l}[\hat{f}(x^*(y_t,z_{t}), \hat{y}^*(z_t),z_{t})-\hat{f}(x^*(z_t), \hat{y}^*(z_t),z_t)]\\
        \leq&\frac{2}{p-l}[\Phi(x^*(y_t,z_{t}),z_{t})-\Phi(x^*(z_t),z_t)]\\
        =&\frac{2}{p-l}[\Phi(x^*(y_t,z_{t}),z_{t})-\hat{f}(x^*(y_t,z_{t}),y_t,z_{t})+\hat{f}(x^*(y_t,z_{t}),y_t,z_{t})-\Phi(x^*(z_t),z_t)]\\
        \overset{(b)}{\leq}&\frac{2}{p-l}[\Phi(x^*(y_t,z_{t}),z_{t})-\hat{f}(x^*(y_t,z_{t}),y_t,z_{t})]\\
        \overset{(c)}{\leq}&\frac{1}{(p-l)\mu}\|\nabla_y f(x^*(y_t,z_{t}),y_t)\|^2\\
        \overset{(d)}{\leq}&\frac{2}{(p-l)\mu}\|\nabla_y f(w_t)\|^2+\frac{2}{(p-l)\mu}\|\nabla_y f(x^*(y_t,z_{t}),y_t)-\nabla_yf(x_{t},y_t)\|^2\\
        \leq&\frac{2}{(p-l)\mu}\|\nabla_y f(w_t)\|^2+\frac{2l^2}{(p-l)\mu}\|x^*(y_t,z_t)-x_{t}\|^2\\
        \overset{(e)}{\leq}&\frac{2}{(p-l)\mu}\|\nabla_y f(w_t)\|^2+\frac{2l^2}{(p-l)^3\mu}\|\nabla_x \hat{f}(w_t,z_t)\|^2\\
        =&\frac{2}{l\mu}\|\nabla_y f(w_t)\|^2+\frac{2}{l\mu}\|\nabla_x \hat{f}(w_t,z_t)\|^2.
    \end{align*}
  $(b)$ can be attributed to the fact that $\hat{f}(x^*(y_t,z_{t}),y_t,z_{t})\leq\Phi(x^*(z_t),z_t)$. $(c)$ arises from the $\mu$-PL property of $\hat{f}(x, \cdot, z)$. In $(a), (d), (e)$, we make use of Lemma \ref{lemma: property of sc}.
\end{proof}

\subsection*{Proof of Theorem \ref{thm:sc}}

We formally state Theorem \ref{thm:sc} below.

\noindent \textbf{Theorem \ref{thm:sc}} 
    Under Assumptions \ref{assum:smooth}, \ref{assum:bdd_var}, \ref{assum:bdd_hetero}, \ref{assum:phi} and \ref{assum:pl}, 
    with $p=2l,\eta_y=\eta_x/256,  \beta=\eta_yK\mu/80000$, 
    $\eta_{x,l}\leq \min\{\frac{1}{2l\sqrt{2(2K-1)(K-1)}}, \sqrt{\frac{\beta}{6144\eta_x p l^2K^3}}, O(\epsilon\sqrt{\kappa^{-1} (\sigma_G^2+\sigma^2)}(Kl)^{-1})\}, \eta_{y,l}\leq\min\{\frac{1}{2l\sqrt{2(2K-1)(K-1)}}, \sqrt{\frac{\eta_y}{3072\eta_x l^2K^2}}, O(\epsilon\sqrt{\kappa^{-1} (\sigma_G^2+\sigma^2)}(Kl)^{-1})\}$, 
    when $m=M$ or $\sigma_G=0$, if we apply Algorithm \ref{alg1} with $K=\Theta(\kappa m^{-1}\epsilon^{-2}), \eta_x=\min\{1/(1000Kl), \frac{\sqrt{m\Delta}}{\sigma\sqrt{KTl}}\}$, we can find an $(\epsilon, \epsilon/\sqrt{\kappa})$-stationary point of $f$ with a per-client sample complexity of $O(\kappa^2m^{-1}\epsilon^{-4})$ and a communication complexity of $O(\kappa\epsilon^{-2})$;  when $m<M$ and $\sigma_G>0$, if we apply Algorithm \ref{alg1} with $\eta_x=\min\{1/(1000Kl), \frac{\sqrt{m\Delta}}{\sigma\sqrt{Tl}K}\}, K=O(1)$, we can find an $(\epsilon, \epsilon/\sqrt{\kappa})$-stationary point of $f$ with a per-client sample complexity of $O(\kappa^2m^{-1}\epsilon^{-4})$ and a communication complexity of $O(\kappa^2m^{-1}\epsilon^{-4})$. Here, $\Delta=V_0-\Phi^*$, $\kappa=l/\mu$.

\begin{proof}
    Combining Lemma \ref{lemma: sc Y=R} and Lemma \ref{lemma: v}, we have
	\begin{align*} \nonumber
		\mathbb{E} V _t -\mathbb{E} V _{t+1} \geq &\left( \frac{\eta_xK}{32}-\frac{96\beta}{\mu}\right)\mathbb{E}\|\nabla_x \hat{f}(w_t,z_t)\|^2 + \left(\frac{\eta_yK}{32}-\frac{96\beta}{\mu}\right)\mathbb{E}\|\nabla_y f(w_t)\|^2+\frac{p\beta}{16}\mathbb{E}\|x_t-z_t\|^2-\\
  &25l\eta_x^2K^2(M-m)\frac{\sigma_G^2}{mM}-15l\eta_x^2K\frac{\sigma^2}{m}-4\eta_xKl^2[24K^2(\eta_{x,l}^2+\eta_{y,l}^2)\sigma_G^2+3K(\eta_{x,l}^2+2K\eta_{y,l}^2)\sigma^2].
	\end{align*}
Setting $\beta=\eta_yK\mu/80000$ yields
	\begin{align} \nonumber
		&\mathbb{E} V _t -\mathbb{E} V _{t+1}\\ \nonumber
        \geq & \frac{\eta_xK}{64}\mathbb{E}\|\nabla_x \hat{f}(w_t)\|^2 +\frac{\eta_yK}{64}\mathbb{E}\|\nabla_y f(w_t)\|^2+\frac{p\beta}{16}\mathbb{E}\|x_t-z_t\|^2-\\ \label{v sc Y=R}
        &25l\eta_x^2K^2(M-m)\frac{\sigma_G^2}{mM}-15l\eta_x^2K\frac{\sigma^2}{m}-4\eta_xKl^2[24K^2(\eta_{x,l}^2+\eta_{y,l}^2)\sigma_G^2+3K(\eta_{x,l}^2+2K\eta_{y,l}^2)\sigma^2].
	\end{align}
Further note that
\begin{align}
\label{nabla_x f}
    \|\nabla_x f(x_{t},y_{t})\|^2\leq&2\|\nabla_x \hat{f}(w_t,z_t)\|^2+2p^2\|x_t-z_t\|^2,
\end{align}
which leads to
\begin{align}\nonumber
    &\frac{1}{T}\sum_{t=0}^{T-1} \mathbb{E}\|\nabla_x f(x_{t},y_{t})\|^2+\kappa\mathbb{E}\|\nabla_y f(x_t,y_t)\|^2\\ \nonumber
    \leq&\frac{1}{T}\sum_{t=0}^{T-1}\max\bigg\{\frac{128\kappa}{\eta_xK}, \frac{64\kappa}{\eta_yK}, \frac{32p}{\beta}\bigg\}\bigg\{\mathbb{E} V _t -\mathbb{E} V _{t+1}+\\ \nonumber
    &25l\eta_x^2K^2(M-m)\frac{\sigma_G^2}{mM}+15l\eta_x^2K\frac{\sigma^2}{m}+4\eta_xKl^2[24K^2(\eta_{x,l}^2+\eta_{y,l}^2)\sigma_G^2+3K(\eta_{x,l}^2+2K\eta_{y,l}^2)\sigma^2]\bigg\}\\ \nonumber
    \overset{(a)}{\leq}&\frac{O(1)\kappa}{\eta_xKT}[V_0-\min_t V_t]+O(1)\kappa \eta_xlK(M-m)\frac{\sigma_G^2}{mM}+O(1)\kappa\eta_x l\frac{\sigma^2}{m}+\\ \nonumber
    &O(1)\kappa l^2[K^2(\eta_{x,l}^2+\eta_{y,l}^2)\sigma_G^2+K(\eta_{x,l}^2+2K\eta_{y,l}^2)\sigma^2]\\ \nonumber
    \leq&O(1)\frac{\kappa\Delta}{\eta_xKT}+O(1)\kappa \eta_xlK(M-m)\frac{\sigma_G^2}{mM}+O(1)\kappa\eta_x l\frac{\sigma^2}{m}+\\
    &O(1)\kappa l^2[K^2(\eta_{x,l}^2+\eta_{y,l}^2)\sigma_G^2+K(\eta_{x,l}^2+2K\eta_{y,l}^2)\sigma^2],
\end{align}
where (a) is because $p/\beta=O(1)\kappa/(\eta_xK)$ and \eqref{v_0-v_t}.

When $m=M$ or $\sigma_G=0$, with $\eta_x=\min\{1/(1000Kl), \frac{\sqrt{m\Delta}}{\sigma\sqrt{KTl}}\}$, $\eta_{x,l}\leq  O(\epsilon\sqrt{\kappa^{-1} (\sigma_G^2+\sigma^2)}(Kl)^{-1}), \eta_{y,l}\leq O(\epsilon\sqrt{\kappa^{-1} (\sigma_G^2+\sigma^2)}(Kl)^{-1})$, we have
\begin{align*}
    &\frac{1}{T}\sum_{t=0}^{T-1} \mathbb{E}\|\nabla_x f(x_{t},y_{t})\|^2+\kappa\mathbb{E}\|\nabla_y f(x_t,y_t)\|^2\leq O(1)\frac{\kappa}{\sqrt{mKT}}+O(1)\frac{\kappa}{T}+O(1)\epsilon^2,
\end{align*}
which implies that we can find an $(\epsilon, \epsilon/\sqrt{\kappa})$-stationary point of $f$ with a per-client sample complexity of $KT=O(\kappa^2m^{-1}\epsilon^{-4})$ and a communication complexity of $T=O(\kappa\epsilon^{-2})$.

When $m<M$ and $\sigma_G>0$, with $\eta_x=\min\{1/(1000Kl), \frac{\sqrt{m\Delta}}{\sigma\sqrt{Tl}K}\}$, $K=O(1)$, $\eta_{x,l}\leq  O(\epsilon\sqrt{\kappa^{-1} (\sigma_G^2+\sigma^2)}(Kl)^{-1}), \eta_{y,l}\leq O(\epsilon\sqrt{\kappa^{-1} (\sigma_G^2+\sigma^2)}(Kl)^{-1})$, we have
\begin{align*}
    &\frac{1}{T}\sum_{t=0}^{T-1} \mathbb{E}\|\nabla_x f(x_{t},y_{t})\|^2+\kappa\mathbb{E}\|\nabla_y f(x_t,y_t)\|^2\leq O(1)\frac{\kappa}{\sqrt{mT}}+O(1)\epsilon^2,
\end{align*}
which implies that we can find an $(\epsilon, \epsilon/\sqrt{\kappa})$-stationary point of $f$ with a per-client sample complexity of $KT=O(\kappa^2m^{-1}\epsilon^{-4})$ and a communication complexity of $T=O(\kappa^2 m^{-1} \epsilon^{-4})$.

\end{proof}

\section{Nonconvex-Strongly-Concave}
\label{app: nc-sc}
Since Nonconvex-PL is weaker than Nonconvex-Strongly-Concave (NC-SC), Theorem \ref{thm:sc} also holds for NC-SC. However, for NC-SC, Theorem \ref{thm:sc constrained y} proves that \alg can achieve similar convergence results when $Y$ is a convex, compact set of $\mathbb{R}^{d_2}$.
\begin{assume} [Strongly Concave in $y$]
\label{assum:sc}
$f$ is $\mu$-strongly concave ($\mu>0$) in $y$, if for any fixed $x$, $\max_y f(x,y)$, $\forall y, y'\in Y$, we have
\begin{align*}
    f(x,y)\leq f(x,y')+\langle\nabla_y f(x,y'), y-y'\rangle-\frac{\mu}{2}\|y-y'\|^2.
\end{align*}
\end{assume}

\begin{lemma}
\label{constraint eb}
Define $y_{+}(x)=P_Y(y+\eta_yK\nabla_y f(x,y))$. Under Assumptions \ref{assum:pl}, \ref{assum:x, bounded y}, \ref{assum:smooth}, and with $\eta_yK\leq 1/1000l$, we have
\begin{align*}
    \|y-y^*(x)\|\leq& \frac{2}{\mu\eta_yK}\|y-y_{+}(x)\|,\\
    \|y_{+}(x)-y^*(x)\|\leq& \frac{2}{\mu\eta_yK}\|y-y_{+}(x)\|.
\end{align*}
\end{lemma}
\begin{proof}
    We define $\hat{g}(y;v)=\|y\|^2-2\langle y,v \rangle+\textbf{1}_Y(y), v_1=y+\eta_yK\nabla_y f(x,y), v_2=y^*(x)+\eta_yK\nabla_y f(x,y^*(x))$. According to the definition of $ y_{+}(x), y^*(x)$, we have
    \begin{align*}
        y_{+}(x)=&\arg\min_y \hat{g}(y;v_1)\\
        y^*(x)=&\arg\min_y \hat{g}(y;v_2).
    \end{align*}
    Note that $\hat{g}(\cdot;v)$ is 2-strongly-convex, according to Lemma \ref{lemma: property of sc}, we have
    \begin{align}
    \label{g1}
        \hat{g}(y_{+}(x);v_2)-\hat{g}(y^*(x);v_2)\geq& \|y_{+}(x)-y^*(x)\|^2\\ \label{g2}
        \hat{g}(y^*(x);v_1)-\hat{g}(y_{+}(x);v_1)\geq& \|y_{+}(x)-y^*(x)\|^2.
    \end{align}
    By the definition of $\hat{g}$:
    \begin{align}
    \label{g3}
        \hat{g}(y_{+}(x);v_1)-\hat{g}(y_{+}(x);v_2)=&-2\langle y_{+}(x), v_1-v_2\rangle,\\ \label{g4}
        \hat{g}(y^*(x);v_1)-\hat{g}(y^*(x);v_2)=&-2\langle y^*(x), v_1-v_2\rangle.
    \end{align}
    Combining \eqref{g1},\eqref{g2},\eqref{g3},\eqref{g4}, we have
    \begin{align}
    \label{y dot}
        \|y_{+}(x)-y^*(x)\|^2\leq \langle y_{+}(x)-y^*(x), v_1-v_2\rangle.
    \end{align}
    Therefore,
    \begin{align}
    \label{y1}
        \|y_{+}(x)-y^*(x)\|\leq \|v_1-v_2\|.
    \end{align}
    By the definition of $v_1, v_2$, we have
    \begin{align}\nonumber
        \|v_1-v_2\|^2=&\|y-y^*(x)\|^2+2\eta_yK\langle y-y^*(x), \nabla_y f(x,y)-\nabla_y f(x,y^*(x))\rangle+\eta_y^2K^2\|\nabla_y f(x,y)-\nabla_y f(x,y^*(x))\|^2\\ \nonumber
        \overset{(a)}{\leq}&\|y-y^*(x)\|^2+(2\eta_yK-\eta_y^2K^2l)\langle y-y^*(x), \nabla_y f(x,y)-\nabla_y f(x,y^*(x))\rangle\\ \nonumber
        \overset{(b)}{\leq}&\|y-y^*(x)\|^2+\eta_yK\langle y-y^*(x), \nabla_y f(x,y)-\nabla_y f(x,y^*(x))\rangle\\ \nonumber
        \overset{(c)}{\leq}&(1-\eta_yK\mu)\|y-y^*(x)\|^2\\ \label{y2}
        \leq&\left(1-\frac{\eta_yK\mu}{2}\right)^2\|y-y^*(x)\|^2
    \end{align}
where $(a)$ is a consequence of several factors. Firstly, due to the concavity of $f$ in $y$, we have $\langle y-y^*(x), \nabla_y f(x,y)-\nabla_y f(x,y^*(x))\rangle\leq 0$. Additionally, Assumption \ref{assum:smooth} ensures that $\|\nabla_y f(x,y)-\nabla_y f(x,y^*(x))\|\leq l \|y-y^*(x)\|$. $(b)$ follows from the condition $\eta_yK\leq 1/l$, and $(c)$ stems from the $\mu$-strong concavity of $f(x, \cdot)$.

    Combining \eqref{y1}, \eqref{y2}, we have
    \begin{align*}
        \|y_{+}(x)-y^*(x)\|\leq \|v_1-v_2\| \leq (1-\eta_yK\mu/2)\|y-y^*(x)\|.
    \end{align*}
    So
    \begin{align*}
        \|y-y_{+}(x)\|\geq \|y-y^*(x)\|-\|y_{+}(x)-y^*(x)\|\geq \frac{\eta_yK\mu}{2}\|y-y^*(x)\|,
    \end{align*}
    \begin{align*}
        \|y-y^*(x)\|\leq \frac{2}{\eta_yK\mu}\|y-y_{+}(x)\|,
    \end{align*}
    which yields
    \begin{align*}
        \|y_{+}(x)-y^*(x)\|\leq  (1-\eta_yK\mu/2)\|y-y^*(x)\|\leq \frac{2}{\eta_yK\mu}\|y-y_{+}(x)\|.
    \end{align*}
    
\end{proof}

\begin{lemma}
\label{constraint pl}
Under Assumption \ref{assum:pl}, \ref{assum:x, bounded y}, \ref{assum:smooth}, and with $\eta_yK\leq 1/1000l$, the following inequelity holds
    \begin{align*}
        \hat{f}(x^*(y,z), y^*(x^*(y,z)),z)-\hat{f}(x^*(y,z), y^+(z),z)\leq\frac{2(1+\eta_yKl)}{\mu\eta_y^2K^2}\|y-y^+(z)\|^2.
    \end{align*}
\end{lemma}
\begin{proof}
Noting that $\hat{f}$ is $\mu$-strongly concave in $y$, we have
    \begin{align}\nonumber
        &\hat{f}(x^*(y,z), y^*(x^*(y,z)),z)-\hat{f}(x^*(y,z), y^+(z),z)\\ \nonumber
        \leq& \langle\nabla_y \hat{f}(x^*(y,z), y^+(z),z), y^*(x^*(y,z))-y^+(z)\rangle-\frac{\mu}{2}\|y^*(x^*(y,z))-y^+(z)\|\\ \nonumber
        \leq& \langle\nabla_y \hat{f}(x^*(y,z), y^+(z),z), y^*(x^*(y,z))-y^+(z)\rangle\\ \nonumber
        =& \langle\nabla_y \hat{f}(x^*(y,z), y,z), y^*(x^*(y,z))-y^+(z)\rangle+\\ \nonumber
        &\langle\nabla_y \hat{f}(x^*(y,z), y^+(z),z)-\nabla_y \hat{f}(x^*(y,z), y,z), y^*(x^*(y,z))-y^+(z)\rangle\\ \nonumber
        \overset{(a)}{\leq}& \frac{1}{\eta_yK}\langle y^+(z)-y, y^*(x^*(y,z))-y^+(z)\rangle+\\ \nonumber
        &\frac{1}{\eta_yK}\langle y+\eta_yK\nabla_y \hat{f}(x^*(y,z), y,z)-y^+(z), y^*(x^*(y,z))-y^+(z)\rangle+\\ \nonumber
        &l\|y-y^+(z)\|\|y^*(x^*(y,z))-y^+(z)\|\\ \nonumber
        \overset{(b)}{\leq}& \frac{1+\eta_yKl}{\eta_yK}\|y-y^+(z)\|\|y^*(x^*(y,z))-y^+(z)\|\\ \nonumber
        \overset{(c)}{\leq}&\frac{2(1+\eta_yKl)}{\mu\eta_y^2K^2}\|y-y^+(z)\|^2,
    \end{align}
    where in $(a)$, we use the $l$-smoothness of $f$ and $\nabla_y \hat{f}=\nabla_y f$,  in $(b)$, we use the fact that when $Y$ is a closed convex set, we have
    \begin{align}
        \langle a-P_Y(a), b-P_Y(a)\rangle\leq 0 ,\quad\forall b\in Y,
    \end{align}
    and in $(c)$, we use Lemma \ref{constraint eb}.
\end{proof}

\begin{lemma}
\label{lemma: sc Y<R}
Under Assumption \ref{assum:pl}, \ref{assum:x, bounded y}, \ref{assum:smooth}, and with $\eta_yK\leq 1/1000l$, we have
\begin{align*}
        &\|x^*(y_t,z_t)-x^*(z_t)\|^2\\
        \leq&\frac{10}{\mu\eta_y^2K^2l}\|y_t-\Bar{y}_{t+1}\|^2+\frac{10}{\mu l}\|\nabla_x \hat{f}(x_t, y_t,z_t)\|^2+\frac{40}{l^2}\|\nabla_x \hat{f}(x_t, y_t,z_t)\|^2.
\end{align*}
\end{lemma}
\begin{proof}
Noting that since $\hat{f}(\cdot,  y, z)$ is $(l+p)$-smooth, we have
\begin{align}\nonumber
    &\hat{f}(x^*(y,z), y^+(z),z)-\hat{f}(x^*(y^+(z),z), y^+(z),z)\\ \nonumber
    \leq& \langle\nabla_x \hat{f}(x^*(y^+(z),z), y^+(z),z), x^*(y,z)-x^*(y^+(z),z) \rangle + \frac{l+p}{2}\|x^*(y,z)-x^*(y^+(z),z)\|^2\\ \nonumber
    \leq& \frac{1}{2l}\|\nabla_x \hat{f}(x^*(y^+(z),z), y^+(z),z)\|^2+\frac{2l+p}{2}\|x^*(y,z)-x^*(y^+(z),z)\|^2\\ \nonumber
    \leq& \frac{2}{l}\|\nabla_x \hat{f}(x, y,z)\|^2+\frac{2}{l}\|\nabla_x \hat{f}(x, y,z)-\nabla_x \hat{f}(x^*(y,z), y,z)\|^2+\frac{2}{l}\|\nabla_x \hat{f}(x^*(y,z), y,z)-\nabla_x \hat{f}(x^*(y^+(z),z), y,z)\|^2+\\ \nonumber
    &\frac{2}{l}\|\nabla_x \hat{f}(x^*(y^+(z),z), y,z)-\nabla_x \hat{f}(x^*(y^+(z),z), y^+(z),z)\|^2+ \frac{2l+p}{2}\|x^*(y,z)-x^*(y^+(z),z)\|^2\\ \nonumber
    \leq& \frac{2}{l}\|\nabla_x \hat{f}(x, y,z)\|^2+\frac{2(p+l)^2}{l}\|x-x^*(y,z)\|^2+\left(\frac{2(p+l)^2}{l}+\frac{2l+p}{2}\right)\|x^*(y,z)-x^*(y^+(z),z)\|^2+2l\|y-y^+(z)\|^2\\ \label{constraint sc}
    \overset{(a)}{\leq}&\frac{20}{l}\|\nabla_x \hat{f}(x, y,z)\|^2+(20\gamma_2+2)l\|y-y^+(z)\|^2,
\end{align}
where we use strong convexity of $\hat{f}(\cdot,  y, z)$ and Lemma \ref{lemma: helper lemma of optimal x} to establish $(a)$. By the strong convexity of $\hat{f}(\cdot,  y, z)$, we have
\begin{align*}
        &\|x^*(y,z)-x^*(z)\|^2\\
        \leq&\frac{2}{p-l}[\hat{f}(x^*(y,z), \hat{y}^*(z),z)-\hat{f}(x^*(z),\hat{y}^*(z), z)]\\
        \overset{(a)}{\leq}&\frac{2}{p-l}[\Phi(x^*(y,z),z)-\hat{f}(x^*(y^+(z),z), y^+(z),z)+\hat{f}(x^*(y^+(z),z), y^+(z),z)-\Phi(x^*(z),z)]\\
        \overset{(b)}{\leq}&\frac{2}{p-l}[\hat{f}(x^*(y,z), y^*(x^*(y,z)),z)-\hat{f}(x^*(y^+(z),z), y^+(z),z)]\\
        \overset{(c)}{\leq}&\frac{2}{l}[\hat{f}(x^*(y,z), y^*(x^*(y,z)),z)-\hat{f}(x^*(y,z), y^+(z),z)]+\frac{40}{l^2}\|\nabla_x \hat{f}(x, y,z)\|^2+(40\gamma_2+4)\|y-y^+(z)\|^2\\
        \overset{(d)}{\leq}&\frac{4(1+\eta_yKl)+(40\gamma_2+4)\mu l \eta_y^2K^2}{\mu\eta_y^2K^2l}\|y-y^+(z)\|^2+\frac{40}{l^2}\|\nabla_x \hat{f}(x, y,z)\|^2\\
        \leq& \frac{5}{\mu\eta_y^2K^2l}\|y-y^+(z)\|^2+\frac{40}{l^2}\|\nabla_x \hat{f}(x, y,z)\|^2,
\end{align*}
where $(a)$ is because that $\hat{f}(x^*(y,z), \hat{y}^*(z),z)\leq \Phi(x^*(y,z),z)$, $\Phi(x^*(z),z)=\hat{f}(x^*(z),\hat{y}^*(z), z)$, $(b)$ is because  $\hat{f}(x^*(y^+(z),z),y^+(z),z)\leq\Phi(x^*(z),z)$, $(c)$ is due to \eqref{constraint sc}, and $(d)$ is due to Lemma \ref{constraint pl}. Then, we have
\begin{align*}
        &\|x^*(y_t,z_t)-x^*(z_t)\|^2\\
        \leq& \frac{5}{\mu\eta_y^2K^2l}\|y_t-y_t^+(z_t)\|^2+\frac{40}{l^2}\|\nabla_x \hat{f}(x_t, y_t,z_t)\|^2\\
        \overset{(a)}{\leq}& \frac{10}{\mu\eta_y^2K^2l}\|y_t-\Bar{y}_{t+1}\|^2+\frac{10l}{\mu}\|x_t-x^*(y_t,z_t)\|^2+\frac{40}{l^2}\|\nabla_x \hat{f}(x_t, y_t,z_t)\|^2\\
        \overset{(b)}{\leq}&\frac{10}{\mu\eta_y^2K^2l}\|y_t-\Bar{y}_{t+1}\|^2+\frac{10}{\mu l}\|\nabla_x \hat{f}(x_t, y_t,z_t)\|^2+\frac{40}{l^2}\|\nabla_x \hat{f}(x_t, y_t,z_t)\|^2,
\end{align*}
where in $(a)$, we use $l$-smoothness of $f$ and in $(b)$, we use strong convexity of $\hat{f}(\cdot,  y, z)$.
\end{proof}

\begin{theorem} 
\label{thm:sc constrained y}
Under the Assumptions \ref{assum:smooth}, \ref{assum:bdd_var}, \ref{assum:bdd_hetero}, \ref{assum:phi}, \ref{assum:x, bounded y}, \ref{assum:bounded G_y}, \ref{assum:sc}, if we apply Algorithm \ref{alg1} with $p=2l,\quad \eta_x=\min\{1/(1000Kl), \quad \frac{\sqrt{m\Delta}}{\sigma\sqrt{KTl}}\},\quad \eta_y=\eta_x/256, \quad \beta=\eta_yK\mu/80000$, 
    $\eta_{x,l}\leq \min\{\frac{1}{2l\sqrt{2(2K-1)(K-1)}}, \quad \sqrt{\frac{\beta}{6144\eta_x p l^2K^3}}, \quad O(\epsilon\sqrt{\kappa^{-1} (\sigma_G^2+\sigma^2)}(Kl)^{-1})\}, \quad \eta_{y,l}\leq\min\{\frac{1}{2l\sqrt{2(2K-1)(K-1)}}, \quad \sqrt{\frac{\eta_y}{3072\eta_x l^2K^2}}, \quad O(\epsilon\sqrt{\kappa^{-1} (\sigma_G^2+\sigma^2)}(Kl)^{-1})\}$, 
    when $m=M$ or $\sigma_G=0$, with $T=\Theta(\kappa\epsilon^{-2}), K=\Theta(\kappa m^{-1}\epsilon^{-2})$, we can find an $(\epsilon, \epsilon/\sqrt{\kappa})$-stationary point of $f$ with a per-client sample complexity of $O(\kappa^2m^{-1}\epsilon^{-4})$ and a communication complexity of $O(\kappa\epsilon^{-2})$. Here, $\Delta=V_0-\Phi^*$, $\kappa=l/\mu$.

\end{theorem}

\begin{proof}
    Combining Lemma \ref{lemma: v}, Lemma \ref{lemma: sc Y<R}, with $\beta=\eta_yK\mu/80000$, we have

	\begin{align} \nonumber
	\mathbb{E} V _t -\mathbb{E} V _{t+1} \geq &\left(\frac{\eta_xK}{32}-\frac{480\beta}{\mu}-\frac{1920\beta}{l}\right)\mathbb{E}\|\nabla_x \hat{f}(w_t,z_t)\|^2 +\left(\frac{\eta_yK}{16}-\frac{480\beta }{\mu}\right)\frac{1}{\eta_y^2K^2}\mathbb{E}\|\Bar{y}_{t+1}-y_t\|^2+\frac{p\beta}{16}\mathbb{E}\|x_t-z_t\|^2-\\ \nonumber
    &25l\eta_x^2K^2(M-m)\frac{\sigma_G^2}{mM}-15l\eta_x^2K\frac{\sigma^2}{m}-8\eta_yK(M-m)\frac{\sigma_G^2}{mM}-\frac{4\eta_y\sigma^2}{m}-\\ \nonumber
    &4\eta_xKl^2[24K^2(\eta_{x,l}^2+\eta_{y,l}^2)(\sigma_G^2+G_y^2)+3K(\eta_{x,l}^2+2K\eta_{y,l}^2)\sigma^2]\\ \nonumber
    \geq&\frac{\eta_xK}{64}\mathbb{E}\|\nabla_x \hat{f}(w_t,z_t)\|^2+\frac{1}{32\eta_yK}\mathbb{E}\|\Bar{y}_{t+1}-y_t\|^2+\frac{p\beta}{16}\mathbb{E}\|x_t-z_t\|^2-\\ \nonumber
    &25l\eta_x^2K^2(M-m)\frac{\sigma_G^2}{mM}-15l\eta_x^2K\frac{\sigma^2}{m}-8\eta_yK(M-m)\frac{\sigma_G^2}{mM}-\frac{4\eta_y\sigma^2}{m}-\\ \label{v, sc constrained}
    &4\eta_xKl^2[24K^2(\eta_{x,l}^2+\eta_{y,l}^2)(\sigma_G^2+G_y^2)+3K(\eta_{x,l}^2+2K\eta_{y,l}^2)\sigma^2].
	\end{align}
With $T=mK=\Theta(\kappa\epsilon^{-2}), K=\Theta(\kappa m^{-1}\epsilon^{-2})$, $\eta_x=\min\{1/(1000Kl), \frac{\sqrt{m\Delta}}{\sigma\sqrt{KTl}}\}=\Theta(\kappa^{-1}m\epsilon^2)$, $\beta=\eta_yK\mu/80000=\Theta(\kappa^{-2}m\epsilon^2)$, when $M=m$ or $\sigma_G=0$, and $\eta_{x,l}^2\leq O(\kappa^{-1}\epsilon^2) K^{-2}, \eta_{y,l}^2\leq O(\kappa^{-1}\epsilon^2) K^{-2}$, we have
 
	\begin{align}
 \label{sc nabla_x hat f}
		\frac{1}{T}\sum_{t=0}^{T-1}\mathbb{E}\|\nabla_x \hat{f}(w_t,z_t)\|^2 \leq &\frac{O(1)}{\eta_x KT}\Delta + O(1)\frac{\eta_xl\sigma^2}{m}+O(1)\frac{\sigma^2}{mK}+O(1)\kappa^{-1}\epsilon^2\leq O(1)\kappa^{-1}\epsilon^2
	\end{align}
	\begin{align}
 \label{sc y_t}
		\frac{1}{T}\sum_{t=0}^{T-1}\frac{1}{\eta_y^2K^2}\mathbb{E}\|\Bar{y}_{t+1}-y_t\|^2 \leq &\frac{O(1)}{\eta_x KT}\Delta + O(1)\frac{\eta_xl\sigma^2}{m}+O(1)\frac{\sigma^2}{mK}+O(1)\kappa^{-1}\epsilon^2\leq O(1)\kappa^{-1}\epsilon^2
	\end{align}
	\begin{align}
 \label{sc z_t}
		\frac{1}{T}\sum_{t=0}^{T-1}p^2\mathbb{E}\|x_t-z_t\|^2 \leq &\frac{O(1)\kappa}{\eta_x KT}\Delta + O(1)\frac{\kappa\eta_xl\sigma^2}{m}+O(1)\frac{\kappa\sigma^2}{mK}+O(1)\epsilon^2\leq O(1)\epsilon^2
	\end{align}
 
Because $\eta_yK\leq 1/l$, we have
\begin{align}\nonumber
        &l^2\|P_Y(y_t+1/l\nabla_y f(x_t,y_t))-y_t\|^2\\ \nonumber
        \leq& \frac{1}{\eta_y^2K^2}\|P_Y(y_t+\eta_yK\nabla_y f(x_t,y_t))-y_t\|^2\\ \nonumber
        =&\frac{1}{\eta_y^2K^2}\|\Bar{y}_{t+1}-y_t\|^2.
\end{align}
So, we have 
\begin{align}
    \frac{1}{T}\sum_{t=0}^{T-1}l^2\mathbb{E}\|P_Y(y_t+1/l\nabla_y f(x_t,y_t))-y_t\|^2\leq O(1)\kappa^{-1}\epsilon^2.
\end{align}
According to \eqref{nabla_x f}, we have
\begin{align}
    &\frac{1}{T}\sum_{t=0}^{T-1} \mathbb{E}\|\nabla_x f(x_{t},y_{t})\|^2\leq \frac{1}{T}\sum_{t=0}^{T-1}  2\mathbb{E}\|\nabla_x \hat{f}(w_t,z_t)\|^2+2p^2\mathbb{E}\|x_t-z_t\|^2\leq O(1)\epsilon^2.
\end{align}
Thus, we can find an $(\epsilon,\epsilon/\sqrt{\kappa})$-stationary point of $f$, with $K=O(\kappa m^{-1}\epsilon^{-2}), T=O(\kappa\epsilon^{-2})$, which means a per-client sample complexity of $KT=O(\kappa^2m^{-1}\epsilon^{-4})$ and a communication complexity of $T=O(\kappa\epsilon^{-2})$.
\end{proof}

\begin{corollary}

 \label{thm:sc constrained y, centralized deterministic}
Under the Assumptions \ref{assum:smooth}, \ref{assum:phi}, \ref{assum:x, bounded y}, \ref{assum:bounded G_y}, \ref{assum:sc}, when $M=1$, if we apply 
Algorithm \ref{alg2} with $p=2l,\quad \eta_x=1/(1000Kl),\quad \eta_y=\eta_x/256, \quad \beta=\eta_yK\mu/80000$, we could have
\begin{align*}\nonumber
    \frac{1}{T}\sum_{t=0}^{T-1} \|\nabla_x f(x_{t},y_{t})\|^2+\kappa l^2\|P_Y(y_t+1/l\nabla_y f(x_t,y_t))-y_t\|^2
    \leq& \frac{cl\Delta\kappa}{T},
\end{align*}
where $\Delta=V_0-\Phi^*$, $\kappa=l/\mu$, $c$ is an $O(1)$ constant.
This implies an sample of $O(\kappa\epsilon^{-2})$ to find an $(\epsilon, \epsilon/\sqrt{\kappa})$-stationary point of $f$.
    
\end{corollary}

\begin{proof}
Applying Algorithm \ref{alg2} with $p=2l,\quad \eta_x=1/(1000l),\quad \eta_y=\eta_x/256, \quad \beta=\eta_yK\mu/80000$ is equivalent to applying
Algorithm \ref{alg1} with $m=M=1$, $K=1$, $p=2l,\quad \eta_x=\min\{1/(1000Kl), \quad \frac{\sqrt{m\Delta}}{\sigma\sqrt{KTl}}\},\quad \eta_y=\eta_x/256, \quad \beta=\eta_yK\mu/80000$ and any appropriate $\eta_{x,l}, \eta_{y,l}$. Thus, according to Theorem \ref{thm:sc constrained y} and (\ref{v, sc constrained}), we have
\begin{align} 
    \mathbb{E} V _t -\mathbb{E} V _{t+1}
    \geq&\frac{\eta_x}{64}\mathbb{E}\|\nabla_x \hat{f}(w_t,z_t)\|^2+\frac{1}{32\eta_y}\mathbb{E}\|\Bar{y}_{t+1}-y_t\|^2+\frac{p\beta}{16}\mathbb{E}\|x_t-z_t\|^2
\end{align}
Telescoping and rearranging, we have
\begin{align}
 \label{sc nabla_x hat f, centralized deterministic}
		\frac{1}{T}\sum_{t=0}^{T-1}\|\nabla_x \hat{f}(w_t,z_t)\|^2 \leq &\frac{64}{\eta_x T}\Delta 
\end{align}
\begin{align}
 \label{sc y_t, centralized deterministic}
		\frac{1}{T}\sum_{t=0}^{T-1}\frac{1}{\eta_y^2}\|\Bar{y}_{t+1}-y_t\|^2 \leq &\frac{32}{\eta_y T}\Delta
\end{align}
\begin{align}
 \label{sc z_t, centralized deterministic}
		\frac{1}{T}\sum_{t=0}^{T-1}p^2\|x_t-z_t\|^2 \leq &\frac{16}{Tp\beta}\Delta=\frac{O(1)\kappa}{\eta_x T} \Delta
\end{align}

Because $\eta_y\leq 1/l$, we have
\begin{align}\nonumber
        &l^2\|P_Y(y_t+1/l\nabla_y f(x_t,y_t))-y_t\|^2\\ \nonumber
        \leq& \frac{1}{\eta_y^2}\|P_Y(y_t+\eta_y\nabla_y f(x_t,y_t))-y_t\|^2\\ \nonumber
        =&\frac{1}{\eta_y^2}\|\Bar{y}_{t+1}-y_t\|^2.
\end{align}

Thus, we have
\begin{align*}\nonumber
    &\frac{1}{T}\sum_{t=0}^{T-1} \|\nabla_x f(x_{t},y_{t})\|^2+\kappa l^2\|P_Y(y_t+1/l\nabla_y f(x_t,y_t))-y_t\|^2\\
    \leq& \frac{1}{T}\sum_{t=0}^{T-1}  2\|\nabla_x \hat{f}(w_t,z_t)\|^2+2p^2\|x_t-z_t\|^2+\frac{\kappa}{\eta_y^2}\|\Bar{y}_{t+1}-y_t\|^2\\
    \leq& \frac{O(1)l\Delta\kappa}{T}.
\end{align*}

\end{proof}
\section{Nonconvex-One-Point-Concave}
\label{app: nc-1pc}
\begin{lemma}
\label{lemma: 1pc}
    Under the Assumptions \ref{assum:smooth},  \ref{assum:x, bounded y}, \ref{assum:1pc_y}, we have
\begin{align*}
    \|x^*(z)-x^*(y^+(z),z)\|^2\leq \frac{2(1+\eta_yKl+\eta_yKl\gamma_2)}{\eta_yK(p-l)}\|y-y^+(z)\|D(Y),
\end{align*}
where $D(Y)$ is the diameter of $Y$.
\end{lemma}

\begin{proof}
    Note that Under the Assumption \ref{assum:1pc_y}, we have
    \begin{align}\nonumber
        &\hat{f}(x^*(y^+(z),z), y^*(x^*(y^+(z),z)),z)-\hat{f}(x^*(y^+(z),z), y^+(z),z)\\ \nonumber
        \leq& \langle\nabla_y \hat{f}(x^*(y^+(z),z), y^+(z),z), y^*(x^*(y^+(z),z))-y^+(z)\rangle\\ \nonumber
        =& \langle\nabla_y \hat{f}(x^*(y,z), y,z), y^*(x^*(y^+(z),z))-y^+(z)\rangle+\\ \nonumber
        &\langle\nabla_y \hat{f}(x^*(y^+(z),z), y^+(z),z)-\nabla_y \hat{f}(x^*(y,z), y,z), y^*(x^*(y^+(z),z))-y^+(z)\rangle\\ \nonumber
        \overset{(a)}{\leq}& \frac{1}{\eta_yK}\langle y^+(z)-y, y^*(x^*(y^+(z),z))-y^+(z)\rangle+\\ \nonumber
        &\frac{1}{\eta_yK}\langle y+\eta_yK\nabla_y \hat{f}(x^*(y,z), y,z)-y^+(z), y^*(x^*(y^+(z),z))-y^+(z)\rangle+\\ \nonumber
        &(l+l\gamma_2)\|y-y^+(z)\|\|y^*(x^*(y^+(z),z))-y^+(z)\|\\ \nonumber
        \overset{(b)}{\leq}& \frac{1+\eta_yKl+\eta_yKl\gamma_2}{\eta_yK}\|y-y^+(z)\|\|y^*(x^*(y^+(z),z))-y^+(z)\|\\ \label{1pc 1}
        \leq&\frac{1+\eta_yKl+\eta_yKl\gamma_2}{\eta_yK}\|y-y^+(z)\|D(Y),
    \end{align}
    where in $(a)$, we use the $l$-smoothness of $f$ , $\nabla_y \hat{f}=\nabla_y f$ and Lemma \ref{lemma: helper lemma of optimal x}, in $(b)$, we use the fact that when $Y$ is a closed, convex set, we have
    \begin{align}
        \langle a-P_Y(a), b-P_Y(a)\rangle\leq 0 ,\quad\forall b\in Y. 
    \end{align}
    
Then by the strong convexity of $\hat{f}(\cdot,  y, z)$, we have
\begin{align*}
        &\|x^*(y^+(z),z)-x^*(z)\|^2\\
        \leq&\frac{2}{p-l}[\hat{f}(x^*(y^+(z),z), \hat{y}^*(z),z)-\hat{f}(x^*(z),\hat{y}^*(z), z)]\\
        \overset{(a)}{\leq}&\frac{2}{p-l}[\Phi(x^*(y^+(z),z),z)-\hat{f}(x^*(y^+(z),z), y^+(z),z)+\hat{f}(x^*(y^+(z),z), y^+(z),z)-\Phi(x^*(z),z)]\\
        \overset{(b)}{\leq}&\frac{2}{p-l}[\hat{f}(x^*(y^+(z),z), y^*(x^*(y^+(z),z)),z)-\hat{f}(x^*(y^+(z),z), y^+(z),z)]\\
        \overset{(c)}{\leq}&\frac{2(1+\eta_yKl+\eta_yKl\gamma_2)}{\eta_yK(p-l)}\|y-y^+(z)\|D(Y).
\end{align*}
$(a)$ can be explained by the fact that $\hat{f}(x^*(y^+(z),z), \hat{y}^*(z),z)\leq \Phi(x^*(y^+(z),z),z)$. $(b)$ arises from the relationship $\hat{f}(x^*(y^+(z),z),y^+(z),z)\leq\Phi(x^*(z),z)$. $(c)$ can be attributed to \eqref{1pc 1}.
\end{proof}

\begin{lemma}
\label{lemma: c, epsilon}
Under the Assumptions \ref{assum:smooth},  \ref{assum:x, bounded y}, \ref{assum:1pc_y},  and when $\eta_y\leq 1/(1000Kl), p=2l$, we have
\begin{align*}
    \|x^*(z)-x^*(y^+(z),z)\|^2\leq \frac{6D(Y)}{l\eta_y^2K^2\epsilon^2}\|y^+(z)-y\|^2+\frac{2D(Y)}{l}\epsilon^2.
\end{align*}
\end{lemma}
\begin{proof}
When $\eta_y\leq 1/(1000Kl), p=2l$, we have
\begin{align*}
    &\|x^*(z)-x^*(y^+(z),z)\|^2\\
    \leq& \frac{2(1+\eta_yKl+\eta_yKl\gamma_2)}{\eta_yK(p-l)}\|y-y^+(z)\|D(Y)\\
    \leq& \frac{4}{\eta_yKl}\|y-y^+(z)\|D(Y)\\
    \leq&\frac{2D(Y)}{l\eta_y^2K^2\epsilon^2}\|y^+(z)-y\|^2+\frac{2D(Y)}{l}\epsilon^2,
\end{align*}
where in the second inequality, we use Lemma \ref{lemma: 1pc}.    
\end{proof}
\begin{lemma}
\label{lemma: v y+}
    Under the Assumptions \ref{assum:x, bounded y}, \ref{assum:bounded G_y}, with $p=2l, \eta_x\leq 1/(1000Kl),\eta_y=\eta_x/256,  \beta\leq\eta_yKl/80000$, $\eta_{x,l}\leq \min\{\frac{1}{2l\sqrt{2(2K-1)(K-1)}}, \sqrt{\frac{\beta}{6144\eta_x p l^2K^3}}\}, \eta_{y,l}\leq\min\{\frac{1}{2l\sqrt{2(2K-1)(K-1)}}, \sqrt{\frac{\eta_y}{3072\eta_x l^2K^2}}\}$, we have
    \begin{align*}
    &\mathbb{E} V _t -\mathbb{E} V _{t+1}\\
    \geq& \frac{\eta_xK}{64}\mathbb{E}\|\nabla_x \hat{f}(w_t,z_t)\|^2+\frac{1}{64\eta_yK}\mathbb{E}\|y^+_{t}(z_t)-y_t\|^2+\frac{p\beta}{16}\mathbb{E}\|x_t-z_t\|^2-48p\beta\mathbb{E}\|x^*(z_{t})-x^*(y_{t}^+(z_t), z_{t})\|^2-\\
    &25l\eta_x^2K^2(M-m)\frac{\sigma_G^2}{mM}-15l\eta_x^2K\frac{\sigma^2}{m}-8\eta_yK(M-m)\frac{\sigma_G^2}{mM}-\frac{4\eta_y\sigma^2}{m}-\\
    &4\eta_xKl^2[24K^2(\eta_{x,l}^2+\eta_{y,l}^2)(\sigma_G^2+G_y^2)+3K(\eta_{x,l}^2+2K\eta_{y,l}^2)\sigma^2].
    \end{align*}
\end{lemma}
\begin{proof}
    According to Lemma \ref{lemma: v}, we have
\begin{align*}
    &\mathbb{E} V _t -\mathbb{E} V _{t+1}\\
    \geq&\frac{\eta_xK}{32}\mathbb{E}\|\nabla_x \hat{f}(w_t,z_t)\|^2+\frac{1}{16\eta_yK}\mathbb{E}\|\Bar{y}_{t+1}-y_t\|^2+\frac{p\beta}{16}\mathbb{E}\|x_t-z_t\|^2-24p\beta\mathbb{E}\|x^*(z_{t})-x^*(y_{t}, z_{t})\|^2-\\
    &25l\eta_x^2K^2(M-m)\frac{\sigma_G^2}{mM}-15l\eta_x^2K\frac{\sigma^2}{m}-8\eta_yK(M-m)\frac{\sigma_G^2}{mM}-\frac{4\eta_y\sigma^2}{m}-\\
    &4\eta_xKl^2[24K^2(\eta_{x,l}^2+\eta_{y,l}^2)(\sigma_G^2+G_y^2)+3K(\eta_{x,l}^2+2K\eta_{y,l}^2)\sigma^2]\\
    \overset{(a)}{\geq}&\frac{\eta_xK}{32}\mathbb{E}\|\nabla_x \hat{f}(w_t,z_t)\|^2+\frac{1}{32\eta_yK}\mathbb{E}\|y^+_{t}(z_t)-y_t\|^2-\frac{1}{8\eta_yK}\mathbb{E}\|\Bar{y}_{t+1}-y_t^+(z_t)\|^2+\frac{p\beta}{16}\mathbb{E}\|x_t-z_t\|^2-\\
    &48p\beta\mathbb{E}\|x^*(z_{t})-x^*(y_{t}^+(z_t), z_{t})\|^2-48p\beta\mathbb{E}\|x^*(y_{t}^+(z_t), z_{t})-x^*(y_{t}, z_{t})\|^2-\\
    &25l\eta_x^2K^2(M-m)\frac{\sigma_G^2}{mM}-15l\eta_x^2K\frac{\sigma^2}{m}-8\eta_yK(M-m)\frac{\sigma_G^2}{mM}-\frac{4\eta_y\sigma^2}{m}-\\
    &4\eta_xKl^2[24K^2(\eta_{x,l}^2+\eta_{y,l}^2)(\sigma_G^2+G_y^2)+3K(\eta_{x,l}^2+2K\eta_{y,l}^2)\sigma^2]\\
    \overset{(b)}{\geq}&\frac{\eta_xK}{32}\mathbb{E}\|\nabla_x \hat{f}(w_t,z_t)\|^2+\frac{1}{32\eta_yK}\mathbb{E}\|y^+_{t}(z_t)-y_t\|^2-\frac{l^2}{8\eta_yK}\mathbb{E}\|x^*(y_t,z_t)-x_t\|^2+\frac{p\beta}{16}\mathbb{E}\|x_t-z_t\|^2-\\
    &48p\beta\mathbb{E}\|x^*(z_{t})-x^*(y_{t}^+(z_t), z_{t})\|^2-48p\beta\gamma_2^2\mathbb{E}\|y_{t}^+(z_t)-y_{t}\|^2-\\
    &25l\eta_x^2K^2(M-m)\frac{\sigma_G^2}{mM}-15l\eta_x^2K\frac{\sigma^2}{m}-8\eta_yK(M-m)\frac{\sigma_G^2}{mM}-\frac{4\eta_y\sigma^2}{m}-\\
    &4\eta_xKl^2[24K^2(\eta_{x,l}^2+\eta_{y,l}^2)(\sigma_G^2+G_y^2)+3K(\eta_{x,l}^2+2K\eta_{y,l}^2)\sigma^2]\\
    \overset{(c)}{\geq}& \left(\frac{\eta_xK}{32}-\frac{l^2}{8\eta_yK (p-l)^2}\right)\mathbb{E}\|\nabla_x \hat{f}(w_t,z_t)\|^2+\left(\frac{1}{32\eta_yK}-48p\beta\gamma_2^2\right)\mathbb{E}\|y^+_{t}(z_t)-y_t\|^2+\frac{p\beta}{16}\mathbb{E}\|x_t-z_t\|^2-\\
    &48p\beta\mathbb{E}\|x^*(z_{t})-x^*(y_{t}^+(z_t), z_{t})\|^2-25l\eta_x^2K^2(M-m)\frac{\sigma_G^2}{mM}-15l\eta_x^2K\frac{\sigma^2}{m}-8\eta_yK(M-m)\frac{\sigma_G^2}{mM}-\frac{4\eta_y\sigma^2}{m}-\\
    &4\eta_xKl^2[24K^2(\eta_{x,l}^2+\eta_{y,l}^2)(\sigma_G^2+G_y^2)+3K(\eta_{x,l}^2+2K\eta_{y,l}^2)\sigma^2]\\
    \geq& \frac{\eta_xK}{64}\mathbb{E}\|\nabla_x \hat{f}(w_t,z_t)\|^2+\frac{1}{64\eta_yK}\mathbb{E}\|y^+_{t}(z_t)-y_t\|^2+\frac{p\beta}{16}\mathbb{E}\|x_t-z_t\|^2-48p\beta\mathbb{E}\|x^*(z_{t})-x^*(y_{t}^+(z_t), z_{t})\|^2-\\
    &25l\eta_x^2K^2(M-m)\frac{\sigma_G^2}{mM}-15l\eta_x^2K\frac{\sigma^2}{m}-8\eta_yK(M-m)\frac{\sigma_G^2}{mM}-\frac{4\eta_y\sigma^2}{m}-\\
    &4\eta_xKl^2[24K^2(\eta_{x,l}^2+\eta_{y,l}^2)(\sigma_G^2+G_y^2)+3K(\eta_{x,l}^2+2K\eta_{y,l}^2)\sigma^2],
\end{align*}
where in $(a)$, we use $\|a\|^2\geq \frac{1}{2}\|a-b\|^2-2\|b\|^2$, in $(b)$, we use Lemma \ref{lemma: helper lemma of optimal x}, in $(c)$, we use the $(p-l)$-strongly convexity of $\hat{f}(\cdot, y, z)$.
\end{proof}

\subsection*{Proof of Theorem \ref{thm:1pc}}

\textbf{Theorem \ref{thm:1pc}} 
    Under Assumptions \ref{assum:smooth}, \ref{assum:bdd_var}, \ref{assum:bdd_hetero}, \ref{assum:phi}, \ref{assum:x, bounded y}, \ref{assum:bounded G_y}, \ref{assum:1pc_y} and $\epsilon^2 \leq l D(Y)$, if we apply Algorithm \ref{alg1} with $p=2l,\quad \eta_x=\min\{1/(1000Kl), \quad \frac{\sqrt{m\Delta}}{\sigma\sqrt{KTl}}\},\quad \eta_y=\eta_x/256, \quad \beta=\eta_yK\epsilon^2/(80000D(Y))$, 
    $\eta_{x,l}\leq \min\{\frac{1}{2l\sqrt{2(2K-1)(K-1)}}, \quad \sqrt{\frac{\beta}{6144\eta_x p l^2K^3}}, \quad O(\epsilon^2\sqrt{(\sigma_G^2+\sigma^2)}(Kl)^{-1})\}, \quad \eta_{y,l}\leq\min\{\frac{1}{2l\sqrt{2(2K-1)(K-1)}}, \quad \sqrt{\frac{\eta_y}{3072\eta_x l^2K^2}},  O(\epsilon^2\sqrt{ (\sigma_G^2+\sigma^2)}(Kl)^{-1})\}$, 
    when $m=M$ or $\sigma_G=0$, with $T=\Theta(\epsilon^{-4}), K=\Theta(m^{-1}\epsilon^{-4})$, we can find an $(\epsilon, \epsilon^2)$-stationary point of $f$ and an $\epsilon$-stationary point of $\Phi_{1/2l}$ with a per-client sample complexity of $O(m^{-1}\epsilon^{-8})$ and a communication complexity of $O(\epsilon^{-4})$. Here, $\Delta=V_0-\Phi^*$.

\begin{proof}
    Combining Lemma \ref{lemma: c, epsilon} and \ref{lemma: v y+}, with $\epsilon^2/D(Y)\leq l$, and $\beta=\eta_yK\epsilon^2/(80000D(Y))\leq\eta_yKl/80000$, we have
	\begin{align} \nonumber
	\mathbb{E} V _t -\mathbb{E} V _{t+1} \geq &\frac{\eta_xK}{64}\mathbb{E}\|\nabla_x \hat{f}(w_t,z_t)\|^2 +\left(\frac{\eta_yK}{64}-\frac{192\beta D(Y)}{\epsilon^2}\right)\frac{1}{\eta_y^2K^2}\mathbb{E}\|y_{t}^+(z_t)-y_t\|^2+\\ \nonumber
    &\frac{p\beta}{16}\mathbb{E}\|x_t-z_t\|^2-192\beta D(Y)\epsilon^2-\\ \nonumber
    &25l\eta_x^2K^2(M-m)\frac{\sigma_G^2}{mM}-15l\eta_x^2K\frac{\sigma^2}{m}-8\eta_yK(M-m)\frac{\sigma_G^2}{mM}-\frac{4\eta_y\sigma^2}{m}-\\ \nonumber
    &4\eta_xKl^2[24K^2(\eta_{x,l}^2+\eta_{y,l}^2)(\sigma_G^2+G_y^2)+3K(\eta_{x,l}^2+2K\eta_{y,l}^2)\sigma^2]\\ \nonumber
    \geq&\frac{\eta_xK}{64}\mathbb{E}\|\nabla_x \hat{f}(w_t,z_t)\|^2+\frac{1}{128\eta_yK}\mathbb{E}\|y_{t}^+(z_t)-y_t\|^2+\frac{p\beta}{16}\mathbb{E}\|x_t-z_t\|^2-192\beta D(Y)\epsilon^2-\\ \nonumber
    &25l\eta_x^2K^2(M-m)\frac{\sigma_G^2}{mM}-15l\eta_x^2K\frac{\sigma^2}{m}-8\eta_yK(M-m)\frac{\sigma_G^2}{mM}-\frac{4\eta_y\sigma^2}{m}-\\ \label{v, 1pc, dc}
    &4\eta_xKl^2[24K^2(\eta_{x,l}^2+\eta_{y,l}^2)(\sigma_G^2+G_y^2)+3K(\eta_{x,l}^2+2K\eta_{y,l}^2)\sigma^2].
	\end{align}
Choosing $T=mK=\Theta(\epsilon^{-4}), K=\Theta(m^{-1}\epsilon^{-4})$, $\eta_x=\min\{1/(1000Kl), \frac{\sqrt{m\Delta}}{\sigma\sqrt{KTl}}\}=\Theta(m\epsilon^4)$, when $M=m$ or $\sigma_G=0$, and $\eta_{x,l}^2\leq O(\epsilon^4) K^{-2}, \eta_{y,l}^2\leq O(\epsilon^4) K^{-2}$, we have
 
	\begin{align}
 \label{c nabla_x hat f}
		\frac{1}{T}\sum_{t=0}^{T-1}\mathbb{E}\|\nabla_x \hat{f}(w_t,z_t)\|^2 \leq &\frac{O(1)}{\eta_x KT}\Delta + O(1)\frac{\eta_xl\sigma^2}{m}+O(1)\frac{\sigma^2}{mK}+O(1)\epsilon^4\leq O(1)\epsilon^4,
	\end{align}
	\begin{align}
 \label{c y+-y}
		\frac{1}{T}\sum_{t=0}^{T-1}\frac{1}{\eta_y^2K^2}\mathbb{E}\|y_{t}^+(z_t)-y_t\|^2 \leq &\frac{O(1)}{\eta_x KT}\Delta + O(1)\frac{\eta_xl\sigma^2}{m}+O(1)\frac{\sigma^2}{mK}+O(1)\epsilon^4\leq O(1)\epsilon^4,
	\end{align}
	\begin{align}
 \label{c z_t}
		\frac{1}{T}\sum_{t=0}^{T-1}p^2\mathbb{E}\|x_t-z_t\|^2 \leq &\frac{O(1)}{\epsilon^2 \eta_x KT}\Delta + O(1)\frac{\eta_xl\sigma^2}{m\epsilon^2}+O(1)\frac{\sigma^2}{mK\epsilon^2}+O(1)\epsilon^2\leq O(1)\epsilon^2.
	\end{align}
Combining \eqref{nabla_x f}, \eqref{c nabla_x hat f}, \eqref{c z_t}, we have
\begin{align}
    &\frac{1}{T}\sum_{t=0}^{T-1} \mathbb{E}\|\nabla_x f(x_{t},y_{t})\|^2\leq \frac{1}{T}\sum_{t=0}^{T-1} 2\mathbb{E}\|\nabla_x \hat{f}(w_t,z_t)\|^2+2p^2\mathbb{E}\|x_t-z_{t}\|^2\leq O(\epsilon^2).
\end{align}
Since $\eta_yK\leq 1/l$, we have
\begin{align}\nonumber
        &l^2\|P_Y(y_t+1/l\nabla_y f(x_t,y_t))-y_t\|^2\\ \nonumber
        \leq& \frac{1}{\eta_y^2K^2}\|P_Y(y_t+\eta_yK\nabla_y f(x_t,y_t))-y_t\|^2\\ \nonumber
        \leq&\frac{2}{\eta_y^2K^2}\|P_Y(y_t+\eta_yK\nabla_y \hat{f}(x^*(y_t, z_t),y_t))-y_t\|^2+\\ \nonumber
        &\frac{2}{\eta_y^2K^2}\|P_Y(y_t+\eta_yK\nabla_y \hat{f}(x^*(y_t, z_t),y_t))-P_Y(y_t+\eta_yK\nabla_y f(x_t,y_t))\|^2\\ \nonumber
        \leq&\frac{2}{\eta_y^2K^2}\|y^+_t(z_t)-y_t\|^2+2l^2\|x^*(y_t, z_t)-x_t\|^2\\ \nonumber
        \leq&\frac{2}{\eta_y^2K^2}\|y^+_t(z_t)-y_t\|^2+2l^2/(p-l)^2\|\nabla_x \hat{f}(w_t,z_t)\|^2\\ \label{relationship between y and y^+}
        \leq&\frac{2}{\eta_y^2K^2}\|y^+_t(z_t)-y_t\|^2+2\|\nabla_x \hat{f}(w_t,z_t)\|^2.
\end{align}

Combining \eqref{relationship between y and y^+},\eqref{c nabla_x hat f}, \eqref{c y+-y}, we have
\begin{align}
    \frac{1}{T}\sum_{t=0}^{T-1} \mathbb{E}l^2\|P_Y(y_t+1/l\nabla_y f(x_t,y_t))-y_t\|^2\leq \frac{1}{T}\sum_{t=0}^{T-1} 2\mathbb{E}\|\nabla_x \hat{f}(w_t,z_t)\|^2+\frac{2}{\eta_y^2K^2}\mathbb{E}\|y^+_t(z_t)-y_t\|^2\leq O(\epsilon^4).
\end{align}
According to Lemma \ref{stationary of Phi}, we have 
\begin{align}\nonumber
    &\|\nabla_x \Phi_{1/2l}(x_t)\|^2=p^2\|x_t-x^*(x_t)\|^2\\ \nonumber
    \leq&  4p^2\|x_t-x^*(y_t,z_t)\|^2+4p^2\|x^*(y_t,z_t)-x^*(y_t^+(z_t),z_t)\|^2+\\ \nonumber
    &4p^2\|x^*(y_t^+(z_t),z_t)-x^*(z_t)\|^2+4p^2\|x^*(z_t)-x^*(x_t)\|^2\\ \nonumber
    \leq&\frac{4p^2}{(p-l)^2}\|\nabla_x \hat{f}(w_t,z_t)\|^2+4p^2\gamma_2^2\|y_t-y_t^+(z_t)\|^2+\\ \label{envelope y+}
    &4p^2\bigg\{\frac{4D(Y)}{l\eta_y^2K^2\epsilon^2}\|y_t^+(z_t)-y_t\|^2+\frac{2D(Y)}{l}\epsilon^2\bigg\}+4p^2\gamma_1^2\|z_t-x_t\|^2,
\end{align}
where in the second inequality, we use the $(p-l)$-strongly convexity of $\hat{f}(\cdot, y,z)$, Lemma \ref{lemma: helper lemma of optimal x} and Lemma \ref{lemma: c, epsilon}. Combining \eqref{envelope y+}, \eqref{c nabla_x hat f}, \eqref{c y+-y}, \eqref{c z_t}, we further have
\begin{align}
    \frac{1}{T}\sum_{t=0}^{T-1} \mathbb{E}\|\nabla_x \Phi_{1/2l}(x_t)\|^2\leq O(\epsilon^2).
\end{align}
Hence, we can identify an $(\epsilon,\epsilon^2)$-stationary point for $f$ and an $\epsilon$-stationary point for $\Phi_{1/2l}$, with respective values of $K=\Theta(m^{-1}\epsilon^{-4})$ and $T=\Theta(\epsilon^{-4})$. This results in a per-client sample complexity of $KT=O(m^{-1}\epsilon^{-8})$ and a communication complexity of $T=O(\epsilon^{-4})$.
\end{proof}

\subsection*{Proof of Corollary \ref{thm:1pc, dc}}

\textbf{Corollary \ref{thm:1pc, dc}} 
Under Assumptions \ref{assum:smooth}, \ref{assum:phi}, \ref{assum:x, bounded y}, \ref{assum:bounded G_y}, \ref{assum:1pc_y}, and when $M=1$, $\epsilon^2 \leq l D(Y)$, if we apply Algorithm \ref{alg2} with $p=2l,\quad \eta_x=1/(1000l),\quad \eta_y=\eta_x/256, \quad \beta=\eta_y\epsilon^2/(80000D(Y))$, $T=\Theta(\epsilon^{-4})$, we can find an $(\epsilon, \epsilon^2)$-stationary point of $f$ and an $\epsilon$-stationary point of $\Phi_{1/2l}$ with a sample complexity of $O(\epsilon^{-4})$.

\begin{proof}
Applying Algorithm \ref{alg2} with $p=2l,\quad \eta_x=1/(1000l),\quad \eta_y=\eta_x/256, \quad \beta=\eta_y\epsilon^2/(80000D(Y))$ is equivalent to applying
Algorithm \ref{alg1} with $p=2l,\quad \eta_x=\min\{1/(1000Kl), \quad \frac{\sqrt{m\Delta}}{\sigma\sqrt{KTl}}\},\quad \eta_y=\eta_x/256, \quad \beta=\eta_yK\epsilon^2/(80000D(Y))$ and any appropriate $\eta_{x,l}, \eta_{y,l}$. Thus, according to Theorem \ref{thm:1pc} and (\ref{v, 1pc, dc}), we have
\begin{align} 
    \mathbb{E} V _t -\mathbb{E} V _{t+1}
    \geq&\frac{\eta_x}{64}\mathbb{E}\|\nabla_x \hat{f}(w_t,z_t)\|^2+\frac{1}{128\eta_y}\mathbb{E}\|y_t^+(z_t)-y_t\|^2+\frac{p\beta}{16}\mathbb{E}\|x_t-z_t\|^2-192\beta D(Y)\epsilon^2
\end{align}
Telescoping and rearranging, we have
\begin{align}
		\frac{1}{T}\sum_{t=0}^{T-1}\|\nabla_x \hat{f}(w_t,z_t)\|^2 \leq &\frac{64}{\eta_x T}\Delta +O(1)\epsilon^4\leq O(\epsilon^4)
\end{align}
\begin{align}
		\frac{1}{T}\sum_{t=0}^{T-1}\frac{1}{\eta_y^2}\|y_t^+(z_t)-y_t\|^2 \leq &\frac{128}{\eta_y T}\Delta+O(1)\epsilon^4\leq O(\epsilon^4)
\end{align}
\begin{align}
		\frac{1}{T}\sum_{t=0}^{T-1}p^2\|x_t-z_t\|^2 \leq &\frac{16}{Tp\beta}\Delta+O(1)\epsilon^2\leq O(\epsilon^2)
\end{align}
Thus, we have
\begin{align}
    &\frac{1}{T}\sum_{t=0}^{T-1} \|\nabla_x f(x_{t},y_{t})\|^2\leq \frac{1}{T}\sum_{t=0}^{T-1} 2\|\nabla_x \hat{f}(w_t,z_t)\|^2+2p^2\|x_t-z_{t}\|^2\leq O(\epsilon^2).
\end{align}
According to (\ref{relationship between y and y^+}), we have
\begin{align}
    \frac{1}{T}\sum_{t=0}^{T-1} l^2\|P_Y(y_t+1/l\nabla_y f(x_t,y_t))-y_t\|^2\leq \frac{1}{T}\sum_{t=0}^{T-1} 2\|\nabla_x \hat{f}(w_t,z_t)\|^2+\frac{2}{\eta_y^2}\|y^+_t(z_t)-y_t\|^2\leq O(\epsilon^4).
\end{align}
With (\ref{envelope y+}), we have
\begin{align}\nonumber
    &\frac{1}{T}\sum_{t=0}^{T-1}\|\nabla_x \Phi_{1/2l}(x_t)\|^2\\ \nonumber
    \leq&\frac{1}{T}\sum_{t=0}^{T-1}\Bigg[\frac{4p^2}{(p-l)^2}\|\nabla_x \hat{f}(w_t,z_t)\|^2+4p^2\gamma_2^2\|y_t-y_t^+(z_t)\|^2+\\ 
    &4p^2\bigg\{\frac{4D(Y)}{l\eta_y^2K^2\epsilon^2}\|y_t^+(z_t)-y_t\|^2+\frac{2D(Y)}{l}\epsilon^2\bigg\}+4p^2\gamma_1^2\|z_t-x_t\|^2\Bigg]\leq O(\epsilon^2)
\end{align}
Thus, we can identify an $(\epsilon,\epsilon^2)$-stationary point for $f$ and an $\epsilon$-stationary point for $\Phi_{1/2l}$ with a sample complexity of $T=O(\epsilon^{-4})$.
\end{proof}

\section{Nonconvex-Concave}
\label{app: nc-c}
\subsection*{Proof of Theorem \ref{thm:c f} and Corollary \ref{thm:c f, dc}}
\begin{proof}
We define $\Tilde{f}(x,y)=f(x,y)-\frac{\epsilon}{4D(Y)}\|y-y_0\|^2$. Then $\Tilde{f}$ is $(l+\epsilon/2D_Y)$-smooth, and $\epsilon/2D(Y)$-strongly-concave. When $\epsilon\leq 2l D(Y)$, $\Tilde{f}$ is $2l$-smooth, we have 
$$\kappa'=\frac{2l D(Y)}{\epsilon}=O(\epsilon^{-1}).$$
Proof of Theorem \ref{thm:c f}: According to Theorem \ref{thm:sc constrained y}, applying Algorithm \ref{alg1} to optimize $\Tilde{f}$, with  $M=m$ or $\sigma_G=0$, we can find $(\hat{x}, \hat{y})$, an $(\epsilon,\epsilon/\sqrt{\kappa'})$-stationary point of $\Tilde{f}$, with a sample complexity of $O(\kappa'^2n^{-1}\epsilon^{-4})=O(n^{-1}\epsilon^{-6})$ and a communication complexity of $O(\kappa'^{1}\epsilon^{-2})=O(\epsilon^{-3})$.

$(\hat{x},\hat{y})$ is an $(\epsilon,\epsilon/\sqrt{\kappa'})$-stationary point of $\Tilde{f}$ means
\begin{align*}
    &\|\nabla_x\Tilde{f}(\hat{x},\hat{y})\|\leq \epsilon\\
    &\|\nabla_y\Tilde{f}(\hat{x},\hat{y})\|\leq \epsilon/\sqrt{\kappa'}\leq \epsilon.
\end{align*}
By the inequality $\max_{x\in X, y \in Y}\|\nabla f(x,y)-\nabla\Tilde{f}(x,y)\|\leq \epsilon/2$, we have
\begin{align*}
    \|\nabla f(\hat{x},\hat{y})\|\leq& \|\nabla\Tilde{f}(\hat{x},\hat{y})\|+\|\nabla f(\hat{x},\hat{y})-\nabla\Tilde{f}(\hat{x},\hat{y})\|
    \leq 2\epsilon.
\end{align*}
Therefore, $(\hat{x},\hat{y})$ is a $O(\epsilon)$-stationary point of $f$. We can find an $\epsilon$-stationary point of $f$ with a per-client sample complexity of $O(n^{-1}\epsilon^{-6})$ and a communication complexity of $O(\epsilon^{-3})$.

Proof of Corollary \ref{thm:c f, dc}: Similarly, according to Corollary \ref{thm:sc constrained y, centralized deterministic}, applying Algorithm \ref{alg2} to optimize $\Tilde{f}$, with  $M=1$, we can find $(\hat{x}, \hat{y})$, an $(\epsilon,\epsilon/\sqrt{\kappa'})$-stationary point of $\Tilde{f}$, with a sample complexity of $O(\kappa'\epsilon^{-2})=O(\epsilon^{-3})$.

$(\hat{x},\hat{y})$ is an $(\epsilon,\epsilon/\sqrt{\kappa'})$-stationary point of $\Tilde{f}$ means
\begin{align*}
    &\|\nabla_x\Tilde{f}(\hat{x},\hat{y})\|\leq \epsilon\\
    &\|\nabla_y\Tilde{f}(\hat{x},\hat{y})\|\leq \epsilon/\sqrt{\kappa'}\leq \epsilon.
\end{align*}
By the inequality $\max_{x\in X, y \in Y}\|\nabla f(x,y)-\nabla\Tilde{f}(x,y)\|\leq \epsilon/2$, we have
\begin{align*}
    \|\nabla f(\hat{x},\hat{y})\|\leq& \|\nabla\Tilde{f}(\hat{x},\hat{y})\|+\|\nabla f(\hat{x},\hat{y})-\nabla\Tilde{f}(\hat{x},\hat{y})\|
    \leq 2\epsilon.
\end{align*}
Therefore, $(\hat{x},\hat{y})$ is a $O(\epsilon)$-stationary point of $f$. We can find an $\epsilon$-stationary point of $f$ with a sample complexity of $O(\epsilon^{-3})$.
\end{proof}

\section{Minimizing the Point-Wise Maximum of Finite Functions}
\label{app: special case}
\begin{lemma}[Lemma B13\citep{zhang2020single}]\label{dual error bound}
Let $x^+(y, z)=x-\eta_xK\nabla_x\hat{f}(x, y, z ).$
If Assumption \ref{assum: strict complementarity} holds for problem~\eqref{special case}, then there exists $\delta>0$, such that if $\|z\|$ is bounded by a constant $D_z$ as
$$\|z\|\le D_z,$$
and
$$\max\{\|x-x^+(y, z)\|, \|y-y^+(z)\|, \|x^+(y, z)-z\|\}<\delta,$$
we have
$$\|x(y^+(z), z)-x^*(z)\|< \gamma_3\|y-y^+(z)\|$$
for  some constant $\gamma_3>0$.

\end{lemma}

\subsection*{Proof of Theorem \ref{thm:special case}}

\textbf{Theorem \ref{thm:1pc}} 
    Under Assumptions \ref{assum:smooth}, \ref{assum:bdd_var}, \ref{assum:bdd_hetero}, \ref{assum:phi}, \ref{assum:bounded G_y}, \ref{assum: strict complementarity}, if we apply Algorithm \ref{alg1} with $p=2l,\quad \eta_x=\min\{1/(1000Kl), \quad \frac{\sqrt{m\Delta}}{\sigma\sqrt{KTl}}\},\quad \eta_y=\eta_x/256, \quad \beta=\min\{\eta_yKl/80000, \delta/2\lambda_1, \delta/4\lambda_2, \eta_yK\delta/4\lambda_3, \frac{1}{6144p\eta_yK\gamma_3^2}\}$, 
    $\eta_{x,l}\leq \min\{\frac{1}{2l\sqrt{2(2K-1)(K-1)}}, \sqrt{\frac{\beta}{6144\eta_x p l^2K^3}},O(\epsilon \sqrt{ (\sigma_G^2+\sigma^2)}(Kl)^{-1})\}$, $ \eta_{y,l}\leq\min\{\frac{1}{2l\sqrt{2(2K-1)(K-1)}}, \sqrt{\frac{\eta_y}{3072\eta_x l^2K^2}},O(\epsilon \sqrt{ (\sigma_G^2+\sigma^2)}(Kl)^{-1})\}$, 
    when $m=M$ or $\sigma_G=0$, with $T=\Theta(\epsilon^{-2}), K=\Theta(m^{-1}\epsilon^{-2})$, we can find an $\epsilon$-stationary point of $f$ and an $\epsilon$-stationary point of $\Phi_{1/2l}$ with a per-client sample complexity of $O(m^{-1}\epsilon^{-4})$ and a communication complexity of $O(\epsilon^{-2})$. Here, $\Delta=V_0-\Phi^*$, $\delta, \lambda_1,\lambda_2,\lambda_3$ are $O(1)$ constants defined in following proof.

\begin{proof}
According to Lemma \ref{lemma: v y+}, we have
    \begin{align*}
    &\mathbb{E} V _t -\mathbb{E} V _{t+1}\\
    \geq& \frac{\eta_xK}{64}\mathbb{E}\|\nabla_x \hat{f}(w_t,z_t)\|^2+\frac{1}{64\eta_yK}\mathbb{E}\|y^+_{t}(z_t)-y_t\|^2+\frac{p\beta}{16}\mathbb{E}\|x_t-z_t\|^2-48p\beta\mathbb{E}\|x^*(z_{t})-x^*(y_{t}^+(z_t), z_{t})\|^2-\\
    &25l\eta_x^2K^2(M-m)\frac{\sigma_G^2}{mM}-15l\eta_x^2K\frac{\sigma^2}{m}-8\eta_yK(M-m)\frac{\sigma_G^2}{mM}-\frac{4\eta_y\sigma^2}{m}-\\
    &4\eta_xKl^2[24K^2(\eta_{x,l}^2+\eta_{y,l}^2)(\sigma_G^2+G_y^2)+3K(\eta_{x,l}^2+2K\eta_{y,l}^2)\sigma^2].
    \end{align*}
With $T=mK=\Theta(\epsilon^{-2}), K=\Theta(m^{-1}\epsilon^{-2})$, $\eta_x=\min\{1/(1000Kl), \frac{\sqrt{m\Delta}}{\sigma\sqrt{KTl}}\}=\Theta(m\epsilon^2), \beta\leq\eta_yKl/80000$, when $M=m$ or $\sigma_G=0$, and $\eta_{x,l}^2\leq O(\epsilon^2) K^{-2}, \eta_{y,l}^2\leq O(\epsilon^2) K^{-2}$, we have
    \begin{align}\nonumber
    &\mathbb{E} V _t -\mathbb{E} V _{t+1}\\ \nonumber
    \geq& \frac{\eta_xK}{64}\mathbb{E}\|\nabla_x \hat{f}(w_t,z_t)\|^2+\frac{1}{64\eta_yK}\mathbb{E}\|y^+_{t}(z_t)-y_t\|^2+\frac{p\beta}{16}\mathbb{E}\|x_t-z_t\|^2-48p\beta\mathbb{E}\|x^*(z_{t})-x^*(y_{t}^+(z_t), z_{t})\|^2-\\ \label{v with error}
    &O(1)\epsilon^2.
    \end{align}
Note that we assume $\|x_t\|\leq D_x$ for all $t$, we define $D_z=\max\{D_x, \|z_0\|\}$, then we can prove that, for all $t$, $\|z_t\|\leq D_z$, we prove it by induction. First for $t=0$, $\|z_0\|\leq D_z$, we assume when $t=i$, we have $\|z_i\|\leq D_z$, then for $t=i+1$, we have $\|z_{i+1}\|\leq \beta\|x_i\|+(1-\beta)\|z_i\|\leq \beta D_x+(1-\beta)D_z\leq D_z$. So, for all $t$, we have $\|z_t\|\leq D_z$.

Next, we will prove that for all $t$, we have
\begin{align}
\label{v decent}
    \mathbb{E} V _t -\mathbb{E} V _{t+1}\geq& \frac{\eta_xK}{128}\mathbb{E}\|\nabla_x \hat{f}(w_t,z_t)\|^2+\frac{1}{128\eta_yK}\mathbb{E}\|y^+_{t}(z_t)-y_t\|^2+\frac{p\beta}{32}\mathbb{E}\|x_t-z_t\|^2-O(1)\epsilon^2.
\end{align}
For any $t$, there are two cases.
\begin{itemize}
    \item \textbf{Case 1:}
    \begin{align}
    \label{first situation}
        \frac{1}{2}\max\{\frac{\eta_xK}{128}\|\nabla_x \hat{f}(w_t,z_t)\|^2, \frac{1}{128\eta_yK}\|y_t^+(z_t)-y_t\|^2, \frac{p\beta}{32}\|x_t-z_t\|^2\}\leq 48p\beta\|x^*(z_{t})-x^*(y_{t}^+(z_t), z_{t})\|^2.
    \end{align}
    \item \textbf{Case 2:}
    \begin{align}
    \label{second situation}
        \frac{1}{2}\max\{\frac{\eta_xK}{128}\|\nabla_x \hat{f}(w_t,z_t)\|^2, \frac{1}{128\eta_yK}\|y_t^+(z_t)-y_t\|^2, \frac{p\beta}{32}\|x_t-z_t\|^2\}\geq 48p\beta\|x^*(z_{t})-x^*(y_{t}^+(z_t), z_{t})\|^2.
    \end{align}
    
\end{itemize}
For \textbf{Case 1}, combining \eqref{first situation} and Lemma \ref{lemma: 1pc}, we have
\begin{align*}
    &\|x^*(z_{t})-x^*(y_{t}^+(z_t), z_{t})\|\leq \frac{2(1+\eta_yKl+\eta_yKl\gamma_2)}{\eta_yK(p-l)}\|y_t-y_t^+(z_t)\|D(Y)\leq \frac{4D(Y)}{\eta_yKl}\|y_t^+(z_t)-y_t\|\\
    &\frac{1}{256\eta_yK}\|y_t^+(z_t)-y_t\|^2\leq 48p\beta\|x^*(z_{t})-x^*(y_{t}^+(z_t), z_{t})\|^2\leq \frac{192p\beta D(Y)}{\eta_yK l}\|y_t^+(z_t)-y_t\|\\
    &\|y_t^+(z_t)-y_t\|\leq O(1)\beta=\lambda_1 \beta\\
    &\|x^*(z_{t})-x^*(y_{t}^+(z_t), z_{t})\|\leq\frac{4D(Y)}{\eta_yKl}\|y_t^+(z_t)-y_t\|\leq O(1)\frac{\beta}{\eta_yK}.
\end{align*}
This leads to the following results:
\begin{align*}
    \|x_t-x_t^+(y_t, z_t)\|=&\eta_xK\|\nabla_x \hat{f}(w_t,z_t)\|\\
    \leq& O(1)\sqrt{p\beta\eta_xK}\|x^*(z_{t})-x^*(y_{t}^+(z_t), z_{t})\|\leq O(1)\frac{\sqrt{\beta \eta_x K}\beta}{\eta_yK}=O(1)\beta=\lambda_2 \beta,\\
    \|x_t-z_t\|\leq&O(1)\|x^*(z_{t})-x^*(y_{t}^+(z_t), z_{t})\|\leq O(1)\frac{\beta}{\eta_yK}=\lambda_3\frac{\beta}{\eta_yK},\\
    \|x_t^+(y_t, z_t)-z_t\|\leq& \eta_xK\|\nabla_x \hat{f}(w_t,z_t)\|+\|x_t-z_t\|\leq \lambda_2\beta+\lambda_3\frac{\beta}{\eta_yK}.
\end{align*}
Therefore, if we choose $\beta=\min\{\eta_yKl/80000, \delta/2\lambda_1, \delta/4\lambda_2, \eta_yK\delta/4\lambda_3\}$, we will have
\begin{align*}
    \max\{\|x-x^+(y, z)\|, \|y-y^+(z)\|, \|x^+(y, z)-z\|\}<\delta.
\end{align*}
According to Lemma \ref{dual error bound}, with $\|z_t\|\leq D_z$ and  $\beta=\min\{\eta_yKl/80000, \delta/2\lambda_1, \delta/4\lambda_2, \eta_yK\delta/4\lambda_3, \frac{1}{6144p\eta_yK\gamma_3^2}\}$, where $\delta, \lambda_1,\lambda_2,\lambda_3$ are all $O(1)$ constants and are independent of $\epsilon$, we have
\begin{align}
\label{error solved}
    48p\beta\mathbb{E}\|x^*(z_{t})-x^*(y_{t}^+(z_t), z_{t})\|^2\leq& 48p\beta\gamma_3^2 \mathbb{E}\|y_t^+(z_t)-y_t\|^2\leq \frac{1}{128\eta_yK}\mathbb{E}\|y_t^+(z_t)-y_t\|^2.
\end{align}
Combining \eqref{v with error} and \eqref{error solved}, we get the \eqref{v decent}.  
In \textbf{Case 2}, we can easily get the \eqref{v decent}. Combining these together, for all $t$, we have
\begin{align}
    \mathbb{E} V _t -\mathbb{E} V _{t+1}\geq& \frac{\eta_xK}{128}\mathbb{E}\|\nabla_x \hat{f}(w_t,z_t)\|^2+\frac{1}{128\eta_yK}\mathbb{E}\|y^+_{t}(z_t)-y_t\|^2+\frac{p\beta}{32}\mathbb{E}\|x_t-z_t\|^2-O(1)\epsilon^2.
\end{align}
Note that $\beta=\min\{\eta_yKl/80000, \delta/2\lambda_1, \delta/4\lambda_2, \eta_yK\delta/4\lambda_3, \frac{1}{6144p\eta_yK\gamma_3^2}\}$, where $\delta, \lambda_1,\lambda_2,\lambda_3$ are all $O(1)$ constants and are independent of $\epsilon$, so $\beta$ is also an $O(1)$ constant and is independent of $\epsilon$, we have
\begin{align}
 \label{specical nabla_x hat f}
		\frac{1}{T}\sum_{t=0}^{T-1}\mathbb{E}\|\nabla_x \hat{f}(w_t,z_t)\|^2 \leq &\frac{O(1)}{\eta_x KT}\Delta + O(1)\epsilon^2\leq O(1)\epsilon^2,
	\end{align}
	\begin{align}
 \label{specical y+-y}
		\frac{1}{T}\sum_{t=0}^{T-1}\frac{1}{\eta_y^2K^2}\mathbb{E}\|y_{t}^+(z_t)-y_t\|^2 \leq &\frac{O(1)}{\eta_x KT}\Delta + O(1)\epsilon^2\leq O(1)\epsilon^2,
	\end{align}
	\begin{align}
 \label{specical z_t}
		\frac{1}{T}\sum_{t=0}^{T-1}p^2\mathbb{E}\|x_t-z_t\|^2 \leq &\frac{O(1)}{T}\Delta + O(1)\epsilon^2\leq O(1)\epsilon^2.
	\end{align}
Combining \eqref{nabla_x f}, \eqref{c nabla_x hat f}, \eqref{c z_t}, we have
\begin{align}
    &\frac{1}{T}\sum_{t=0}^{T-1} \mathbb{E}\|\nabla_x f(x_{t},y_{t})\|^2\leq \frac{1}{T}\sum_{t=0}^{T-1} 2\mathbb{E}\|\nabla_x \hat{f}(w_t,z_t)\|^2+2p^2\mathbb{E}\|x_t-z_{t}\|^2\leq O(\epsilon^2).
\end{align}
Combining \eqref{relationship between y and y^+},\eqref{specical nabla_x hat f}, \eqref{specical y+-y} yields
\begin{align}
    \frac{1}{T}\sum_{t=0}^{T-1} \mathbb{E}l^2\|P_Y(y_t+1/l\nabla_y f(x_t,y_t))-y_t\|^2\leq \frac{1}{T}\sum_{t=0}^{T-1} 2\mathbb{E}\|\nabla_x \hat{f}(w_t,z_t)\|^2+\frac{2}{\eta_y^2K^2}\mathbb{E}\|y^+_t(z_t)-y_t\|^2\leq O(\epsilon^2).
\end{align}
Note that from previous proof, for any $0\leq t<T$, we have
\begin{align}
\label{error for envelope}
    48p\beta \mathbb{E}\|x^*(z_{t})-x^*(y_{t}^+(z_t), z_{t})\|^2\leq \frac{1}{256\eta_yK}\mathbb{E}\|y_t^+(z_t)-y_t\|^2.
\end{align}
According to Lemma \ref{stationary of Phi}, we have
\begin{align}\nonumber
    &\|\nabla_x \Phi_{1/2l}(x_t)\|^2=p^2\|x_t-x^*(x_t)\|^2\\ \nonumber
    \leq&  4p^2\|x_t-x^*(y_t,z_t)\|^2+4p^2\|x^*(y_t,z_t)-x^*(y_t^+(z_t),z_t)\|^2+\\ \nonumber
    &4p^2\|x^*(y_t^+(z_t),z_t)-x^*(z_t)\|^2+4p^2\|x^*(z_t)-x^*(x_t)\|^2\\ \nonumber
    \leq&\frac{4p^2}{(p-l)^2}\|\nabla_x \hat{f}(w_t,z_t)\|^2+4p^2\gamma_2^2\|y_t-y_t^+(z_t)\|^2+\\ \label{envelope special case}
    &\frac{p}{\eta_yK\beta}\|y_t-y_t^+(z_t)\|^2+4p^2\gamma_1^2\|z_t-x_t\|^2,
\end{align}
where in the second inequality, we use the $(p-l)$-strongly convexity of $\hat{f}(\cdot, y,z)$, Lemma \ref{lemma: helper lemma of optimal x} and \eqref{error for envelope}. Combining \eqref{envelope special case}, \eqref{specical nabla_x hat f}, \eqref{specical y+-y}, \eqref{specical z_t}, we have
\begin{align}
    \frac{1}{T}\sum_{t=0}^{T-1} \mathbb{E}\|\nabla_x \Phi_{1/2l}(x_t)\|^2\leq O(\epsilon^2).
\end{align}
Therefore, we can find an $(\epsilon,\epsilon^2)$-stationary point of $f$ and an $\epsilon$-stationary point of $\Phi$, with $K=\Theta(m^{-1}\epsilon^{-2}), T=\Theta(\epsilon^{-2})$, which means a per-client sample complexity of $KT=O(m^{-1}\epsilon^{-4})$ and a communication complexity of $T=O(\epsilon^{-2})$.
\end{proof}

\section{PL-PL}
\label{app: plpl}

Since $p=0$ and $Y=\mathbb{R}^{d_2}$ in this section, the updates of \alg are:
\begin{align*}
    &x_{t+1}=x_t-\eta_xK(u_{x,t}-e_{x,t}),\\
    &y_{t+1}=y_t+\eta_yK(u_{y,t}-e_{y,t}).
\end{align*}

We cite the following known results for ease of exposition.

\begin{lemma} [\citet{nouiehed2019solving}]  \label{lemma: Phi}
In the minimax problem, when $-f(x,\cdot)$ satisfies PL condition with constant $\mu_2$ for any $x$ and $f$ satisfies Assumption \ref{assum:smooth}, then the function $\Phi(x) := \max_y f(x,y)$ is $L$-smooth with $L:=l+l^2/\mu_2$ and $\nabla \Phi(x) = \nabla_x f(x, y^*(x))$ for any $y^*(x)\in \Argmax_y f(x,y)$. 
\end{lemma}

\begin{lemma}[\citet{yang2020global}]  \label{lemma: Phi pl}
In the minimax problem, when the objective function $f$ satisfies Assumption \ref{assum:smooth} (Lipschitz gradient) and the two-sided PL condition with constant $\mu_1$ and $\mu_2$, then function $\Phi( x ) := \max_{ y }f( x ,  y )$ satisfies the PL condition with $\mu_1$. 
\end{lemma}

\subsection*{Proof of Theorem \ref{thm: plpl}}
\begin{proof}
We denote  $\kappa_1=l/\mu_1, \kappa_2=l/\mu_2$, $\kappa'=\max\{\kappa_1,\kappa_2\}$, $\kappa''=\min\{\kappa_1, \kappa_2\}$ in this section.

Parameters setting: $p=0$. 
When $\mu_1\geq \mu_2$, we choose  $\eta_x=\frac{\eta_y\mu_2^2}{64l^2}$, $\eta_{x,l}\leq \min\{\frac{1}{2l\sqrt{2(2K-1)(K-1)}}, \sqrt{\frac{\eta_x}{1536\eta_y l^2K^2}}, \\O(\epsilon\kappa'^{-2} \sqrt{ (\sigma_G^2+\sigma^2)}(Kl)^{-1})\}$, $ \eta_{y,l}\leq\min\{\frac{1}{2l\sqrt{2(2K-1)(K-1)}}, \sqrt{\frac{1}{1536 l^2K^2}},O(\epsilon\kappa'^{-2} \sqrt{ (\sigma_G^2+\sigma^2)}(Kl)^{-1})\}$, when $m=M$ or $\sigma_G=0$, we choose $K=O(1)m^{-1}\kappa'\kappa_1\kappa_2^2\epsilon^{-2}$, $T=O(1)\kappa_1\kappa_2^2\log (\epsilon^{-1}\kappa')$, $\eta_y= \frac{1}{4lK}$, when $m<M$ and $\sigma_G>0$, we choose $\eta_y=\min\{\frac{1}{4lK}, O(1)m\kappa'^{-1}\kappa_1^{-1}\kappa_2^{-2}\epsilon^2\}$, $K=O(1)$, $T=O(1)m^{-1}\kappa'\kappa_1^2\kappa_2^4\epsilon^{-2}\log(\epsilon^{-1}\kappa')$.

Conversely, when $\mu_1\leq \mu_2$, we choose  $\eta_y=\frac{\eta_x\mu_1^2}{64l^2}$, $\eta_{x,l}\leq \min\{\frac{1}{2l\sqrt{2(2K-1)(K-1)}}, \, \sqrt{\frac{1}{1536 l^2K^2}},\\O(\epsilon\kappa'^{-2} \sqrt{ (\sigma_G^2+\sigma^2)}(Kl)^{-1})\}$, $ \eta_{y,l}\leq\min\{\frac{1}{2l\sqrt{2(2K-1)(K-1)}}, \sqrt{\frac{\eta_y}{1536\eta_x l^2K^2}},O(\epsilon\kappa'^{-2} \sqrt{ (\sigma_G^2+\sigma^2)}(Kl)^{-1})\}$, when $m=M$ or $\sigma_G=0$, we choose $K=O(1)m^{-1}\kappa'\kappa_2\kappa_1^2\epsilon^{-2}$, $T=O(1)\kappa_2\kappa_1^2\log (\epsilon^{-1}\kappa')$, $\eta_x= \frac{1}{4lK}$, when $m<M$ and $\sigma_G>0$, we choose $\eta_x=\min\{\frac{1}{4lK}, O(1)m\kappa'^{-1}\kappa_1^{-2}\kappa_2^{-1}\epsilon^2\}$, $K=O(1)$, $T=O(1)m^{-1}\kappa'\kappa_2^2\kappa_1^4\epsilon^{-2}\log(\epsilon^{-1}\kappa')$.

We first consider the proof when $\mu_1\geq\mu_2$.

Since $\Phi(\cdot)$ is $L$-smooth, $L=l+\frac{l^2}{\mu_2}=(1+\kappa_2)l$ by Lemma \ref{lemma: Phi}, we have
    \begin{align*}
        &\mathbb{E}\Phi(x_t)-\mathbb{E}\Phi(x_{t+1})
        \geq \mathbb{E} \langle \nabla_x \Phi(x_t), x_t-x_{t+1}\rangle-\frac{L}{2}\mathbb{E}\|x_t-x_{t+1}\|^2.
    \end{align*}
Because of the $l$-smoothness of $f(\cdot)$, we have
    \begin{align*}
    &\mathbb{E}f(w_{t+1})-\mathbb{E}f(w_{t})\\
    \geq&\mathbb{E}\langle \nabla_x f(w_t), x_{t+1} - x_{t}\rangle +\mathbb{E}\langle \nabla_y f(w_t), y_{t+1} - y_{t}\rangle - \frac{l}{2}\mathbb{E}\|x_{t+1}-x_t\|^2-\frac{l}{2}\mathbb{E}\|y_{t+1}-y_t\|^2\\
    = &  \mathbb{E}\langle \nabla_x f(w_t), x_{t+1} - x_{t}\rangle +\eta_yK\mathbb{E}\langle \nabla_y f(w_t), \nabla_y f(w_t)-\Bar{e}_{y,t}\rangle - \frac{l}{2}\mathbb{E}\|x_{t+1}-x_t\|^2-\frac{l}{2}\mathbb{E}\|y_{t+1}-y_t\|^2\\
    = &  \mathbb{E}\langle \nabla_x f(w_t), x_{t+1} - x_{t}\rangle +\eta_yK\mathbb{E}\|\nabla_y f(w_t)\|^2 +\frac{\eta_yK}{2}\mathbb{E}\| \nabla_y f(w_t)-\Bar{e}_{y,t}\|^2-\frac{\eta_yK}{2}\mathbb{E}\| \nabla_y f(w_t)\|^2-\frac{\eta_yK}{2}\mathbb{E}\| \Bar{e}_{y,t}\|^2 -\\
    &\frac{l}{2}\mathbb{E}\|x_{t+1}-x_t\|^2-\frac{l}{2}\mathbb{E}\|y_{t+1}-y_t\|^2\\
    \geq& \mathbb{E}\langle \nabla_x f(w_t), x_{t+1} - x_{t}\rangle +\frac{\eta_yK}{2}\mathbb{E}\|\nabla_y f(w_t)\|^2- \frac{\eta_yK}{2}\mathbb{E}\| \Bar{e}_{y,t}\|^2 -\frac{l}{2}\mathbb{E}\|x_{t+1}-x_t\|^2-\frac{l}{2}\mathbb{E}\|y_{t+1}-y_t\|^2\\
    \overset{(a)}{\geq}&(\frac{\eta_yK}{2}-l\eta_y^2K^2)\mathbb{E}\|\nabla_y f(w_t)\|^2-l\eta_y^2K^2d_{y,t}-\frac{\eta_y K}{2}\mathbb{E}\|\Bar{e}_{y,t}\|^2+\mathbb{E} \langle \nabla_x f(w_t), x_{t+1} - x_{t}\rangle -\frac{l}{2}\mathbb{E}\|x_{t+1}-x_t\|^2\\
    \overset{(b)}{\geq}&\frac{\eta_yK}{4}\mathbb{E}\|\nabla_y f(w_t)\|^2-\frac{\eta_y K}{2}\mathbb{E}\|\Bar{e}_{y,t}\|^2-l\eta_y^2K^2d_{y,t}+\mathbb{E} \langle \nabla_x f(w_t), x_{t+1} - x_{t}\rangle -\frac{l}{2}\mathbb{E}\|x_{t+1}-x_t\|^2,
\end{align*}
where $(a)$ is due to the Lemma \ref{lemma: bound |x_t-x_{t+1}|}, and $(b)$ is due to the condition $\eta_y\leq \frac{1}{4lK}$.

Define $W_t=\Phi(x_t)-\Phi^*+\Phi(x_t)-f(x_t,y_t)=2\Phi(x_t)-f(x_t,y_t)-\Phi^*$, we have
\begin{align*}
    &\mathbb{E}W_t-\mathbb{E}W_{t+1}\\
    \geq& 2\mathbb{E} \langle \nabla_x \Phi(x_t), x_t-x_{t+1}\rangle-L\mathbb{E}\|x_t-x_{t+1}\|^2+\frac{\eta_yK}{4}\mathbb{E}\|\nabla_y f(w_t)\|^2-\frac{\eta_y K}{2}\mathbb{E}\|\Bar{e}_{y,t}\|^2-l\eta_y^2K^2d_{y,t}+\\
    &\mathbb{E} \langle \nabla_x f(w_t), x_{t+1} - x_{t}\rangle -\frac{l}{2}\mathbb{E}\|x_{t+1}-x_t\|^2.
\end{align*}
Denote $A_3=2\mathbb{E} \langle \nabla_x \Phi(x_t), x_t-x_{t+1}\rangle-L\mathbb{E}\|x_t-x_{t+1}\|^2+\mathbb{E} \langle \nabla_x f(w_t), x_{t+1} - x_{t}\rangle -\frac{l}{2}\mathbb{E}\|x_{t+1}-x_t\|^2$, we have
\begin{align*}
    A_3=& 2\mathbb{E} \langle \nabla_x \Phi(x_t)-\nabla_x f(w_t), x_t-x_{t+1}\rangle-\frac{2L+l}{2}\mathbb{E}\|x_t-x_{t+1}\|^2+\mathbb{E} \langle \nabla_x f(w_t), x_{t+1} - x_{t}\rangle \\
    \geq&\eta_yK\mathbb{E} \langle \nabla_x f(w_t), \nabla_x f(w_t)-\Bar{e}_{x,t}\rangle-2\eta_xK\mathbb{E} \|\nabla_x \Phi(x_t)-\nabla_x f(w_t)\|\|\nabla_x f(w_t)-\Bar{e}_{x,t}\|-\frac{2L+l}{2}\mathbb{E}\|x_t-x_{t+1}\|^2\\
     \overset{(a)}{\geq}& \frac{\eta_xK}{2}\mathbb{E} \|\nabla_x f(w_t)\|^2 -\frac{\eta_xK}{2}\mathbb{E}\|\Bar{e}_{x,t}\|^2-8\eta_xK\mathbb{E} \|\nabla_x \Phi(x_t)-\nabla_x f(w_t)\|^2-\frac{\eta_xK}{8}\mathbb{E}\|\nabla_x f(w_t)-\Bar{e}_{x,t}\|^2-\\
     &(2L+l)\eta_x^2K^2\mathbb{E}\|\nabla_x f(w_t)\|^2-(2L+l)\eta_x^2K^2 d_{x,t}\\
    \overset{(b)}{\geq}& \frac{\eta_xK}{2}\mathbb{E} \|\nabla_x f(w_t)\|^2 -\frac{\eta_xK}{2}\mathbb{E}\|\Bar{e}_{x,t}\|^2-8\eta_xKl^2\mathbb{E} \|y_t-y^*(x_t)\|^2-\frac{\eta_xK}{4}\mathbb{E}\|\nabla_x f(w_t)\|^2-\frac{\eta_xK}{4}\mathbb{E}\|\Bar{e}_{x,t}\|^2-\\
     &(2L+l)\eta_x^2K^2\mathbb{E}\|\nabla_x f(w_t)\|^2-(2L+l)\eta_x^2K^2 d_{x,t} \\
     \overset{(c)}{\geq}&\left(\frac{\eta_xK}{4}-\eta_x^2K^2(2L+l)\right)\mathbb{E}\|\nabla_x f(w_t)\|^2-\frac{3\eta_xK}{4}\mathbb{E}\|\Bar{e}_{x,t}\|^2-(2L+l)\eta_x^2K^2d_{x,t}-\frac{8\eta_xKl^2}{\mu_2^2}\mathbb{E}\|\nabla_y f(w_t)\|^2\\ 
     \overset{(d)}{\geq}&\frac{\eta_xK}{8}\mathbb{E}\|\nabla_x f(w_t)\|^2-\frac{8\eta_xKl^2}{\mu_2^2}\mathbb{E}\|\nabla_y f(w_t)\|^2-\frac{3\eta_xK}{4}\mathbb{E}\|\Bar{e}_{x,t}\|^2-(2L+l)\eta_x^2K^2d_{x,t},
\end{align*}
where $(a)$ is due to Lemma \ref{lemma: bound |x_t-x_{t+1}|}, $(b)$ is due to $l$-smoothness of $f$, $(c)$ is due to $\mu_2$-PL condition of $f(x,\cdot)$, $(d)$ is due to the condition $\eta_x=\frac{\eta_y}{64\kappa_2^2}\leq \frac{1}{8(2L+l)K}$.

Then, we have
\begin{align*}
    &\mathbb{E}W_t-\mathbb{E}W_{t+1}\\
    \geq& \frac{\eta_xK}{8}\mathbb{E}\|\nabla_x f(w_t)\|^2+\left(\frac{\eta_yK}{4}-\frac{8\eta_xKl^2}{\mu_2^2}\right)\mathbb{E}\|\nabla_y f(w_t)\|^2-\frac{\eta_y K}{2}\mathbb{E}\|\Bar{e}_{y,t}\|^2-l\eta_y^2K^2d_{y,t}-\\
    &\frac{3\eta_xK}{4}\mathbb{E}\|\Bar{e}_{x,t}\|^2-(2L+l)\eta_x^2K^2d_{x,t}\\
    \overset{(a)}{\geq}&\frac{\eta_xK}{8}\mathbb{E}\|\nabla_x f(w_t)\|^2+\frac{\eta_yK}{8}\mathbb{E}\|\nabla_y f(w_t)\|^2-\frac{\eta_y K}{2}\mathbb{E}\|\Bar{e}_{y,t}\|^2-l\eta_y^2K^2d_{y,t}-\\
    &\frac{3\eta_xK}{4}\mathbb{E}\|\Bar{e}_{x,t}\|^2-(2L+l)\eta_x^2K^2d_{x,t}\\
    \overset{(b)}{\geq}&\frac{\eta_xK}{8}\mathbb{E}\|\nabla_x f(w_t)\|^2+\frac{\eta_yK}{8}\mathbb{E}\|\nabla_y f(w_t)\|^2-(\eta_yK+2l\eta_y^2K^2)\mathbb{E}\|\Bar{e}_{y,t}\|^2-2l\eta_y^2K\frac{\sigma^2}{m}-4l\eta_y^2K^2(M-m)\frac{\sigma_G^2}{mM}-\\
    &(\eta_xK+2(2L+l)\eta_x^2K^2)\mathbb{E}\|\Bar{e}_{x,t}\|^2-2(2L+l)\eta_x^2K\frac{\sigma^2}{m}-4(2L+l)\eta_x^2K^2(M-m)\frac{\sigma_G^2}{mM}\\
    \overset{(c)}{\geq}&\frac{\eta_xK}{8}\mathbb{E}\|\nabla_x f(w_t)\|^2+\frac{\eta_yK}{8}\mathbb{E}\|\nabla_y f(w_t)\|^2-4l\eta_y^2K\frac{\sigma^2}{m}-8l\eta_y^2K^2(M-m)\frac{\sigma_G^2}{mM}-\\
    &4\eta_yKl^2[24K^2\eta_{x,l}^2\mathbb{E}\|\nabla_x f(w_{t})\|^2+24K^2\eta_{y,l}^2\mathbb{E}\|\nabla_y f(w_{t})\|^2+24K^2(\eta_{x,l}^2+\eta_{y,l}^2)\sigma_G^2+3K(\eta_{x,l}^2+2K\eta_{y,l}^2)\sigma^2]\\
    \overset{(d)}{\geq}&\frac{\eta_xK}{16}\mathbb{E}\|\nabla_x f(w_t)\|^2+\frac{\eta_yK}{16}\mathbb{E}\|\nabla_y f(w_t)\|^2-4l\eta_y^2K\frac{\sigma^2}{m}-8l\eta_y^2K^2(M-m)\frac{\sigma_G^2}{mM}-\\
    &4\eta_yKl^2[24K^2(\eta_{x,l}^2+\eta_{y,l}^2)\sigma_G^2+3K(\eta_{x,l}^2+2K\eta_{y,l}^2)\sigma^2],
\end{align*}
where $(a)$ is due to the condition $\eta_x=\frac{\eta_y\mu_2^2}{64l^2}=\frac{\eta_y}{64\kappa_2^2}$, $(b)$ is due to Lemma \ref{lemma: bound d}, $(c)$ is due to Lemma \ref{lemma: bound e}, $(d)$ is due to the condition $\eta_{x,l}^2\leq \frac{\eta_x}{1536K^2\eta_yl^2}, \eta_{y,l}\leq \frac{1}{1536K^2l^2}$.

Note that 
\begin{align*}
    &\frac{\eta_xK}{16}\mathbb{E}\|\nabla_x f(w_t)\|^2+\frac{\eta_yK}{16}\mathbb{E}\|\nabla_y f(w_t)\|^2\\
    \geq&\frac{\eta_xK}{32}\mathbb{E}\|\nabla_x \Phi(x_t)\|^2-\frac{\eta_xK}{8}\mathbb{E}\|\nabla_x \Phi(x_t)-\nabla_x f(w_t)\|^2+\frac{\eta_yK}{16}\mathbb{E}\|\nabla_y f(w_t)\|^2\\
    \geq&\frac{\eta_xK}{32}\mathbb{E}\|\nabla_x \Phi(x_t)\|^2-\frac{\eta_xKl^2}{8}\mathbb{E}\|y_t-y^*(x_t)\|^2+\frac{\eta_yK}{16}\mathbb{E}\|\nabla_y f(w_t)\|^2\\
    \overset{(a)}{\geq}&\frac{\eta_xK}{32}\mathbb{E}\|\nabla_x \Phi(x_t)\|^2-\frac{\eta_xKl^2}{8\mu_2^2}\mathbb{E}\|\nabla_y f(w_t\|^2+\frac{\eta_yK}{16}\mathbb{E}\|\nabla_y f(w_t)\|^2\\
    \overset{(b)}{\geq}&\frac{\eta_xK}{32}\mathbb{E}\|\nabla_x \Phi(x_t)\|^2+\frac{\eta_yK}{32}\mathbb{E}\|\nabla_y f(w_t)\|^2\\
    \overset{(c)}{\geq}&\frac{\eta_xK\mu_1}{16}(\Phi(x_t)-\Phi^*)+\frac{\eta_yK\mu_2}{16}\mathbb{E}(\Phi(x_t)-f(x_t,y_t))\\
    \geq&\frac{\eta_xK\mu_1}{16}\mathbb{E}(\Phi(x_t)-\Phi^*+\Phi(x_t)-f(x_t,y_t))\\
    \geq&\frac{\eta_xK\mu_1}{16}\mathbb{E}W_t,
\end{align*}
where $(a)$ is due to $\mu_2$-PL condition of $f(x,\cdot)$, $(b)$ is due to the condition $\eta_x=\frac{\eta_y\mu_2^2}{64l^2}=\frac{\eta_y}{64\kappa_2^2}$, and $(c)$ is due to the two-side PL condition of $f$ and Lemma \ref{lemma: Phi pl}.

Thus, we have
\begin{align*}
    \mathbb{E}W_t-\mathbb{E}W_{t+1}\geq& \frac{\eta_xK\mu_1}{16}\mathbb{E}W_t-4l\eta_y^2K\frac{\sigma^2}{m}-8l\eta_y^2K^2(M-m)\frac{\sigma_G^2}{mM}-\\
    &4\eta_yKl^2[24K^2(\eta_{x,l}^2+\eta_{y,l}^2)\sigma_G^2+3K(\eta_{x,l}^2+2K\eta_{y,l}^2)\sigma^2].
\end{align*}
By telescoping and rearranging, we have
\begin{align*}
    \mathbb{E}W_{t+1}\leq&\left(1-\frac{\eta_xK\mu_1}{16}\right) \mathbb{E}W_t-4l\eta_y^2K\frac{\sigma^2}{m}-8l\eta_y^2K^2(M-m)\frac{\sigma_G^2}{mM}-\\
    &4\eta_yKl^2[24K^2(\eta_{x,l}^2+\eta_{y,l}^2)\sigma_G^2+3K(\eta_{x,l}^2+2K\eta_{y,l}^2)\sigma^2],
\end{align*}
\begin{align*}
    \mathbb{E}W_{t}\leq&\left(1-\frac{\eta_xK\mu_1}{16}\right)^t \mathbb{E}W_0+\frac{64l\eta_y^2\sigma^2}{\eta_x\mu_1 m}+\frac{126l\eta_y^2K}{\eta_x\mu_1}(M-m)\frac{\sigma_G^2}{mM}+\\
    &\frac{64\eta_yl^2}{\eta_x\mu_1}[24K^2(\eta_{x,l}^2+\eta_{y,l}^2)\sigma_G^2+3K(\eta_{x,l}^2+2K\eta_{y,l}^2)\sigma^2]\\
    =&\left(1-\frac{\eta_yKl}{256\kappa_1\kappa_2^2}\right)^t \mathbb{E}W_0+O(1)\frac{\kappa_1\kappa_2^2\eta_y\sigma^2}{m}+O(1)\kappa_1\kappa_2^2\eta_yK(M-m)\frac{\sigma_G^2}{mM}+\\
    &O(1){\kappa_1\kappa_2^2l}[24K^2(\eta_{x,l}^2+\eta_{y,l}^2)\sigma_G^2+3K(\eta_{x,l}^2+2K\eta_{y,l}^2)\sigma^2].
\end{align*}
Note that
\begin{align*}
    \mathbb{E}\|x_t-x^*\|^2\leq\frac{2}{\mu_1}\mathbb{E}(\Phi(x_t)-\Phi^*),
\end{align*}
\begin{align*}
    \mathbb{E}\|y_t-y^*\|^2\leq2\mathbb{E}\|y_t-y^*(x_t)\|^2+2\mathbb{E}\|y^*(x_t)-y^*\|^2\leq \frac{2}{\mu_2}(\Phi(x_t)-f(x_t,y_t))+\frac{2}{\mu_1}\mathbb{E}(\Phi(x_t)-\Phi^*).
\end{align*}
\begin{align}\nonumber
    \mathbb{E}\|x_t-x^*\|^2+\mathbb{E}\|y_t-y^*\|^2\leq& O(1)\kappa'\left(1-\frac{\eta_yKl}{256\kappa_1\kappa_2^2}\right)^t \mathbb{E}W_0+O(1)\frac{\kappa'\kappa_1\kappa_2^2\eta_y\sigma^2}{m}+O(1)\kappa'\kappa_1\kappa_2^2\eta_yK(M-m)\frac{\sigma_G^2}{mM}+\\ \label{pl pl eq}
    &O(1){\kappa'\kappa_1\kappa_2^2l}[24K^2(\eta_{x,l}^2+\eta_{y,l}^2)\sigma_G^2+3K(\eta_{x,l}^2+2K\eta_{y,l}^2)\sigma^2].
\end{align}
When $M=m$ or $\sigma_G=0$, with $K=O(1)m^{-1}\kappa'\kappa_1\kappa_2^2\epsilon^{-2}$, $\eta_y=1/(4lK)=O(1)m\kappa'^{-1}\kappa_1^{-1}\kappa_2^{-2}\epsilon^2$, $T=O(1)\kappa_1\kappa_2^2\log (\epsilon^{-1}\kappa')$, $\eta_{x,l}^2\leq O(1)\kappa'^{-1}\kappa_1^{-1}\kappa_2^{-2}K^{-2}\epsilon^2$, $\eta_{x,l}^2\leq O(1)\kappa'^{-1}\kappa_1^{-1}\kappa_2^{-2}K^{-2}\epsilon^2$,  we have
\begin{align*}
    \mathbb{E}\|x_T-x^*\|^2+\mathbb{E}\|y_T-y^*\|^2\leq O(1)\epsilon^2,
\end{align*}
which means a per-client sample complexity of $O(m^{-1}\kappa'\kappa_1^2\kappa_2^4\epsilon^{-2}\log(\epsilon^{-1}\kappa'))$, a communication complexity of $O(\kappa_1\kappa_2^2\log(\epsilon^{-1}\kappa'))$.

When $m<M$ and $\sigma_G>0$, with $\eta_y=O(1)m\kappa'^{-1}\kappa_1^{-1}\kappa_2^{-2}\epsilon^2$, $K=O(1)$, $T=O(1)m^{-1}\kappa'\kappa_1^2\kappa_2^4\epsilon^{-2}\log(\epsilon^{-1}\kappa')$, $\eta_{x,l}^2\leq O(1)\kappa'^{-1}\kappa_1^{-1}\kappa_2^{-2}K^{-2}\epsilon^2$, $\eta_{x,l}^2\leq O(1)\kappa'^{-1}\kappa_1^{-1}\kappa_2^{-2}K^{-2}\epsilon^2$,  we have
\begin{align*}
    \mathbb{E}\|x_T-x^*\|^2+\mathbb{E}\|y_T-y^*\|^2\leq O(1)\epsilon^2,
\end{align*}
which means both per-client sample complexity and communication complexity  are $O(m^{-1}\kappa'\kappa_1^2\kappa_2^4\epsilon^{-2}\log(\epsilon^{-1}\kappa'))$.
Using Kakutoni’s Theorem, we have
\begin{align*}
    \min_x\max_y f(x,y)=\max_y\min_x f(x,y)=\min_y\max_x (-f(x,y))=\min_y\max_x g(y,x),
\end{align*}
where we denote $g(y,x)=-f(x,y)$. 

Thus, the minimax problem of a function with $\mu_1$-PL-$\mu_2$-PL is equivalent to minimax problem of a function with $\mu_2$-PL-$\mu_1$-PL. When $M=m$ or $\sigma_G=0$, it is guaranteed to find $x_T, y_T$ satisfying $\mathbb{E}\|x_T-x^*\|^2+\mathbb{E}\|y_T-y^*\|^2\leq O(1)\epsilon^2$ with a  per-client sample complexity of $O(m^{-1}\kappa'\kappa_1^4\kappa_2^2\epsilon^{-2}\log(\epsilon^{-1}\kappa'))$ and a communication complexity of $O(\kappa_1^2\kappa_2\log(\epsilon^{-1}\kappa'))$.

Overall, when $M=m$ or $\sigma_G=0$, we can find $x_T, y_T$ satisfying $\mathbb{E}\|x_T-x^*\|^2+\mathbb{E}\|y_T-y^*\|^2\leq O(1)\epsilon^2$ with a  per-client sample complexity of $O(m^{-1}\kappa'^3\kappa''^4\epsilon^{-2}\log(\epsilon^{-1}\kappa'))$ and a communication complexity of $O(\kappa'\kappa''^2\log(\epsilon^{-1}\kappa'))$, where $\kappa'=\max\{\kappa_1,\kappa_2\},\kappa''=\min\{\kappa_1, \kappa_2\}$.

Similarly, when $m<M$ and $\sigma_G>0$, we we can find $x_T, y_T$ satisfying $\mathbb{E}\|x_T-x^*\|^2+\mathbb{E}\|y_T-y^*\|^2\leq O(1)\epsilon^2$ with a  per-client sample complexity of $O(m^{-1}\kappa'^3\kappa''^4\epsilon^{-2}\log(\epsilon^{-1}\kappa'))$ and a communication complexity of $O(m^{-1}\kappa'^3\kappa''^4\epsilon^{-2}\log(\epsilon^{-1}\kappa'))$, where $\kappa'=\max\{\kappa_1,\kappa_2\},\kappa''=\min\{\kappa_1, \kappa_2\}$.
\end{proof}
\section{Proof of Proposition \ref{Translation}}
\label{app: Translation}
\begin{proof}
According to Proposition 2.1 and (7) in \cite{yang2022faster}, if $(\Tilde{x},\Tilde{y})$ is an $(\epsilon, \epsilon/\sqrt{\kappa})$-stationary point of $f$, then $\|\nabla_x \Phi_{1/2l}(\Tilde{x})\|\leq 2\sqrt{2}\epsilon$. If we could find $\hat{x}$ such that $\mathbb{E}\|\hat{x}-x^*(\Tilde{x})\|\leq\frac{\epsilon}{\kappa l}$, then
\begin{align*}
    \mathbb{E}\|\nabla_x \Phi(\hat{x})\|&\leq \mathbb{E}\|\nabla_x \Phi(x^*(\Tilde{x}))\|+\mathbb{E}\|\nabla_x \Phi(\hat{x})-\nabla_x \Phi(x^*(\Tilde{x}))\|\\
    &\leq \mathbb{E}\|\nabla_x \Phi_{1/2l}(\Tilde{x})\|+ 2\kappa l\mathbb{E}\|\hat{x}-x^*(\Tilde{x})\|\\
    &\leq (2\sqrt{2}+2)\epsilon,
\end{align*}
where the second inequality is because of Lemma \ref{stationary of Phi} and Lemma \ref{lemma: Phi}. Note that $x^*(\Tilde{x})$ is the solution to $\min_x\max_y \Tilde{f}(x,y)=\min_x\max_y f(x,y)+l\|x-\Tilde{x}\|^2$. 

Note that $\Tilde{f}(x,y)=f(x,y)+l\|x-\Tilde{x}\|^2$ is $3l$-smooth, $l$-strongly convex in $x$, $\mu$-PL in $y$. According to Theorem \ref{thm: plpl}, we can use \alg to optimize $\Tilde{f}(x,y)$ from initial point $(\Tilde{x}, \Tilde{y})$. Furthermore, according to \eqref{pl pl eq}, with $\eta_x\leq 1/(4lK), \eta_y=\frac{\eta_x\mu_1^2}{64l^2}$, $\eta_{x,l}\leq \min\{\frac{1}{2l\sqrt{2(2K-1)(K-1)}}, \sqrt{\frac{1}{1536 l^2K^2}},O(\epsilon\kappa^{-3} \sqrt{ (\sigma_G^2+\sigma^2)}(Kl)^{-1})\}$, $ \eta_{y,l}\leq\min\{\frac{1}{2l\sqrt{2(2K-1)(K-1)}}, \sqrt{\frac{\eta_y}{1536\eta_x l^2K^2}},O(\epsilon\kappa^{-3} \sqrt{ (\sigma_G^2+\sigma^2)}(Kl)^{-1})\}$, we have 
\begin{align*}
    \mathbb{E}\|x_t-x^*\|^2&\leq O(1)\kappa\left(1-\frac{\eta_xKl}{256\times9\kappa}\right)^t \mathbb{E}W_0+O(1)\frac{\kappa^2\eta_x\sigma^2}{m}+O(1)\kappa^2\eta_xK(M-m)\frac{\sigma_G^2}{mM}+O(1)\kappa^{-2}\epsilon^2,
\end{align*}
where since $\min_x\max_y\Tilde{f}(x,y)= \min_y\max_x (-\Tilde{f}(x,y))$, we redefine $W=\Tilde{\Psi}^*-\Tilde{\Psi}(y)+\Tilde{f}(x, y)-\Tilde{\Psi}(y)$, $\Tilde{\Psi}(y)=\min_x \Tilde{f}(x , y)=\min_x (-\Tilde{f}(x,y))$. We then have
\begin{align*}
    W_0=&\Tilde{\Psi}^*-\Tilde{\Psi}(\Tilde{y})+\Tilde{f}(\Tilde{x}, \Tilde{y})-\Tilde{\Psi}(\Tilde{y})\\
    \overset{(a)}{\leq}&\Tilde{\Psi}^*-\Tilde{\Psi}(\Tilde{y})+\frac{1}{2l}\|\nabla_x \Tilde{f}(\Tilde{x}, \Tilde{y})\|^2 \\
    =&\max_y\min_x \Tilde{f}(x , y)- \max_y\Tilde{f}(\Tilde{x} , y)+\max_y\Tilde{f}(\Tilde{x} , y)-\Tilde{f}(\Tilde{x}, \Tilde{y})+\Tilde{f}(\Tilde{x}, \Tilde{y})-\min_x \Tilde{f}(x , \Tilde{y}) +\frac{1}{2l}\|\nabla_x \Tilde{f}(\Tilde{x}, \Tilde{y})\|^2 \\
    \overset{(b)}{\leq}& \frac{1}{\mu}\|\nabla_y \Tilde{f}( \Tilde{x}, \Tilde{y})\|^2+\frac{1}{l}\|\nabla_x \Tilde{f}( \Tilde{x}, \Tilde{y})\|^2+\frac{1}{2l}\|\nabla_x \Tilde{f}(\Tilde{x}, \Tilde{y})\|^2\\
    \overset{(c)}{\leq}& \frac{2\epsilon^2}{l},
\end{align*}
where $(a)$ is due to $l$-strongly convexness of $x$, $(b)$ is due to $l$-strongly convexness of $x$ and $\mu$-PL of $y$, $(c)$ is because that $(\Tilde{x},\Tilde{y})$ is an $(\epsilon, \epsilon/\sqrt{\kappa})$-stationary point of $f$. Thus, we have
\begin{align*}
    \mathbb{E}\|x_t-x^*\|^2&\leq O(1)\kappa\epsilon^2\left(1-\frac{\eta_xKl}{256\times9\kappa}\right)^t +O(1)\frac{\kappa^2\eta_x\sigma^2}{m}+O(1)\kappa^2\eta_xK(M-m)\frac{\sigma_G^2}{mM}+O(1)\kappa^{-2}\epsilon^2,
\end{align*}

Therefore, when $m=M$ or $\sigma_G=0$, with $\eta_x=1/(4lK)=O(1)m\epsilon^2\kappa^{-4}$, $K=O(1)m^{-1}\epsilon^{-2}\kappa^{4}$, $T=O(1)\kappa\log(\kappa)$ we can find $\hat{x}$ such that $\mathbb{E}\|\hat{x}-x^*(\Tilde{x})\|\leq O(\frac{\epsilon}{\kappa})$ and $\mathbb{E}\|\nabla_x \Phi(\hat{x})\|\leq O(\epsilon)$ with $KT=O(m^{-1}\kappa^5\epsilon^{-2}\log(\kappa))$ per-client sample complexity and $T=O(\kappa\log(\kappa))$ communication complexity. When $m<M$ and $\sigma_G>0$, with  $K=O(1)$, $\eta_x=\min\{1/(4lK), O(1)m\epsilon^2\kappa^{-4}\}=O(1)m\epsilon^2\kappa^{-4}$, $T=O(1)m^{-1}\kappa^5\epsilon^{-2}\log(\kappa)$, we can find $\hat{x}$ such that $\mathbb{E}\|\hat{x}-x^*(\Tilde{x})\|\leq O(\frac{\epsilon}{\kappa})$ and $\mathbb{E}\|\nabla_x \Phi(\hat{x})\|\leq O(\epsilon)$ with $KT=O(m^{-1}\kappa^5\epsilon^{-2}\log(\kappa))$ per-client sample complexity and $T=O(m^{-1}\kappa^5\epsilon^{-2}\log(\kappa))$ communication complexity.
\end{proof}

\section{Additional Experiments}
\label{app: exp}

\subsection*{Fair Classification}
\begin{figure}[!ht]
    \centering
    \includegraphics[width=30em]{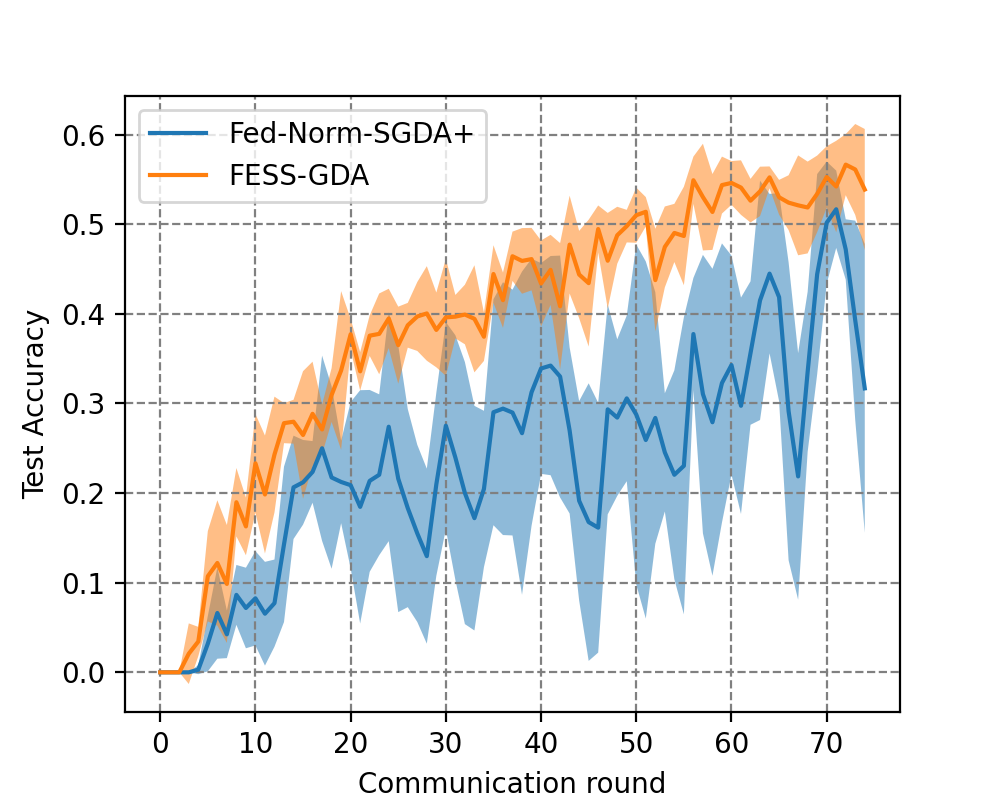}
    \caption{Comparison between Fed-Norm-SGDA+ and \alg for the worst test accuracy over 10 categories of CIFAR-10.}
    \label{fig: fc_w}
\end{figure}
For the fair classification task, we have presented the average test accuracy results in Section \ref{exp: fc}. To compare the fairness of models trained with \alg and Fed-Norm-SGDA+, following the same setting in Section \ref{exp: fc}, we now present the worst-case test accuracy of models over 10 categories in Figure \ref{fig: fc_w}. Figure \ref{fig:c} and Figure \ref{fig: fc_w} show that models trained with \alg not only have better average test accuracy over all categories, but also have better worst-case test accuracy over all categories, which demonstrates that models trained with \alg have better overall performance as well as fairness compared to models trained with Fed-Norm-SGDA+.

\subsection*{Communication Savings from Multiple Local Updates}
\begin{figure}[!ht]
    \centering
    \includegraphics[width=30em]{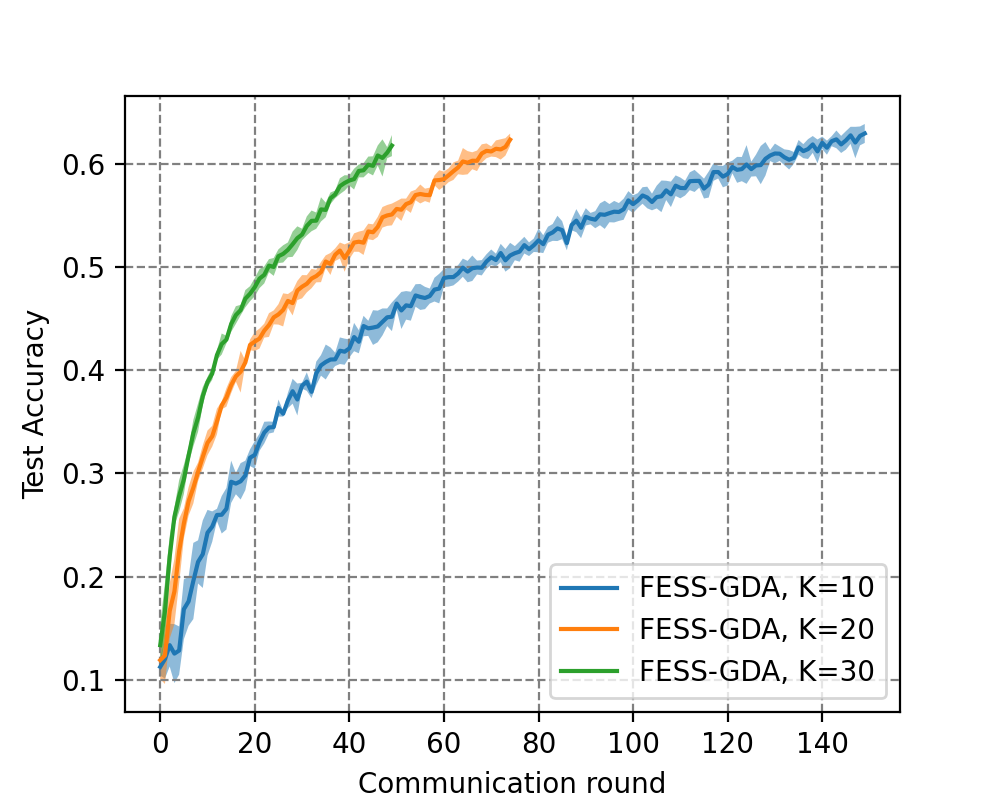}
    \caption{\alg for the fair classification task on CIFAR-10 with different number of local updates.}
    \label{fig:communication}
\end{figure}
We test \alg for the fair classification task on the CIFAR10 dataset using the same setting as in Section \ref{exp: fc} with $\eta_{x,l}=\eta_{y,l}=0.1, \eta_{x,g}=\eta_{y,g}=1, p=0.1, \beta=0.9$ and number of local updates $K$ from $\{10,20,30\}$. Each experiment is repeated 5 times and we report the average performance. As we can see from Figure \ref{fig:communication}, \alg has significant communication savings from multiple local updates.

\subsection*{Model Architecture for Fair Classification}
Table \ref{tab:net} shows the architecture of the convolutional neural network we used for the fair classification task.
\begin{table}[!ht]
\centering
\caption{Model Architecture for CIFAR10 dataset}
\vskip 0.1in
\begin{tabular}{@{}lll@{}}
\toprule
Layer Type           & Shape               &padding\\ \midrule
Convolution $+$ ReLU & $3 \times 3 \times 16$  &1   \\
Max Pooling          & $2 \times 2$                \\
Convolution $+$ ReLU & $3 \times 3 \times 32$ &1\\
Max Pooling          & $2 \times 2$             \\
Convolution $+$ ReLU & $3 \times 3 \times 64$ &1\\
Max Pooling          & $2 \times 2$             \\
Fully Connected $+$ ReLU & $512$ \\
Fully Connected $+$ ReLU & $64$ \\
Fully Connected & $10$ \\ \bottomrule
\end{tabular}
\label{tab:net}
\end{table}

\end{document}